\documentclass{article}

\usepackage[%
	backend=bibtex,%
	style=numeric-comp,
	maxbibnames=99,
	hyperref=true,%
	maxcitenames=2,%
	natbib,%
	doi=true,%
	url=true,%
	isbn=false,%
	firstinits=true%
]{biblatex}
\bibliography{neurips_2022}

     \usepackage[final,nonatbib]{neurips_2022}

\usepackage{amsmath,amsfonts,bm}

\def\eqref#1{equation~\ref{#1}}

\def\1{\bm{1}}

\def\vb{{\bm{b}}}
\def\vc{{\bm{c}}}
\def\vd{{\bm{d}}}
\def\ve{{\bm{e}}}
\def\vf{{\bm{f}}}
\def\vg{{\bm{g}}}
\def\vh{{\bm{h}}}
\def\vi{{\bm{i}}}

\def\vo{{\bm{o}}}

\def\vr{{\bm{r}}}
\def\vs{{\bm{s}}}

\def\vx{{\bm{x}}}

\def\vz{{\bm{z}}}

\def\mA{{\bm{A}}}
\def\mB{{\bm{B}}}
\def\mC{{\bm{C}}}
\def\mD{{\bm{D}}}

\def\mI{{\bm{I}}}
\def\mJ{{\bm{J}}}
\def\mK{{\bm{K}}}
\def\mL{{\bm{L}}}

\def\mO{{\bm{O}}}
\def\mP{{\bm{P}}}

\def\mS{{\bm{S}}}

\def\mU{{\bm{U}}}
\def\mV{{\bm{V}}}
\def\mW{{\bm{W}}}

\DeclareMathAlphabet{\mathsfit}{\encodingdefault}{\sfdefault}{m}{sl}
\SetMathAlphabet{\mathsfit}{bold}{\encodingdefault}{\sfdefault}{bx}{n}

\def\gN{{\mathcal{N}}}

\def\sR{{\mathbb{R}}}

\usepackage[utf8]{inputenc} 
\usepackage[T1]{fontenc}    
\usepackage{hyperref}       
\usepackage{url}            
\usepackage{booktabs}       
\usepackage{amsfonts}       
\usepackage{nicefrac}       
\usepackage{microtype}      
\usepackage{xcolor}         
\usepackage{float}
\usepackage{graphicx}
\usepackage{subfigure}
\usepackage[belowskip=-15pt,aboveskip=0pt]{caption}
\usepackage{amsmath}
\usepackage{amssymb}
\usepackage{mathtools}
\usepackage{amsthm}
\newcommand\norm[1]{\left\lVert#1\right\rVert}
\theoremstyle{plain}

\newcommand*{\tran}{^{\mkern-1.5mu\mathsf{T}}}
\renewcommand{\eqref}[1]{~(\ref{#1})}
\newtheorem{theorem}{Theorem}
\newtheorem{corollary}{Corollary}     
\newtheorem{remark}{Remark}           
\newtheorem{proposition}{Proposition}

\usepackage{authblk}

\title{On the difficulty of learning chaotic dynamics with RNNs}

\author[1,2,*]{\textbf{Jonas M. Mikhaeil}}
\author[1,4*]{\textbf{Zahra Monfared}}
\author[1,2,3]{\textbf{Daniel Durstewitz}}
\affil[ ]{\texttt{j.mikhaeil@columbia.edu, $\{$zahra.monfared, daniel.durstewitz$\}$@zi-mannheim.de }}

\affil[1]{\footnotesize Department of Theoretical Neuroscience, Central Institute of Mental Health, Medical Faculty Mannheim, Heidelberg University, Mannheim, Germany}
\affil[2]{\footnotesize Faculty of Physics and Astronomy, Heidelberg University, Heidelberg, Germany}
\affil[3]{\footnotesize
Interdisciplinary Center for Scientific Computing, Heidelberg University 
}
\affil[4]{\footnotesize
Department of Mathematics \& Informatics and Cluster of Excellence STRUCTURES, Heidelberg University, Heidelberg, Germany}
\affil[*]{\footnotesize
These authors contributed equally}

\begin{document}

\maketitle

\begin{abstract}
 Recurrent neural networks (RNNs) are wide-spread machine learning tools for modeling sequential and time series data. They are notoriously hard to train because their loss gradients backpropagated in time tend to saturate or diverge during training. This is known as the exploding and vanishing gradient problem. Previous solutions to this issue either built on rather complicated, purpose-engineered architectures with gated memory buffers, or - more recently - imposed constraints that ensure convergence to a fixed point or restrict (the eigenspectrum of) the recurrence matrix. Such constraints, however, convey severe limitations on the expressivity of the RNN. Essential intrinsic dynamics such as multistability or chaos are disabled. This is inherently at disaccord with the chaotic nature of many, if not most, time series encountered in nature and society. It is particularly problematic in scientific applications where one aims to reconstruct the underlying dynamical system. 
Here we offer a comprehensive theoretical treatment of this problem by relating the loss gradients during RNN training to the Lyapunov spectrum of RNN-generated orbits. We mathematically prove that RNNs producing stable equilibrium or cyclic behavior have bounded gradients, whereas the gradients of RNNs with chaotic dynamics always diverge. 
Based on these analyses and insights we suggest ways of how to optimize the training process on chaotic data according to the system's Lyapunov spectrum, regardless of the employed RNN architecture. \end{abstract}

\section{Introduction}
Recurrent neural networks (RNNs) are widely used across various fields in engineering and science for learning sequential tasks or modeling and predicting time series \citep{lipton_critical_2015}. Yet, they struggle when long-term temporal dependencies, very slow, or hugely varying time scales are involved \citep{hochreiter_thesis_1991,bengio1994learning,schmidt_identifying_2021,li_indrnn_2018,rusch_coupled_2021}. Time series or sequential data with such properties are, however, very common in fields like climate physics \citep{thomson_time_1990}, neuroscience \citep{fusi_neural_2007,russo_cell_2017}, ecology \citep{turchin_complex_1992}, or language processing \citep{cho_learning_2014}. 
Training RNNs on such data is hard because the loss gradients backpropagated in time easily saturate or diverge in this process. This is commonly referred to as the exploding and vanishing gradient problem (EVGP) \citep{hochreiter_thesis_1991,bengio1994learning,pascanu_difficulty_2013}. 

One solution to the EVGP is based on specifically designed RNN architectures with gating mechanisms, such as long short-term memory (LSTM) \citep{hochreiter_long_1997} or gated recurrent units (GRU) \citep{cho-etal-2014-properties}. These architectures allow states at earlier time steps to more easily influence activity much later through a kind of protected memory buffer, thus alleviating the EVGP by structural design. In practice, such models need to be backed up by further techniques like gradient clipping to keep the gradients in check \citep{pascanu_difficulty_2013}. The relatively complex architectural design of these networks impedes their mathematical analysis and requires reverse engineering after training \citep{maheswaranathan_reverse_2019,monfared_existence_2020,monfared_transformation_2020,schmidt_identifying_2021}. Partly to forego these complications, a variety of other solutions has been proposed recently, imposing restrictions on the recurrence matrix to bound the gradients \citep{Arjovsky_2016,chang_antisymmetricrnn_2019}, or enforcing global stability by design or regularization \citep{erichson_lipschitz_2020,kolter_learning_2019}. Often these procedures dramatically curtail
the expressivity of the RNN \citep{kerg_non-normal_2019,orhan_improved_2020,schmidt_identifying_2021}; in particular, they rule out chaotic dynamics (for reasons discussed further below).
 
This is at odds with the plethora of chaotic phenomena in nature, engineering, and society. Chaotic dynamics are commonplace, almost default in any complex physical or biological system. This includes scientific areas as diverse as neuroscience \citep{durstewitz_dynamical_2007,vreeswijk_chaos_1996}, physiology
\citep{kesmia_control_2020}, geophysics
\citep{sivakumar_chaos_2004}, climate systems \citep{tziperman_controlling_1997}, astrophysics
\citep{laskar_chaotic_1993}, ecology \citep{duarte_quantifying_2010}, chemical reactions \citep{chaos_chemistry_1993}, cell \citep{olsen_chaos_1977} or population \citep{Wishart_chaos_87} biology. Chaotic phenomena are also crucial for the understanding of societal and epidemiological processes, such as the spread of diseases
\citep{mangiarotti_chaos_2020}, or in economics \citep{faggini_chaotic_2014}. They are further relevant in purely technical contexts such as electrical engineering
\citep{tchitnga_novel_2019,kamdjeu_kengne_dynamics_2021}
or laser optics
\citep{kantz_nonlinear_1993}. They have even been suggested to play an up to now largely neglected, but potentially very significant role in speech recognition \citep{sabanal_fractal_1996} and natural language processing \citep{inoue_transient_2021}.
Hence, in almost any practical setting, chaotic phenomena abound. They cannot, in general, be ignored when devising RNN training algorithms.

Here we offer a comprehensive theoretical treatment of the relation between RNN dynamics and the behavior of the loss gradients during training. We find a close connection between an RNN's loss gradients and the largest Lyapunov exponent of its freely generated orbits. 
We mathematically prove that RNNs producing stable fixed point or cyclic behavior have bounded gradients. Crucially, however, the loss gradients of RNNs producing chaotic dynamics always diverge. 
Hence, the chaotic nature of many time series data induces a \textit{principle} problem, and, 
despite significant efforts in the past to solve the EVGP, 
training RNNs on such data remains an open issue.
We illustrate the implications of our theory for RNN training on several simulated and empirical chaotic time series, and adapt the  
idea of \textit{sparsely forced} Back-Propagation Through Time (BPTT) 
as a simple yet effective remedy that enables to learn the underlying dynamics despite exploding gradients. 
\vspace{-.2cm}
\section{Related works}
\textit{Exploding and vanishing gradients.}
While `classical' remedies of the EVGP \citep{ hochreiter_thesis_1991,bengio1994learning,pascanu_difficulty_2013} rest on purpose-tailored architectures with gating mechanisms, which safeguard information flow across longer temporal distances \citep{hochreiter_long_1997,cho-etal-2014-properties}, 
the focus has recently shifted to simpler RNNs that address the EVGP by restricting the recurrence matrix to be orthogonal \citep{henaff_recurrent_2016,helfrich_orthogonal_2018,jing_gated_2019}, unitary \citep{Arjovsky_2016}, or antisymmetric \citep{chang_antisymmetricrnn_2019}, or by ensuring globally stable fixed point solutions \citep{Kag2020RNNs,kag_time_2021}, for example through co-trained Lyapunov functions \citep{kolter_learning_2019}. However, all these approaches impose strong limitations on the dynamical repertoire of the RNN, enforcing global convergence to fixed points or simple cycles.\footnote{We make this point more formal 
in Appx. \ref{thm_ortho_rnn}.}
In doing so, they drastically reduce the expressiveness of these models  \citep{kerg_non-normal_2019,orhan_improved_2020}. To address this problem, \citet{erichson_lipschitz_2020} somewhat relaxed the constraints on the recurrence matrix by introducing a skew-symmetric decomposition combined with a Lipschitz condition on the activation function. Another recent approach discretizes oscillator ordinary differential equations (ODEs) to arrive at a stable system of coupled \citep{rusch_coupled_2021} or independent \citep{rusch_unicornn_2021} oscillators which increase the RNN's expressiveness while bounding its gradients. By design (and as acknowledged by the authors), neither of these architectures is capable of producing chaotic dynamics, however, as the underlying ODEs do not allow for exponential divergence of close-by trajectories (a prerequisite for chaos). Given these often principle limitations of parametrically or dynamically strongly constrained models, a fruitful direction may be to modify the training process itself, e.g. through modified or auxiliary loss functions \citep{trinh2018learning,schmidt_identifying_2021}, or special procedures for parameter updating \citep{pmlr-v139-kag21a} or loss truncation \citep{williams_learning_1989,menick2021practical}. Our empirical evaluation will follow up on such ideas, but also highlight that simple loss truncation, windowing, or architectural solutions like LSTMs are not sufficient. 

\textit{Learning dynamical systems.}
Surprisingly disconnected from the work on the EVGP and learning long-term dependencies, a huge and long-standing literature deals with training RNNs on nonlinear dynamical systems (DS) \citep{pearlmutter_1990,trischler_synthesis_2016,vlachas_learning_2020}, including chaotic systems like the famous Lorenz equations \citep{lorenz_deterministic_1963} or chaotic turbulence in fluid dynamics \citep{li_fourier_2021,Pathak18}. 
Of those, methods based on reservoir computing \citep{Pathak18} are special in that they start with a large complex dynamical repertoire to begin with for which an optimal mapping onto the observations is learned, together with a feedback to the reservoir (see also \citep{brenner22} for issues associated with this strategy in the context of DS reconstruction). 
Teacher forcing (TF; \citep{williams_learning_1989,pearlmutter_1990,doya_bifurcations_1992}, see also \citep{Goodfellow-et-al-2016}) is one of the earliest techniques introduced to keep RNN trajectories on track while training. 
The idea behind TF is to simply replace RNN states by observations when available, thereby also effectively cutting off the gradients. 
TF essentially derives from ideas in dynamical control theory, and adaptive schemes that increasingly hand over control to the RNN throughout training have been devised \citep{abarbanel_predicting_2013,abarbanel_machine_2018,bengio_scheduled_2015}.
A related technique from the control theory literature is ``multiple shooting'' \citep{voss_nonlinear_2004}: Here the whole observed time series is chopped into chunks, and for each chunk of trajectory a new initial condition is estimated. Explicit constraints ensure continuity between the separate trajectory bits during optimization. State space models and the Expectation-Maximization algorithm became popular particularly in the 90es for uncovering the latent dynamics underlying a set of time series observations \citep{ghahramani_learning_1998}, and remain an important tool until today \citep{durstewitz_advanced_2017,Koppe_nonlin_2019}. Most recently, approaches based on variational inference and the reparameterization trick \citep{kingma_auto-encoding_2014}, like sequential variational autoencoders (SVAE), gained in popularity for DS approximation \citep{hernandez_nonlinear_2020,bommer2021identifying}. 
``Deterministic'' RNNs (i.e., with latent states not treated as random variables), like conventional LSTMs \citep{vlachas_data-driven_2018}, remain top choices for DS reconstruction, however.  

Although connections between DS ideas and loss gradients have been drawn early on \citep{bengio1994learning}, so far only particular scenarios (like fixed point attractors) have been considered. Closest to our work is recent work by \citet{schmidt_identifying_2021}, where non-divergence of loss gradients is established when RNNs converge to fixed points or cycles. However, this was done only for the particular class of piecewise-linear RNNs (PLRNNs), more restrictive conditions for cycles were imposed than assumed here, and - most importantly - the chaotic case on which we focus here was not considered. 
Recent studies
\citep{engelken_lyapunov_2020,vogt_2022} also point out the general connections between Lyapunov exponents and loss gradients that we develop in sect. \ref{prelim}, 
but do not provide any in-depth theoretical treatment, proofs,
or empirical evaluation of methods to alleviate exploding gradients in training, as we do here. Thus, a systematic theoretical framework that relates RNN dynamics more generally, and across a range of different RNN architectures, to the behavior of its training gradients, is still lacking so far. 
\section{Theoretical analysis: Relation between RNN dynamics and loss gradients}\label{sec:3}
In our analysis, we will cover all major types of 
asymptotic dynamics (fixed points, cycles, chaos, and quasi-periodicity), and mathematically investigate their implications for the loss gradients. We will do this for all major classes of RNNs, including standard RNNs with largely arbitrary activation function, LSTMs, GRUs, and PLRNNs. 
The next section will first develop and illustrate the basic intuition behind the relations between RNN dynamics and loss gradients. 
\subsection{Preliminaries: RNN dynamics and loss gradients}\label{prelim}
Formally, all popular RNN architectures, including LSTMs, GRUs, or PLRNNs, are discrete time 
DS, defined by a (first-order-Markovian) recursive prescription for the temporal evolution of the latent states $\vz_t\in \mathbb{R}^M$ of the general form 
\begin{align}\label{eq-rnn}
\vz_t = F_{\boldsymbol\theta}(\vz_{t-1},\vs_{t}), \, 
\end{align}
where $\vs_t \in  \mathbb{R}^N$ is the input at time $t$ and $\boldsymbol\theta$ are RNN parameters. Map $F_{\boldsymbol\theta}$ may be instantiated by 
any of the common RNN architectures: For instance, for standard RNNs we have $F_{\boldsymbol\theta}(\vz_{t-1},\vs_{t}) = f\big( \mW \vz_{t-1} \,+\, \mB \vs_t \, + \,\vh \big)$, where $f$ is an element-wise activation function like $\tanh$ or a rectified linear unit (ReLU), $\mW$ a connection matrix, matrix $\mB$ weighs the inputs, and $\vh$ is the usual bias term (see sect. \ref{sec:plrnn}--\ref{thm_ortho_rnn} for the definition of other RNN models explored here). 

Assuming we start at some initial value $\vz_1 \in \mathbb{R}^M$, and given a sequence of external inputs $\mS=\{\vs_t\}$, we can recursively rewrite eq. \eqref{eq-rnn} as
\begin{align}\label{}
\vz_T = F_{\boldsymbol\theta}(F_{\boldsymbol\theta}(F_{\boldsymbol\theta}(...F_{\boldsymbol\theta}(\vz_{1},\vs_{2})...)))\, =: \, F_{\boldsymbol\theta}^{T-1}(\vz_{1},\{\vs_{t}\}).
\end{align}
In DS theory, we characterize the long-term behavior of such sequences by its spectrum of Lyapunov exponents. The Lyapunov exponents estimate the exponential growth rates in different local directions of the system's state space, and the largest Lyapunov exponent gives the dominant exponential behavior. Let us denote the system's Jacobian at time $t$ by 
\begin{align}\label{eq:jacobian}
\mJ_t \, := \, \frac{\partial F_{\boldsymbol\theta}(\vz_{t-1},\vs_{t})}{\partial \vz_{t-1}} = \,  \frac{\partial \vz_t}{\partial \vz_{t-1}}.
\end{align}
Then, the maximum Lyapunov exponent along an RNN trajectory $\{\vz_1,\vz_2,\cdots,\vz_T, \cdots\}$ is defined as
\begin{align}\label{eq:lyap}
\lambda_{max} \,:= \, 
\lim_{T \rightarrow \infty} \frac{1}{T} 
\log   \norm{\prod_{r=0}^{T-2}  \mJ_{T-r} },
\end{align}
where ${\rVert \cdot \lVert}$ denotes the spectral norm (or any subordinate norm) of a matrix. If $\lambda_{max}<0$ nearby trajectories will ultimately converge to a fixed point or cycle, while for $\lambda_{max}>0$ (a necessary condition for chaos) initially nearby trajectories will exponentially separate, i.e. we will have divergence along one (or more) directions in state space. This accounts for the sensitive dependence on initial conditions in chaotic systems.
Now let $\mathcal{L}(\mW,\mB,\vh)$ 
be some loss function employed for RNN training that decomposes in time as $\mathcal{L}\,= \,\sum_{t=1}^T \mathcal{L}_t$.
Suppose we fancy BPTT as our training algorithm (similar derivations could be performed for Real Time Recurrent Learning [RTRL]), we recursively develop the loss gradients w.r.t. some RNN parameter $\theta$ 
in time (i.e., across layers of the RNN unrolled in time) as 
\begin{align}\label{eq:loss_grad}
%
\frac{\partial \mathcal{L}}{\partial \theta} \,=\, \sum_{t=1}^T \frac{\partial \mathcal{L}_t}{\partial\theta} \quad \textrm{with} \quad
\frac{\partial \mathcal{L}_t}{\partial \theta}
\, = \, \sum_{r=1}^t \frac{\partial \mathcal{L}_t}{\partial \vz_t}\, \frac{\partial \vz_t}{\partial \vz_r} \, \frac{\partial^{+} \vz_r}{\partial \theta}, 
\end{align}
and
\vspace{-.3cm}
\begin{align}\nonumber
 \frac{\partial \vz_t}{\partial \vz_r} &\, = \,  \frac{\partial \vz_t}{\partial \vz_{t-1}}\, \frac{\partial \vz_{t-1}}{\partial \vz_{t-2}}\, \cdots \, \frac{\partial \vz_{r+1}}{\partial \vz_{r}} 
\\[1ex]\label{eq:prod_jacob}
& \, = \, \prod_{k=0}^{t-r-1}  \frac{\partial \vz_{t-k}}{\partial \vz_{t-k-1}} \, = \, \prod_{k=0}^{t-r-1} \mJ_{t-k},
\end{align}
where $\partial^{+}$ denotes the immediate derivative. Now observe that the behavior of the loss gradients crucially depends on the \textit{product series} of Jacobians in eqn.\eqref{eq:prod_jacob}: 
If the maximum absolute eigenvalues of the Jacobians $\mJ_{t}$ will, in the geometric mean, be larger than $1$ (i.e., $\norm{\prod_{r=0}^{T-2} \mJ_{T-r} }^{1/T}>1$), gradients will explode as $T \rightarrow \infty$, while they will saturate if $\norm{\prod_{r=0}^{T-2} \mJ_{T-r} }^{1/T}<1$. Thus, the key point to note is that the same terms that occur in the definition of the Lyapunov spectrum, eqn.\eqref{eq:lyap}, resurface in the loss gradients, eqn. \eqref{eq:loss_grad} \& \eqref{eq:prod_jacob}. This accounts for the tight links between system dynamics and gradients.
\subsection{Fixed points and cyclic dynamics}
%
Let us start by considering the simplest types of limit dynamics that can occur in RNNs (or any discrete-time DS): fixed points and cycles. In fact, by far most of the literature on global stability in RNNs and on loss gradients focused on just fixed points \citep{chang_antisymmetricrnn_2019,kolter_learning_2019, erichson_lipschitz_2020}, with only few authors who recently started to also connect cyclic behavior to loss gradients \citep{schmidt_identifying_2021,rusch_coupled_2021}. Recall that a fixed point of a recursive map $\vz_t=F(\vz_{t-1})$ is defined as a point $\vz^*$ for which we have $\vz^*=F(\vz^*)$.\footnote{From here on we will suppress the explicit dependence on external inputs $\vs_t$ notation-wise, see Remark \ref{rm-forcing}.} Likewise, a $k$-cycle ($k>1$) is a set of temporally consecutive periodic points $P_k:=\{\vz_{t_1}, \vz_{t_2}, \dots, \vz_{t_k}\}=\{\vz_{t_1}, F(\vz_{t_1}), \dots, F^{k-1}(\vz_{t_1})\}$ that we obtain from recursive application of the map such that each of the cyclic points $\vz_{t_r} \in P_k$ is a fixed point of the $k$ times iterated map $F^{k}$ (with $k$ being the smallest positive integer for which this holds). 
To simplify the subsequent treatment, we will collectively refer to fixed points and cycles as $k$-cycles ($k \geq 1$). Further recall that a fixed point or $k$-cycle is called \textit{stable} if the maximum absolute eigenvalue of the Jacobian evaluated at that point is smaller than $1$, \textit{neutrally stable} if exactly $1$, and \textit{unstable} otherwise. Although the results we develop in this and the following sections will hold more widely, we will restrict our attention to recursive maps $F_{\boldsymbol\theta}$ from the class of RNNs $\mathcal{R}=\{\texttt{standardRNN}, \,\texttt{LSTM},\,\texttt{GRU},\, \texttt{PLRNN}\}$ (see Appx. \ref{pre-all} for details).

Based on the observations made in the previous sections we can state the following theorem that links RNN dynamics and loss gradients:
\begin{theorem}\label{thm-k-2}
Consider an 
RNN $F_{\boldsymbol\theta} \in \mathcal{R}$ parameterized by $\boldsymbol\theta$, 
and assume that it converges to a stable fixed point or $k$-cycle $\Gamma_k$ $(k \geq 1)$ with $\mathcal{B}_{\Gamma_{k}}$ as its basin of attraction. 
Then for every $\vz_1 \in \mathcal{B}_{\Gamma_{k}}$ (i) the Jacobian $\frac{\partial \vz_T}{\partial \vz_1}$ exponentially vanishes as $T \to \infty$; (ii) for $\Gamma_{k}$ the tangent vectors $\frac{\partial \vz_T}{\partial \theta}$
and thus the gradient of the loss function, $\frac{\partial \mathcal{L}_T}{\partial \boldsymbol{\theta}}$, 
will be bounded from above, i.e. will not diverge for $T \to \infty$; and (iii) for the PLRNN \eqref{eq-plrnn} both $\norm{\frac{\partial \vz_T}{\partial \theta}}$ and  $\norm{\frac{\partial \mathcal{L}_T}{\partial \boldsymbol{\theta}}}$ will remain bounded for every $\vz_1 \in \mathcal{B}_{\Gamma_{k}}$ as $T \to \infty$.
\end{theorem}
\begin{proof}
$(i)\,$ Assume that $\Gamma_k$ is a stable $k$-cycle $(k\geq 1)$ denoted by
\begin{align}\nonumber
\Gamma_k &\, = \, \{\vz_1,\vz_2, \cdots,  \vz_T, \cdots \} \, = \,  \{\vz_{t^{*k}},\vz_{t^{*k}-1}, \cdots,
\\[1ex]\label{gam}
& \hspace{.9cm} \vz_{t^{*k}-(k-1)}, \vz_{t^{*k}},\vz_{t^{*k}-1}, \cdots, \vz_{t^{*k}-(k-1)}, \cdots \}.
\end{align}
Then, the largest Lyapunov exponent of $\Gamma_k$ is given by
\begin{align}\nonumber
 \lambda_{\Gamma_k} &\, = \, \lim_{t\rightarrow\infty} \frac{1}{t} \, \ln \norm{ J^{*}_{t}\, J^{*}_{t-1}\, \cdots \,  J^{*}_{2}} 
\\[1ex]\label{LE-2}
 &\,= \, 
\lim_{j\rightarrow\infty} \frac{1}{jk} \, \ln \norm{ \bigg(\prod_{s=0}^{k-1} J_{t^{*k}-s}\bigg)^j}.
\end{align}
By assumption of stability of $\Gamma_k$ we have $\lambda_{\Gamma_k} \, < \, 0$ and also $\rho \big(\prod_{s=0}^{k-1} J_{t^{*k}-s} \big)<1$ 
(the spectral radius), which implies
\begin{align}\label{eq-sefr}
\lim_{t\rightarrow\infty}  J^{*}_{t}\, J^{*}_{t-1}\, \cdots \,  J^{*}_{2} 
 \, = \, 
 \lim_{j\rightarrow\infty} \bigg(\prod_{s=0}^{k-1} J_{t^{*k}-s}\bigg)^j \, = \, 0.
\end{align}
Now suppose that 
$\mathcal{O}_{\vz_{1}}$ is an orbit of 
the map eqn. \eqref{eq-rnn} converging to $\Gamma_{k}$, i.e. $\vz_1 \in \mathcal{B}_{\Gamma_{k}}$. Since $\mathcal{O}_{\vz_{1}}$ and $\Gamma_{k}$ have the same largest Lyapunov exponent, we have
\begin{align}\label{LE-o_z1}
 \lambda_{\mathcal{O}_{\vz_{1}}} \, = \, \lim_{T\rightarrow\infty} \frac{1}{T} \,  \ln \norm{ J_{T}\, J_{T-1}\, \cdots \,  J_{2}} \,= \, \lambda_{\Gamma_k} \, < \, 0,
\end{align}
and hence for $\vz_1 \in \mathcal{B}_{\Gamma_{k}}$
\begin{align}\label{}
\lim_{T\rightarrow\infty} \norm{\frac{\partial \vz_T}{\partial \vz_1}} \, = \, \lim_{T\rightarrow\infty}  \norm{J_{T}\, J_{T-1}\, \cdots \,  J_{2}} %
 \, = \, 0.
\end{align}
~\\
$(ii)\,  \& \, (iii) \,$ See Appx. \ref{p-thm-k-2}.
\end{proof}
\begin{remark}\label{rm-genmaps}
The result of Theorem \ref{thm-k-2} part $(i)$ will be generally true for any first-order-Markovian recursive map \eqref{eq-rnn},  
but the conclusions in part $(ii)$ may hinge on its specific definition.
\end{remark}
\begin{remark}\label{rm-forcing}
None of the results above and throughout sect. \ref{sec:3} require the dynamics to be autonomous, the theory applies whether there is external input or not. In fact, mathematically, non-autonomous (externally forced) systems can always be rewritten as autonomous dynamical systems \citep{alligood1996,perko2001,zhang2009}, see Appx. \ref{auto_to_nonauto} for details.
\end{remark}
The results above ensure that loss gradients will not diverge (explode) as $T \rightarrow \infty$ in RNNs that are ``well-behaved'' in the sense that they converge to a fixed point or cycle (see Fig. \ref{fig:EVGPillu}a). This is a generalization of the results given in Theorem 1 in \citet{schmidt_identifying_2021}, where this was shown only a) for the specific class of PLRNNs and b) for specific constraints imposed on the eigenvalue spectrum of the RNN's Jacobians which were relaxed in our theorem above. 

While our treatment above is centered on the ``exploding-gradients''
case, various architectural modifications or regularization techniques can ensure that gradients do not vanish either, i.e. remain bounded from below as well. This was established, for instance, in \citet{schmidt_identifying_2021} for PLRNNs using 'manifold attractor regularization'. In Appx. \ref{p-thm-k-2} we show that the results from Theorem 2 from \citet{schmidt_identifying_2021} on \textit{doubly} bounded (from below and above) loss gradients can indeed be extended to the more general case covered by Theorem \ref{thm-k-2} above.
\subsection{Chaotic dynamics}
\begin{figure}[h]  
\centering    
\subfigure{\label{fig:a:1}\includegraphics[width=0.4\linewidth]{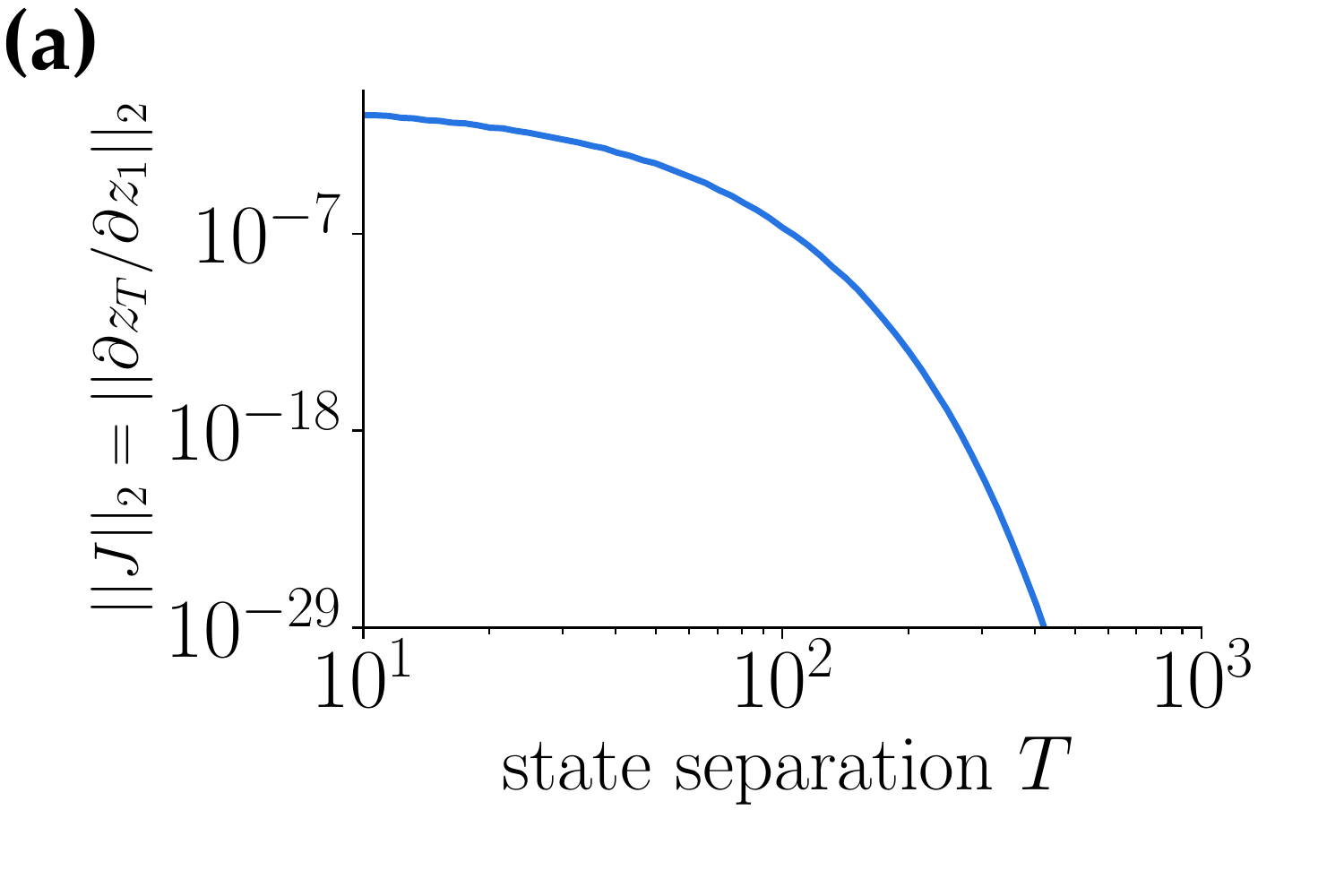}}
\subfigure{\label{fig:b:1}\includegraphics[width=0.4\linewidth]{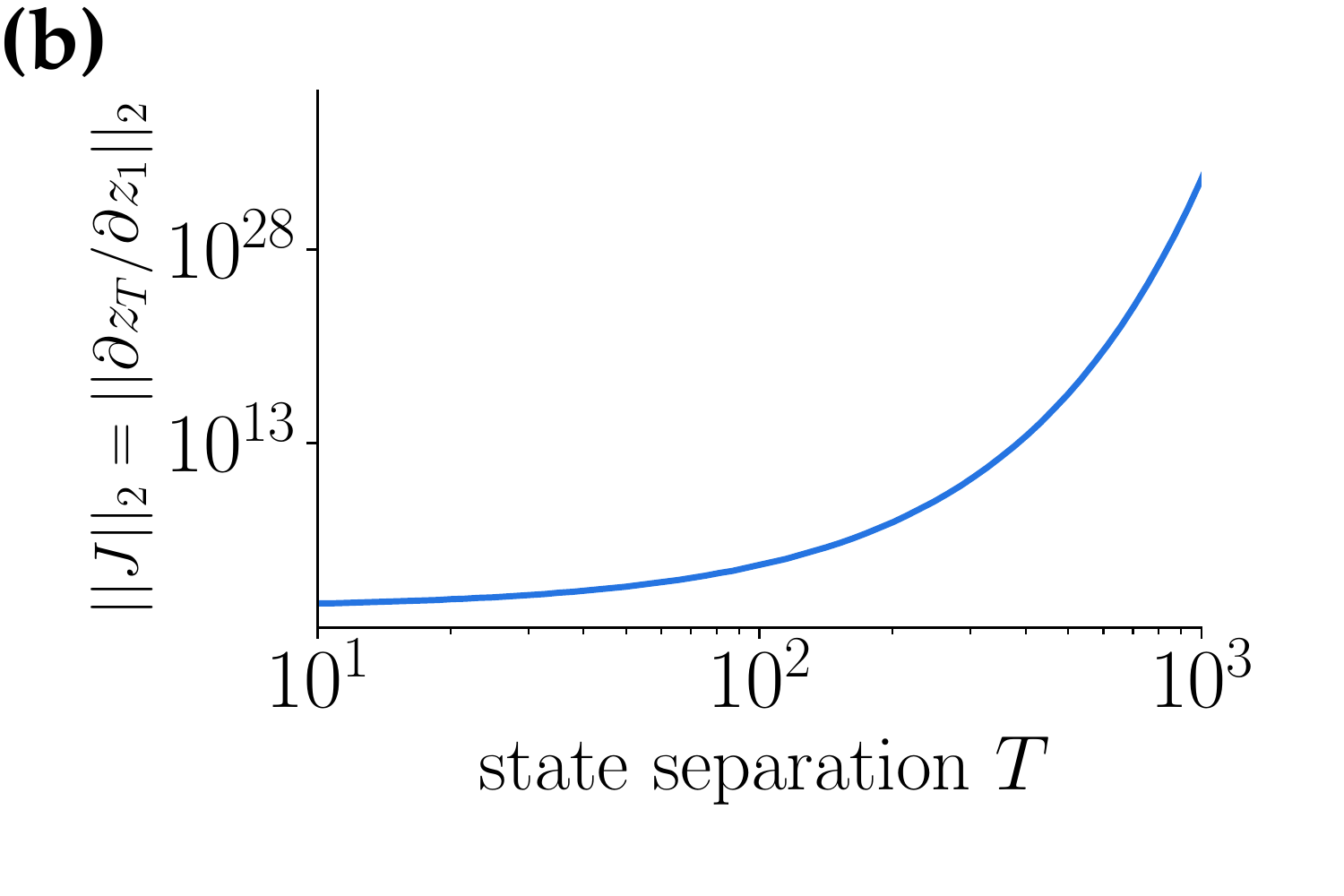}}
\vspace{-0.3cm}
\caption{\small 
Illustration of the 
exploding gradient problem when training RNNs on dynamical systems. 
Jacobians (a) decay away across time separation $T$ for an RNN trained on a simple cycle (cf. Thm. \ref{thm-k-2}), but (b) quickly shoot through the roof when training was performed on chaotic time series (Lorenz system; cf. Thm. \ref{thm-3}). Note the doubly-logarithmic scale of these graphs.
}\label{fig:EVGPillu}
\normalsize
\end{figure}

We will now consider the all-important chaotic case. 
Let $F$ be a recursive map and $\mathcal{O}_{\vz_{1}}\, = \, \{\vz_1, \vz_2, \vz_3, \cdots \}$  be an 
orbit of $F$. The orbit is chaotic if (i) it is not asymptotically periodic and (ii) has at least one positive Lyapunov exponent \citep{gle,mei}. 
If the system's invariant set is \textit{bounded}, condition (ii) is considered a standard signature of chaos, as in this case 
two nearby orbits 
separate exponentially fast, but at the same time their mutual separation cannot go to infinity so that there are also folds. 
The following theorem states the sufficient condition for exploding gradients:
\begin{theorem}\label{thm-3}
Suppose that 
an 
RNN $F_{\boldsymbol\theta} \in \mathcal{R}$ (parameterized by $\boldsymbol\theta$) has a chaotic attractor $\Gamma^{*}$ with $\mathcal{B}_{\Gamma^{*}}$ as its basin of attraction. 
Then, for almost every orbit 
with $\vz_1 \in \mathcal{B}_{\Gamma^{*}}$, (i) the Jacobians connecting temporally distal states $\vz_T$ and $\vz_t$ ($T \gg t$), $\frac{\partial \vz_T}{\partial \vz_t}$, will exponentially explode for $T\rightarrow\infty$, and (ii) the tangent vector $\frac{\partial \vz_T}{\partial \theta}$ and so the gradients of the loss function, $\frac{\partial \mathcal{L}_T}{\partial \theta}$, will diverge as $T\rightarrow\infty$.
\end{theorem}
\vspace{-.3cm}
\begin{proof}
Let the RNN $F_{\boldsymbol\theta} \in \mathcal{R}$ have a chaotic orbit denoted by $\Gamma^{*} = \{\vz_1^{*},\vz_2^{*}, \cdots, \vz^{*}_{T}, \cdots \}$. 
Then, denoting by $\,J^{*}_{T}\,$ the Jacobian of \eqref{eq-rnn} at $\vz^{*}_{T} \in \Gamma^{*}$, the largest Lyapunov exponent of $\Gamma^{*}$ is given by 
\begin{align}\label{LE}
 \lambda\, = \, \lim_{T\rightarrow\infty} \frac{1}{T} \, \ln \norm{ J^{*}_{T}\, J^{*}_{T-1}\, \cdots \,  J^{*}_{2}}.    
\end{align}
Since $\Gamma^{*}$ is chaotic, so $\lambda>0$. Hence, from \eqref{LE}, it is concluded that
\begin{align}\label{yaein}
 \lim_{T\rightarrow\infty}  \norm{ J^{*}_{T}\, J^{*}_{T-1}\, \cdots \,  J^{*}_{2}}\, = \, \lim_{T\rightarrow\infty}  \norm{
\frac{\partial \vz^{*}_T}{\partial \vz^{*}_t}}
 = \infty, \, \, \, \, T \gg t. 
\end{align}
Now, according to Oseledec's multiplicative ergodic Theorem, almost all the points in the basin of attraction of $\Gamma^{*}$ have the same largest Lyapunov exponent $\lambda$. Thus, \eqref{yaein} holds for almost every $\vz_1 \in \mathcal{B}_{\Gamma^{*}}$.
 
$(ii)\,$ See Appx. \ref{p-thm-3}.
\end{proof}
\begin{remark}
The first part of Theorem \ref{thm-3} holds for all first-order-Markovian recursive maps \eqref{eq-rnn}. Note that for LSTMs, $\,\frac{\partial \vz_T}{\partial \vz_t}$ ($\vz:= (\vh, \vc)\tran$) denotes the full Jacobian of both hidden and cell states. 
\end{remark}
We collect some further mathematical results and remarks related to Theorem \ref{thm-3} in Appx. \ref{rem-cor-thm-chaos}.

Hence, the essential result is that for all popular RNNs $\mathcal{R}$ and activation functions, loss gradients will inevitably diverge if the RNN latent states converge to a chaotic attractor (as illustrated in Fig. \ref{fig:EVGPillu}b). 
\subsection{Quasi-periodicity}
%
Quasi-periodicity is a long-term behavior which occurs on a torus and, superficially, bears some similarity to chaos in the sense that, strictly speaking, orbits are also \textit{aperiodic}. That is, as $T\rightarrow\infty$, trajectories will never close up with themselves. 
Moreover, every trajectory becomes arbitrarily close to any point on the torus, that is, it is dense. One important difference between quasi-periodic and chaotic systems is, however, that in a quasi-periodic system, as time passes, two close initial conditions are \textit{linearly} diverging, while in a chaotic system the divergence is exponential. 
%
\begin{theorem}\label{thm-quasi}
Assume that an 
RNN $F_{\boldsymbol\theta} \in \mathcal{R}$ (parameterized by $\boldsymbol\theta$) has a quasi-periodic attractor $\Gamma$ with $\mathcal{B}_{\Gamma}$ as its basin of attraction. Then, for every $\vz_1 \in \mathcal{B}_{\Gamma}$ 
\begin{align}\nonumber
& \forall\, 0<\epsilon<1 \, \, \, \, \exists \,T_0>1 \, \, \, s.t.\, \, \, \forall \,T \geq T_0 \, \, \Longrightarrow \, \,
\\[1ex]\label{re-quasi}
& (1-\epsilon)^{T-1}  \, < \, \norm{\frac{\partial \vz_T}{\partial \vz_1}} \, < \, (1+\epsilon)^{T-1}.    
\end{align}
\end{theorem}
\vspace{-.3cm}
\begin{proof}
See Appx. \ref{p-thm-quasi}.
\end{proof}
\vspace{-.3cm}
According to Theorem \ref{thm-quasi}, for every orbit converging to a quasi-periodic attractor, the Jacobians $\frac{\partial \vz_T}{\partial \vz_t}$ may diverge or vanish as $T\rightarrow\infty$, but this will \textit{not} occur exponentially fast as $T\rightarrow\infty$. 
Thus, even for bounded non-chaotic RNNs we may sometimes stumble into the problem of diverging gradients. Although this may be a less common scenario, we point out it may occur if we train RNNs on real data from oscillatory systems with incommensurate frequencies, as for instance encountered in electronic engineering.

In Appx. \ref{sec-other} we have collected further mathematical results on the connection between RNN dynamics and loss gradients that hold regardless of the RNN's limiting behavior.
\section{Empirical evaluation}\label{EmpSim}
%
Our theoretical results imply that chaotic time series pose a principle challenge for RNN training that cannot easily be circumvented through specifically designed architectures, constraints, or regularization criteria. If the underlying DS we aim to capture is chaotic, loss gradients propagated back in time will inevitably explode. Hence we need to curtail gradients in an ideal way. 
The issue arises especially in scientific ML where time series from chaotic systems are ubiquitous and the aim is to \textit{reconstruct} the generating DS with its limiting behavior. Thus, our exposition will focus on this area.
%
\subsection{Training on systems with exploding gradients by sparse teacher forcing}
\label{Sec:Training}
To illustrate the connections between theory and RNN training, we revive the old idea of TF \citep{williams_learning_1989} as a mechanism for truncating error gradients and keeping model-generated trajectories on track while training. 
However, we would like to do this such that important information about the system dynamics does not get lost, for which Lyapunov theory offers some guidance. Specifically, we should not force the system back onto the true trajectory all or most of the time (as in ``classical TF''), but should effectively ``re-calibrate'' it only at certain time points chosen wisely according to the system's local divergence rates. This procedure will be referred to as \textit{sparsely forced BPTT} in the following.
Assume we want to train an RNN with hidden states $\vz_t  \in \mathbb{R}^M$ and linear (or affine) output layer on a
time-series $\{\vx_1,\vx_2,\cdots,\vx_T\}$ generated by a chaotic system.\footnote{Note that in DS reconstruction one usually considers the data as observations (\textit{unsupervised} problem).} 
The linear output layer $\hat \vx_t = \mB \vz_t$, $\mB \in \mathbb{R}^{N\times M}$, maps the RNN hidden states into the observation space. 
This allows us to modify the original TF procedure by constructing a control series $\{\tilde\vz_1,\tilde\vz_2,\cdots,\tilde\vz_T\}$ from the observations by ``inverting'' the linear output mapping
\footnote{To ensure invertibility, one could add a regularizer $\lambda \mathbf{I}$ to $\mB\tran\mB$ in eqn. \eqref{eq:mod_inv}, as in ridge regression, but we did not find this necessary in any of our examples.}
\begin{align}\label{eq:mod_inv}
\tilde\vz_t = (\mB\tran\mB)^{-1}\mB\tran \vx_t.
\end{align}
The idea is to supply this control signal only sparsely, separated by the learning interval $\tau$ between consecutive forcings. Hence, defining $\mathcal{T}=\{n\tau +1 \}_{n \in \mathbb{N}_0}$ as the set of all time points at which we force the RNN onto the `true' values, the RNN updates can be written as
\begin{align}\label{eq:forcing}
  \vz_{t+1} =
  \begin{cases}
       RNN(\tilde\vz_{t}) & \text{if $t\in \mathcal{T}$} \\
       RNN(\vz_{t}) & \text{else} 
  \end{cases}.
\end{align}
This forcing is applied \textit{after} calculation of the loss, such that $\mathcal{L}_{t}= \norm{\vx_t - \mB\vz_t}_2^2$ irrespective of whether $t$ is in $\mathcal{T}$ or not (and of course it is applied only during training, not at test time!).
Replacing hidden states $\vz_t$ with their teacher-forced signals $\tilde \vz_t$ simply breaks divergence between true and predicted trajectories at time points $t\in \mathcal{T}$, and also cuts off the Jacobians by breaking the temporal contingency (for details see Appx. \ref{supp:loss_trunc}). 
%
 %
%
The learning interval $\tau$ hence controls how many time steps are included in the gradient calculation and has to be chosen with care such as to balance the effects of exploding gradients vs. those of losing relevant time scales and long-term dependencies. While it is general wisdom that an optimal batch size will facilitate training,\footnote{It is also reminiscent of truncated BPTT, but with the all-important differences that we suggest 1) a theoretically informed choice of the optimal ‘truncation length’ (forcing interval) and 2) a specific 
procedure 
for replacing current latent states by control values. 
As shown in sect. \ref{Sec:Example1} \& \ref{Sec:Example2}, both these aspects are indeed crucial to avoid diverging gradients and trajectories 
whilst not loosing relevant longer time scales.} the point here is thus much more specific: Ideally $\tau$ should be chosen in accordance with the system's Lyapunov spectrum, for instance based on the \textit{predictability time} \citep{bezruchko_extracting_2010}
%
\begin{align}\label{eq:pred_time}
\tau_{\text{pred}}\, =\, \frac{\ln 2}{\lambda_{\max}}.
\end{align} 
Various open-source packages exist for calculating the maximal Lyapunov exponent $\lambda_{\max}$ from empirical time series data (e.g., Julia: DynamicalSystems.jl \citep{Datseris2018}, C++: TISEAN \citep{hegger_practical_1999}), potentially after delay-embedding the data (see Appx. \ref{supp:emp-time-series} for more details). Note that this needs to be done on the \textit{empirical} data and only once \textit{before} RNN training, as $\lambda_{\max}$ is an invariant characteristic of the empirical system that we are aiming to reconstruct by our RNN (i.e., after successful training the RNN should have the same invariant properties as the underlying DS).
We also emphasize that such a simple recipe for addressing the exploding gradient problem is based on modifying the training routine, and is thus in principle applicable to any model architecture.
\subsection{Example 1: Lorenz system and externally forced Duffing oscillator in chaotic regime}
\label{Sec:Example1}
\begin{figure*} 
\centering    
\subfigure{\label{fig:a:1}\includegraphics[width=0.40\linewidth]{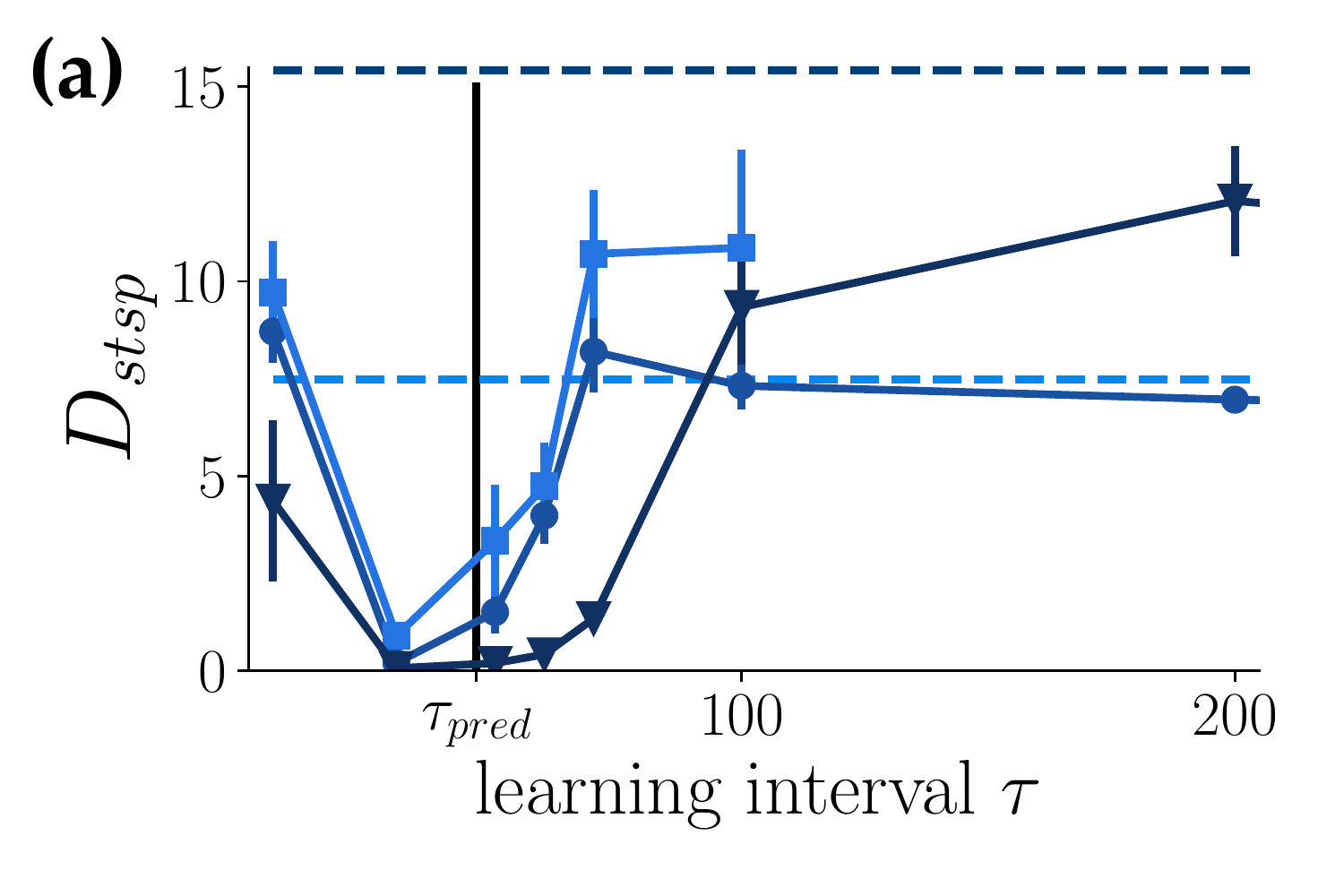}}
\hspace{1cm}
\subfigure{\label{fig:a}\includegraphics[width=0.40\linewidth]{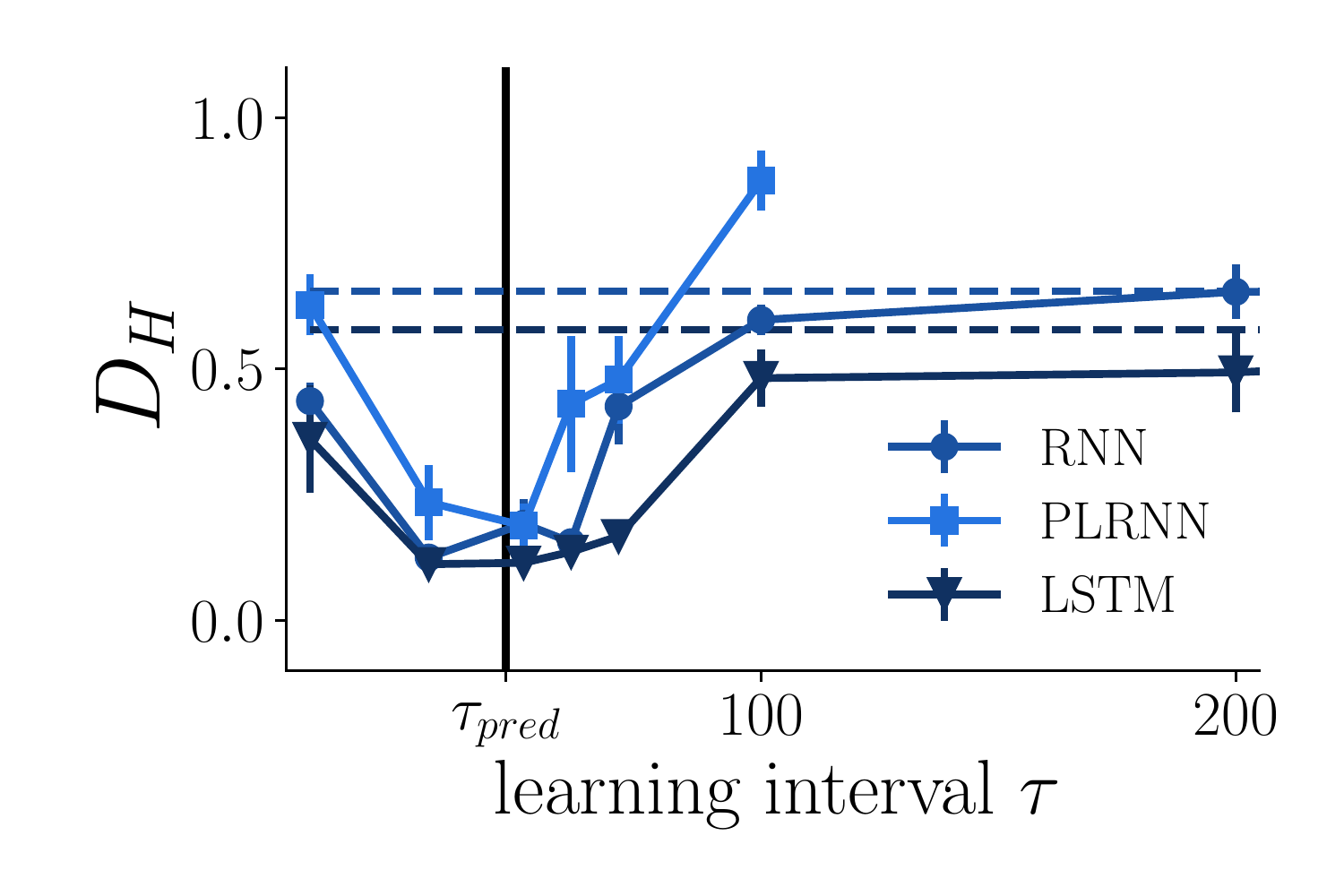}}
\\
\vspace{-0.4cm}
\subfigure{\label{fig:b:1}\includegraphics[width=0.40\linewidth]{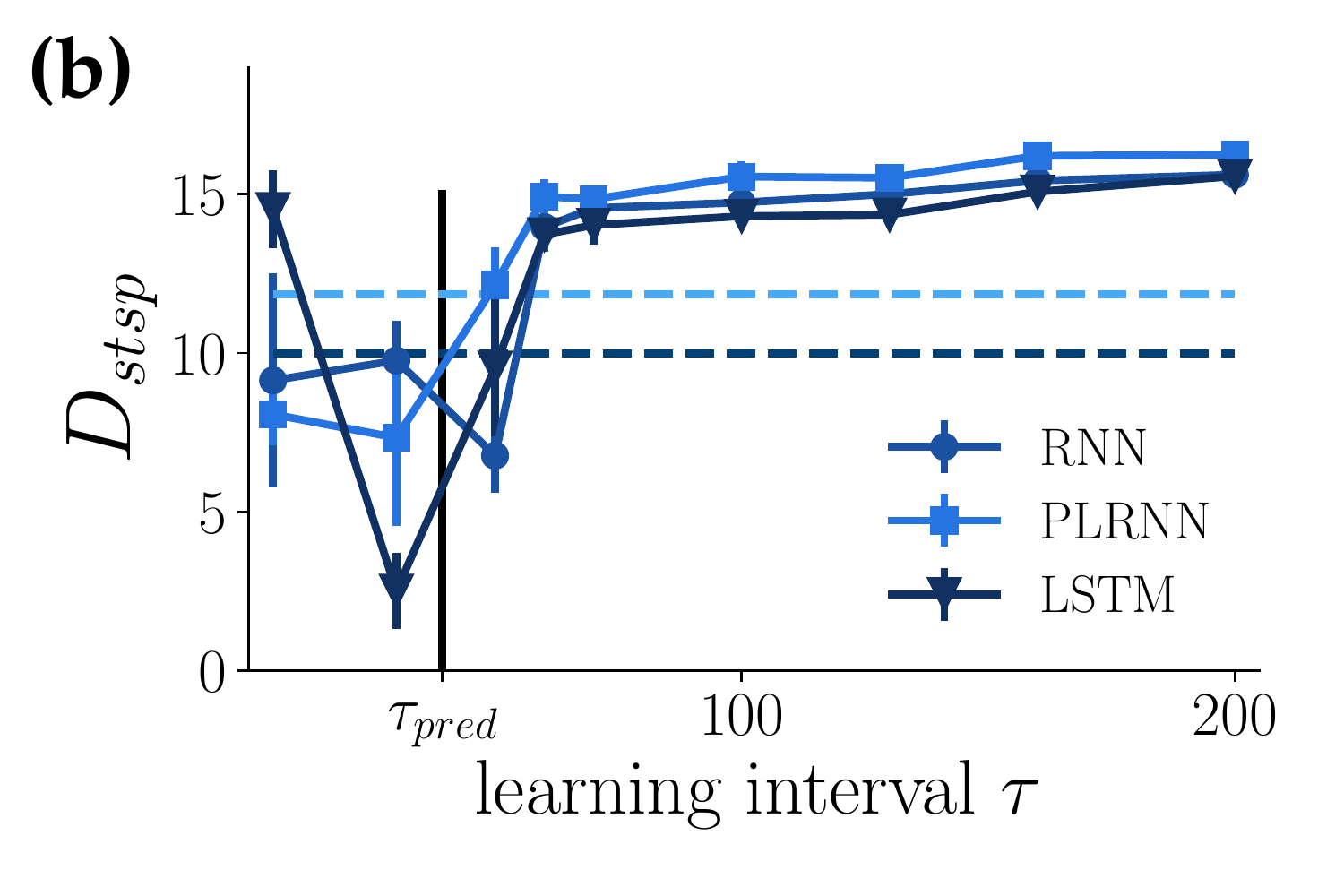}}
\hspace{1cm}
\subfigure{\label{fig:b}\includegraphics[width=0.40\linewidth]{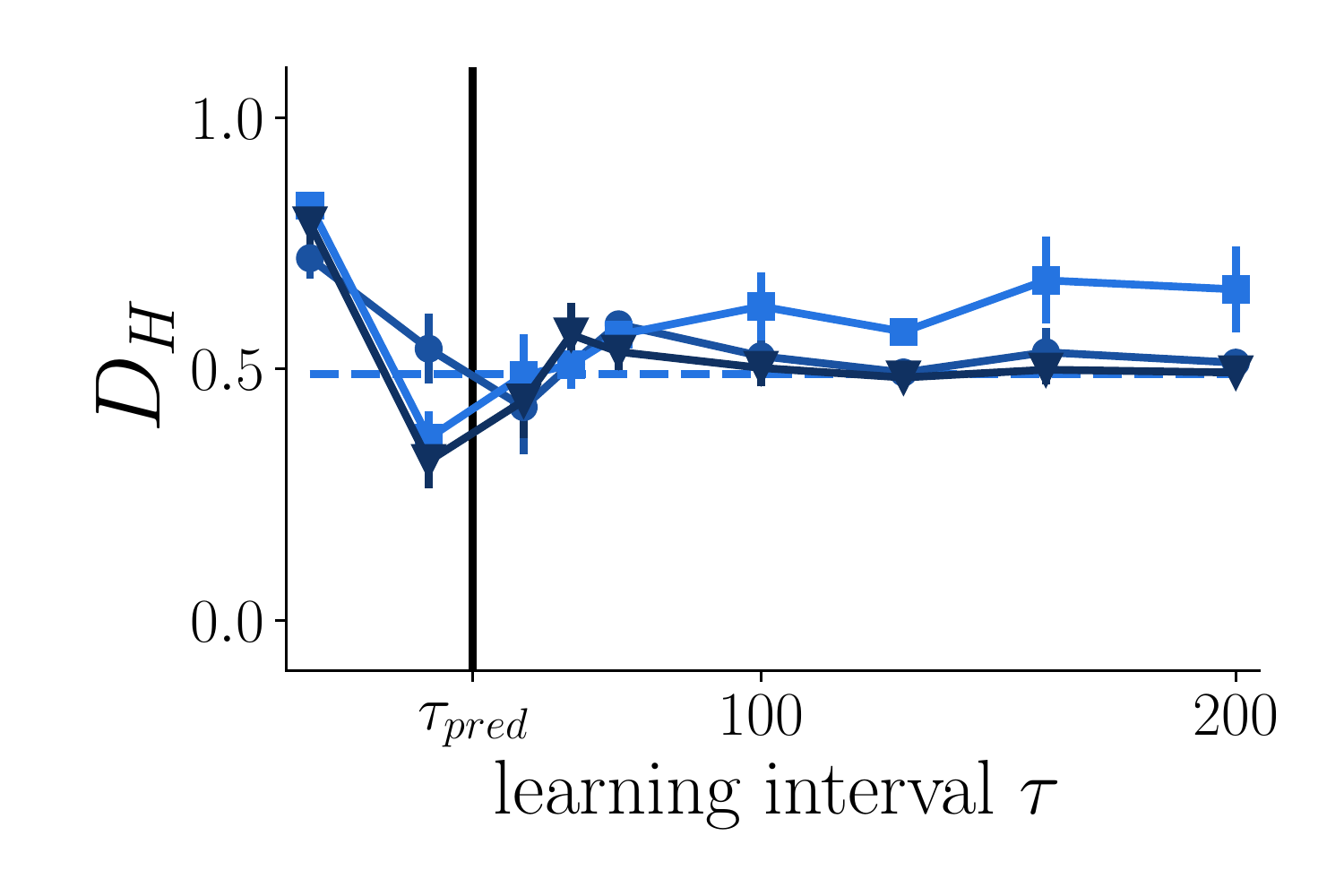}}
\vspace{-0.1cm}
\caption{\small Overlap in attractor geometry ($D_{stsp}$, lower = better) and dimension-wise comparison of power-spectra  ($D_H$, lower = better) against learning interval $\tau$ for (a) the Lorenz and (b) the chaotically forced Duffing oscillator. Continuous lines = 
sparsely forced BPTT. 
Dashed lines = 
classical BPTT with gradient clipping. Prediction time indicated vertically in black.}\label{fig:tau-sweep2}
\normalsize
\end{figure*}
%
 Let us illustrate these ideas on two classical textbook examples of chaotic DS, the chaotic Lorenz attractor as an autonomous system, and the chaotically forced Duffing oscillator as an example with explicit external input (see Appx. \ref{supp:lorenz} for details). Trajectories 
 were repeatedly drawn from these systems, on which we trained a PLRNN, a vanilla RNN with $\tanh$ activation function, and a LSTM 
 by stochastic gradient descent (SGD) to minimize the MSE loss between predicted and actual observations. As optimizer we used Adam \citep{kingma_adam_2017} from PyTorch \citep{Paszke2017AutomaticDI} with a learning rate of $0.001$. For all models, training proceeded solely by \textit{sparsely forced BPTT} and did not employ gradient clipping or any other technique that may interfere with optimal loss truncation.

In nonlinear DS reconstruction, we are mainly interested in reproducing \textit{invariant} properties of the underlying system such as the attractor geometry (or topology \cite{takens_strange_1982,sauer_embedology_1991}) or the frequency composition (i.e., time-independent properties), 
while measures like ahead-prediction errors are less meaningful especially on chaotic time series \citep{wood_statistical_2010,Koppe_nonlin_2019}. 
Thus, in evaluating training performance, here we follow \citet{Koppe_nonlin_2019} in using a Kullback-Leibler divergence $D_{stsp}$ to quantify the agreement between observed and generated probability distributions across state-space to asses the overlap in attractor geometry (Appx. \ref{supp-klx}). 
Moreover, we calculate the dimension-wise
Hellinger distance $D_H$ between power spectra to quantify the temporal agreement of the observed and generated time-series (Appx. \ref{supp-psc}).

Fig. \ref{fig:tau-sweep2} shows the dependence of the reconstruction quality on the learning interval $\tau$ for all RNN architectures on (a) the Lorenz and (b) the externally forced Duffing system. Fig. \ref{fig:lorenz-reconstruction} provides particular examples of reconstructions for $\tau$ chosen too small, too large, or about right.
\begin{figure}[!htb]
\centering 
\subfigure{\label{fig:a:1}\includegraphics[width=0.20\linewidth]{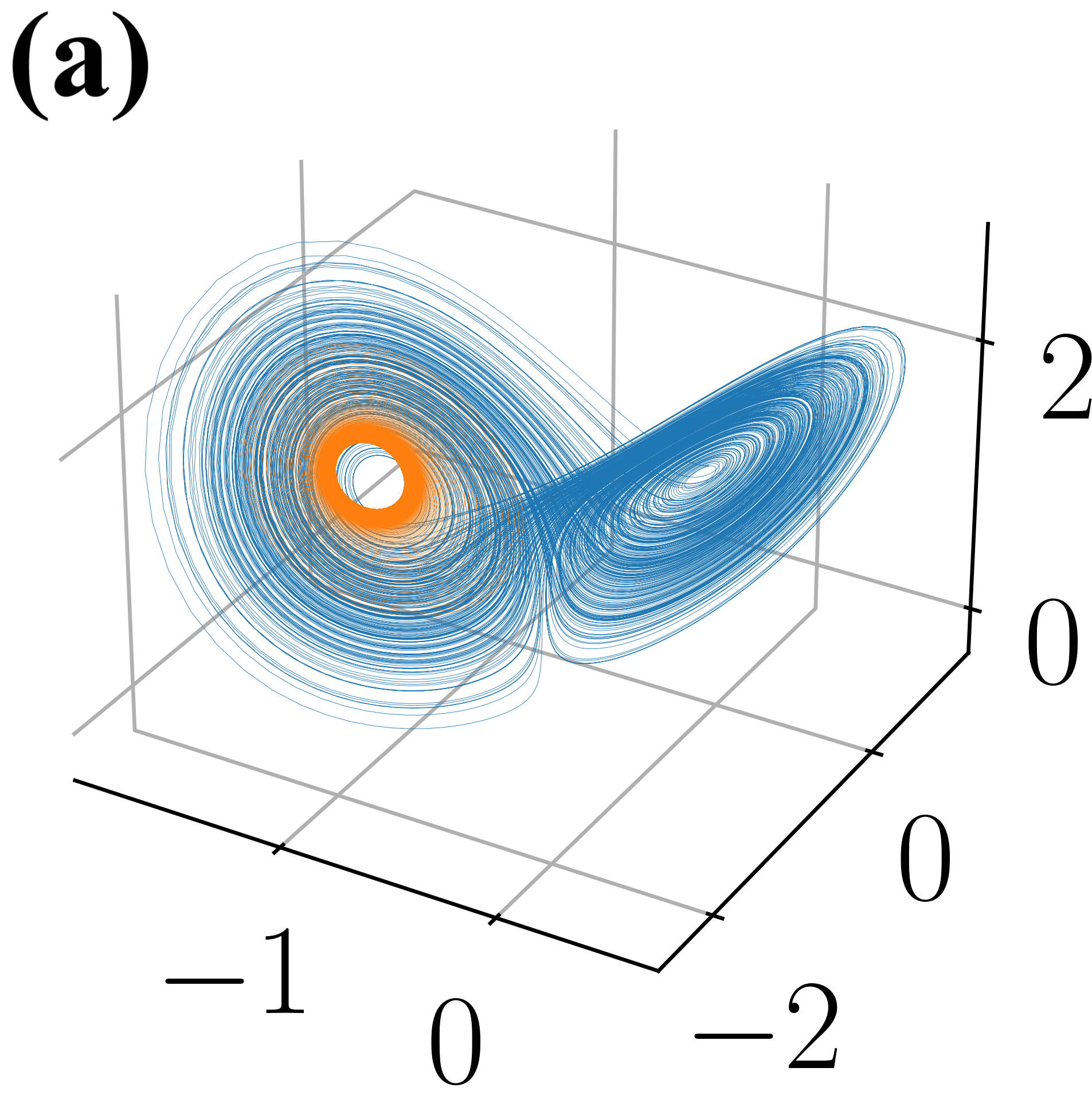}}
\hspace{1.2cm}
\subfigure{\includegraphics[width=0.20\linewidth]{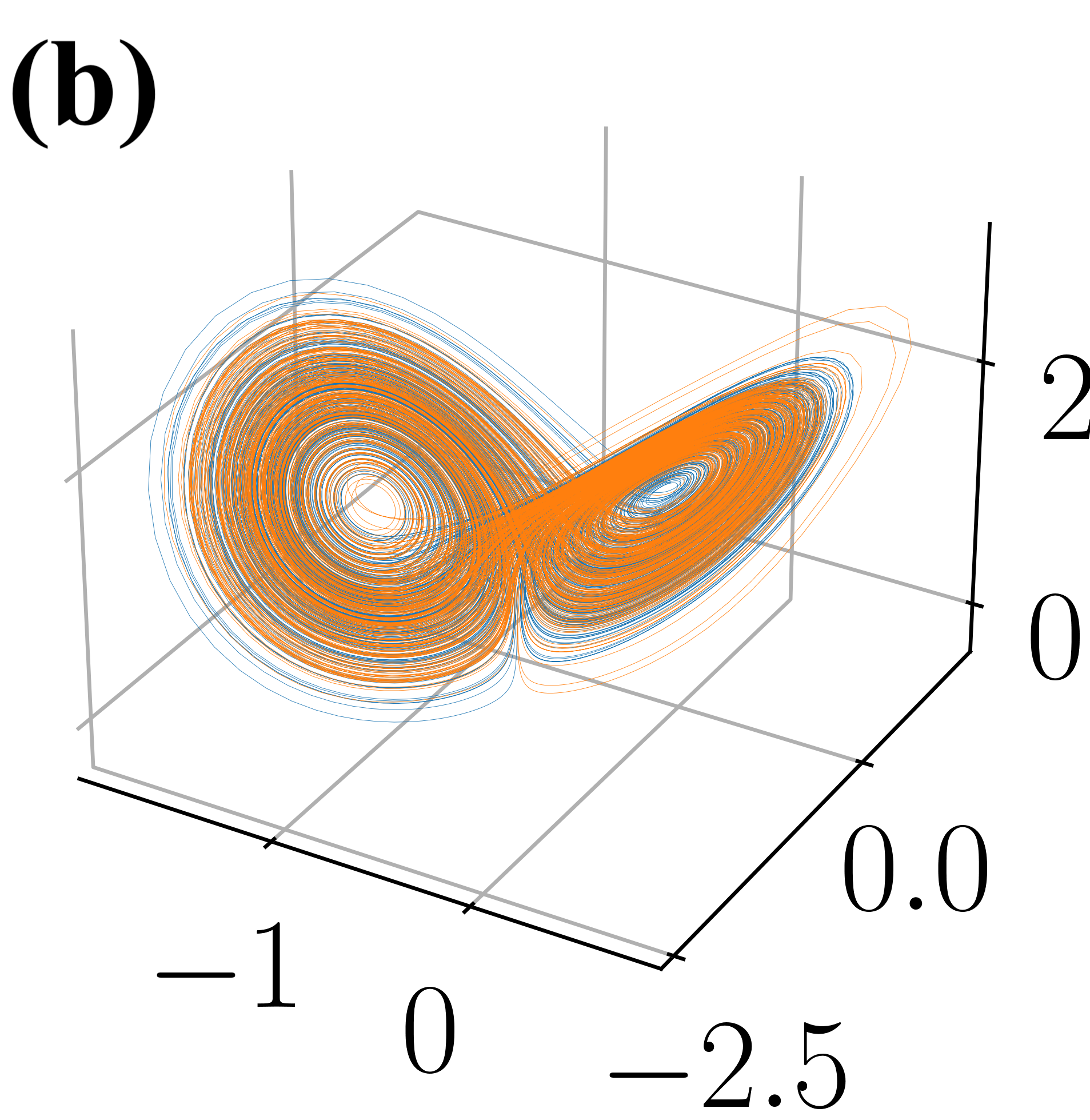}}
\hspace{1.2cm}
\subfigure{\includegraphics[width=0.20\linewidth]{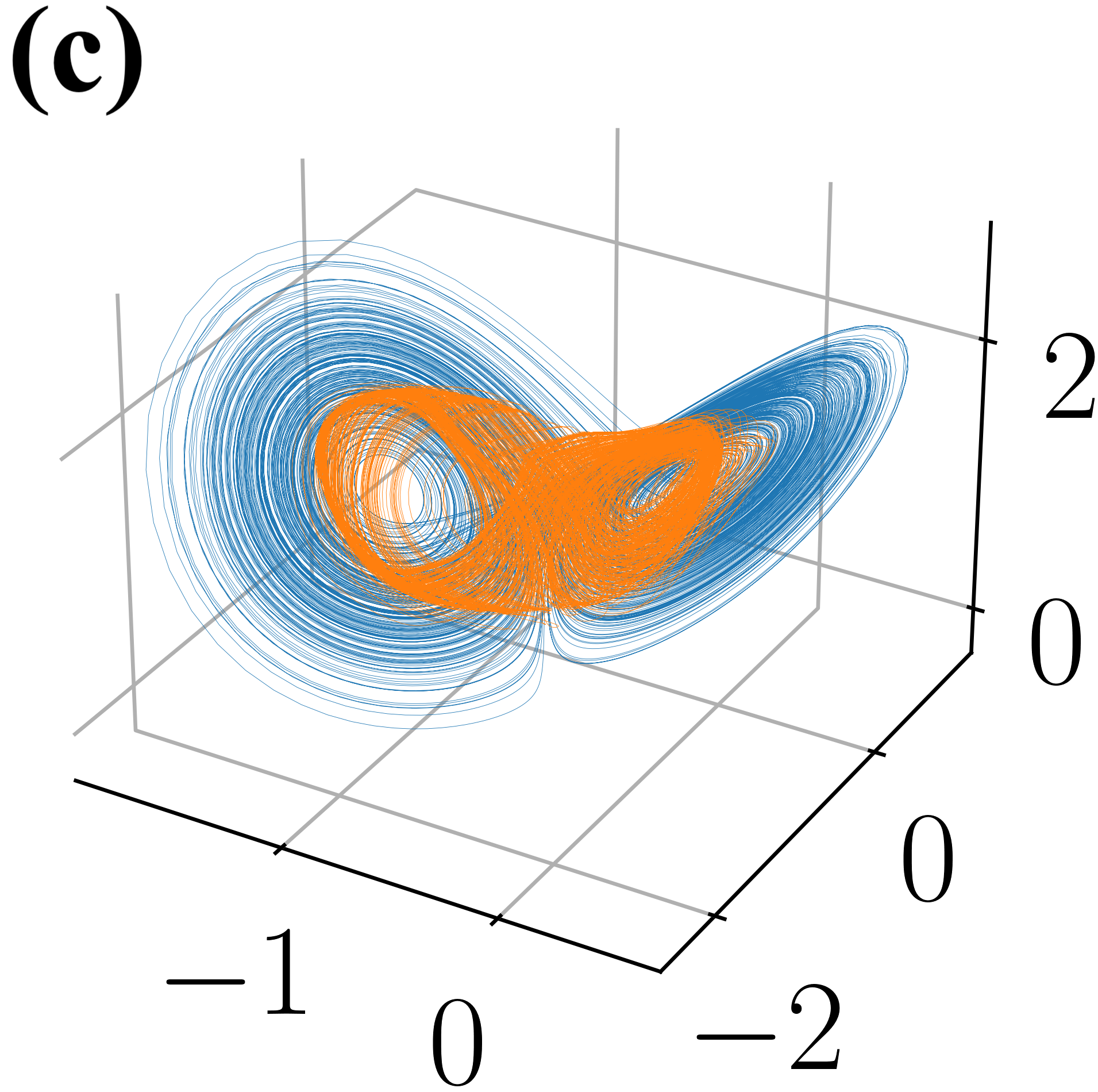}}
\caption{\small Lorenz attractor (blue) and example reconstructions by an LSTM (orange) trained with a learning interval (a) chosen too small ($\tau = 5$), (b) chosen optimally ($\tau = 30$), and (c) chosen too large ($\tau=200$). See Fig. \ref{fig:RNN-reconstruction} for a vanilla RNN example.}\label{fig:lorenz-reconstruction}
\normalsize
\end{figure}
~\\

For all models we find a system-dependent range for the optimal learning interval that agrees well with the predictability time defined in eqn.\eqref{eq:pred_time}, where estimates for the maximal Lyapunov exponent were taken from the literature \citep{rosenstein_practical_1993, gilpin2021chaos}. As a reference, dashed lines represent the reconstruction performance for all architectures when trained with classical BPTT and gradient clipping. The training procedure was the same as for sparsely forced BPTT, except that instead of supplying a control-signal, gradients were normalized to $1$ prior to each parameter update 
(see Fig. \ref{fig:gradClipping} for a more systematic evaluation of different clipping procedures and thresholds). As evidenced by the much worse performance, gradient clipping does not effectively address the EVGP, \textit{even for LSTMs}. As further shown in Fig. \ref{fig:windowing} in Appx. \ref{sec:windowing}, using the optimal 
(or any other) window length $\tau$ but resetting the initial condition for each chunk to 
either zero or the last forward-iterated state $F_{\boldsymbol\theta}(\vz_{t-1})$ 
(instead of its control value $\tilde\vz_{t}$) equally destroys performance. This suggests that neither mere gradient normalization nor simple windowing are sufficient, but will wipe out essential information about the dynamics.
 
In Appx. \ref{sec-supp-emp} we collect further results on the chaotic Rössler attractor (Fig. \ref{fig:tau-sweep-roessler} \& \ref{fig:roessler-reconstruction}), high-dimensional Mackey-Glass equations (Fig. \ref{fig:tau-sweep-MackeyGlass}), and the Lorenz attractor with partial observations (Fig. \ref{fig:tau-sweep-partObs}).
%
\subsection{Example 2: Chaotic weather data}
\label{Sec:Example2}
\begin{figure*}[!h]
\centering
\subfigure{\label{fig:a:1}\includegraphics[width=0.328\linewidth]{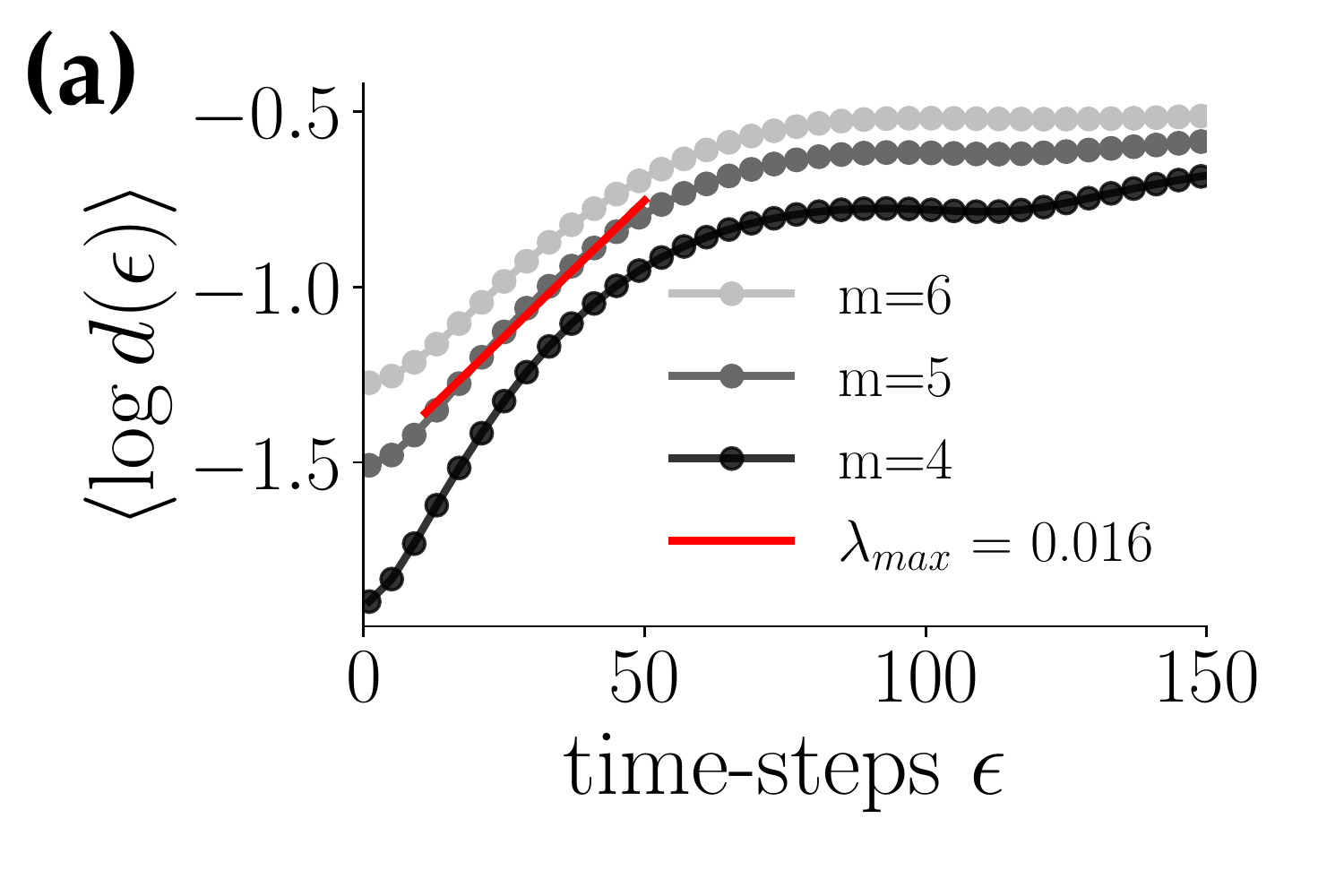}}
\subfigure{\label{fig:b:1}\includegraphics[width=0.328\linewidth]{
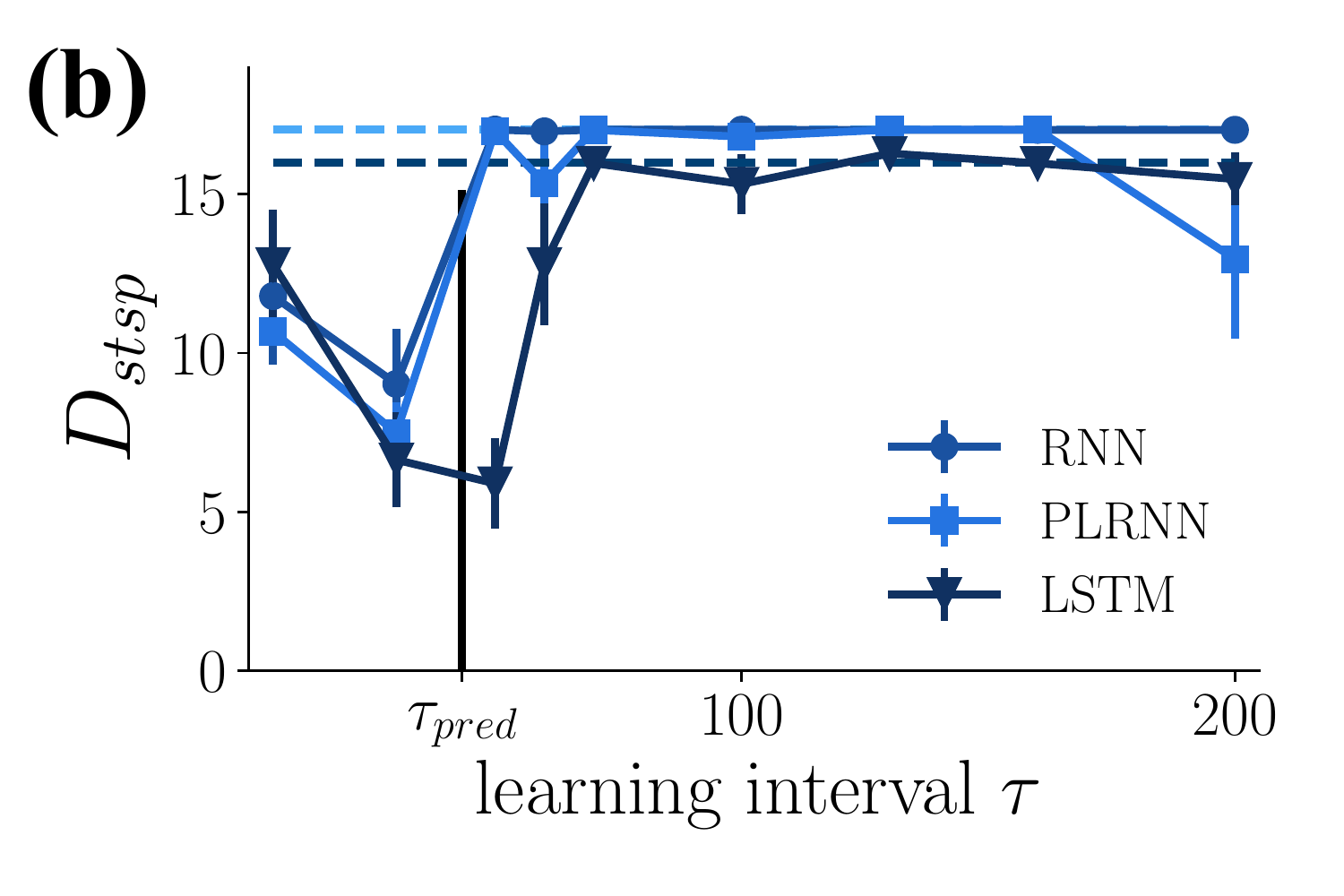}}
\subfigure{\label{fig:a}\includegraphics[width=0.328\linewidth]{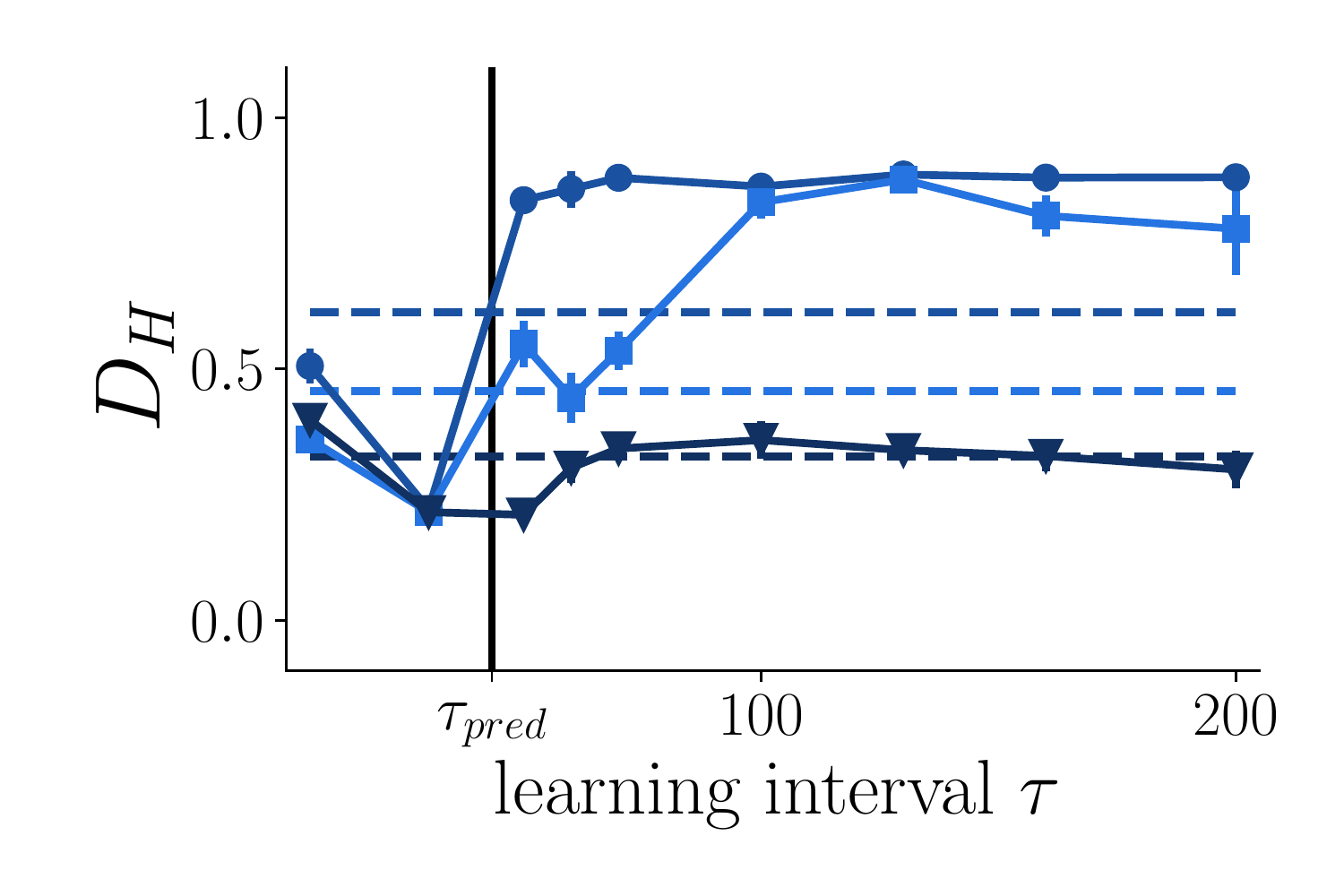}}
\caption{\small (a) The maximal Lyapunov exponent was determined as the slope of the average log-divergence of nearest neighbors in embedding space ($m=$ embedding dimension). (b) Reconstruction quality assessed by attractor overlap (lower = better) and dimension-wise comparison of power-spectra  ($D_H$, lower = better). Black vertical lines = $\tau_{\text{pred}}$. }\label{fig:tau-sweep-emp-data}
\normalsize
\end{figure*}
~\\
As for an empirical example, we trained all RNNs (vanilla RNN, PLRNN, LSTM) on a temperature time series recorded at the \href{https://www.bgc-jena.mpg.de/wetter/}{Weather Station at the Max Planck Institute for Biogeochemistry in Jena, Germany}. 
To expose the chaotic behavior and obtain a robust estimate of the maximal Lyapunov exponent, trends and yearly cycles were removed, and nonlinear noise-reduction was performed (\citep{kantz_nonlinear_1993}; Appx. \ref{supp:emp-time-series}). The maximal Lyapunov exponent was determined with the \textit{TISEAN} package \citep{hegger_practical_1999}, as shown in Figure \ref{fig:tau-sweep-emp-data} (a). The value obtained is in close agreement with the literature 
($\lambda_{max} \approx 0.017$ 
\citep{millan_nonlinear_2010}).

Figure \ref{fig:tau-sweep-emp-data} shows that also for these empirical data the optimal training interval $\tau$ agrees well with the predictability time, eqn.\eqref{eq:pred_time}, for all trained RNNs. Furthermore, as was the case for the DS benchmarks, gradient clipping was not able to satisfactorily tackle the EVGP, even when paired with architectures like LSTMs 
explicitly designed for alleviating this problem. Similar results are reported for another real-world dataset,  electroencephalogram (EEG) recordings, in Appx. \ref{Sec:EEG-analysis}.
\section{Discussion and conclusions} 
\label{Sec:Disc}
    In this paper we proved that RNN dynamics and loss gradients are intimately related for all major types of RNNs and activation functions. If the RNN is ``well behaved'' in the sense that its dynamics converges to a fixed point or cycle, loss gradients will remain bounded, and established remedies
    \citep{hochreiter_long_1997,schmidt_identifying_2021} can be used to refrain them from vanishing. However, if the dynamics are chaotic, gradients will always explode. This constitutes a \textit{principle} problem in RNN training that cannot easily be mastered through architectural design or gradient clipping. This is because to avoid exploding gradients while training on time series from chaotic systems, one either needs to constrain the RNN so much that chaotic behavior is completely disabled to begin with (i.e., ultimately by forcing all Lyapunov exponents to be smaller or equal to zero), implying a very poor fit to such data. Or one needs to be a bit more lenient and thereby allow for the possibility of exploding gradients (as LSTMs or PLRNNs in fact do). This problem is furthermore practically highly relevant, as most time series we encounter in nature, and many from man-made systems as well, are inherently chaotic. 

    While we do not offer a full solution to this problem here, we suggest it might be tackled in training by taking a system's local divergence rates as measured through the Lyapunov spectrum into account. Hence, rather than conquering the EVGP by structural design or specific constraints or regularization terms, we recommend to put the focus more on the training process itself. 
    We illustrated this point empirically using \textit{sparsely forced BPTT}, a training technique that pulls trajectories back on track at times determined by the maximal Lyapunov exponent. Doing so leads to optimal reconstruction results for a variety of simulated and real-world benchmarks, regardless of the specific RNN architecture employed in training.     

    As noted in sect. \ref{Sec:Training}, fairly standard packages are available for computing maximal Lyapunov exponents from data. Some background knowledge, as provided in classical textbooks (e.g. Ch. 5 in \citep{kantz_schreiber_2003}), may be required for properly reading the output from these packages: Essentially, one would be looking be for a linear scaling region as in Figs. \ref{fig:tau-sweep-emp-data}a \& \ref{fig:tau-sweep-emp-data-EEG}a, ignoring both the initial noise transient as well as the plateau caused by reaching the full attractor extent. If unsure about the exact value, a moderate amount of jittering around the estimated mean value may help (see Appx. Fig. \ref{fig:random-sparseTF}). A further interesting direction for improvement might be to regulate the forcing interval through an annealing procedure \citep{abarbanel_predicting_2013,abarbanel_machine_2018}, for instance starting at $\tau=1$ and ramping up to $\tau=\tau_{\text{pred}}$ throughout training, similar as in adaptive schemes \citep{bengio_scheduled_2015}.\footnote{We thank one of the referees for pointing this out.} 
    The idea here would be to first get the short-term behavior right, and then challenge the system more and more for longer time spans until the predictability time is reached.

    We stress that our goal here above all was to provide a mathematically grounded perspective on the problem, with the empirical section focused on elucidating the practical implications of the theoretical results. 
    We believe that a more thorough theoretical understanding is important and needed for guiding future research into more powerful training procedures that avoid exploding gradients \textit{without} compromising expressiveness. 
    In our application examples, we developed the case from the perspective of scientific machine learning, which by now is a broad area in its own right with huge societal relevance (e.g., climate or epidemiological time series), and where the reconstruction of geometrical or topological (invariant) properties is important, beyond mere prediction. Nevertheless, we believe that our theoretical results will also have implications for other domains, like NLP \citep{inoue_transient_2021}. While scientific time series problems traditionally have been extensively considered from a DS perspective (e.g., \citep{kantz_schreiber_2003}), much more groundwork is needed, however, in areas like NLP, where, for instance, it may not even be immediately clear how to best define a Lyapunov spectrum.
    

All code from this paper is available at \url{https://github.com/DurstewitzLab/ChaosRNN}.

\section*{Acknowledgements}
This work was funded by the German Research Foundation (DFG) under Germany's Excellence Strategy – EXC-2181 – 390900948 (STRUCTURES), and through grant Du 354/10-1 to DD.

\newpage
\printbibliography{}
\newpage
\begin{enumerate}
\item For all authors...
\begin{enumerate}
  \item Do the main claims made in the abstract and introduction accurately reflect the paper's contributions and scope?
    \answerYes{}
  \item Did you describe the limitations of your work?
    \answerYes{See Sec. \ref{Sec:Disc}. }
   \item Did you discuss any potential negative societal impacts of your work?
    \answerNA{}
  \item Have you read the ethics review guidelines and ensured that your paper conforms to them?
    \answerYes{}
\end{enumerate}

\item If you are including theoretical results...
\begin{enumerate}
  \item Did you state the full set of assumptions of all theoretical results?
    \answerYes{We provide comprehensive proofs of all presented theorems in the Appendix. }
        \item Did you include complete proofs of all theoretical results?
    \answerYes{}
\end{enumerate}

\item If you ran experiments...
\begin{enumerate}
  \item Did you include the code, data, and instructions needed to reproduce the main experimental results (either in the supplemental material or as a URL)?
    \answerYes{}{The code will be made public after publication.}
  \item Did you specify all the training details (e.g., data splits, hyperparameters, how they were chosen)?
    \answerYes{}{}
        \item Did you report error bars (e.g., with respect to the random seed after running experiments multiple times)?
    \answerYes{}
        \item Did you include the total amount of compute and the type of resources used (e.g., type of GPUs, internal cluster, or cloud provider)?
    \answerNA{}
\end{enumerate}

\item If you are using existing assets (e.g., code, data, models) or curating/releasing new assets...
\begin{enumerate}
  \item If your work uses existing assets, did you cite the creators?
    \answerYes{}
  \item Did you mention the license of the assets?
    \answerYes{}
  \item Did you include any new assets either in the supplemental material or as a URL?
    \answerNo{}
  \item Did you discuss whether and how consent was obtained from people whose data you're using/curating?
    \answerNA{}
  \item Did you discuss whether the data you are using/curating contains personally identifiable information or offensive content?
    \answerNA{}
\end{enumerate}

\item If you used crowdsourcing or conducted research with human subjects...
\begin{enumerate}
  \item Did you include the full text of instructions given to participants and screenshots, if applicable?
    \answerNA{}
  \item Did you describe any potential participant risks, with links to Institutional Review Board (IRB) approvals, if applicable?
    \answerNA{}
  \item Did you include the estimated hourly wage paid to participants and the total amount spent on participant compensation?
    \answerNA{}
\end{enumerate}

\end{enumerate}


\newpage
\appendix

\section{Appendix}\label{app}
\subsection{Theorems: Preliminaries}\label{pre-all}
\subsubsection{Transforming non-autonomous into autonomous discrete-time DS}\label{auto_to_nonauto} 
Following \citep{zhang2009}, and based on similar reasoning as for continuous time (ODE-based) DS \citep{alligood1996,perko2001}, let us consider the non-autonomous discrete-time DS 
\begin{align}\label{eq:non-aut}
\vx_{t+1} \,=\, F(\vx_t,t),  \hspace{.5cm} \vx \in \mathbb{R}^{n}.      
\end{align}
Defining $\,\vz_t \,=\, (\vx_t,t)\tran\,$ and $\,G(\vz_t) \,= \,(F(\vx_t,t),t+1)\tran\,$, system \eqref{eq:non-aut} can be rewritten as the autonomous system 
\begin{align}\label{eq:aut}
 \vz_{t+1} = G(\vz_t), \hspace{.5cm} \vz\in \mathbb{R}^{n+1}.
\end{align}
Hence, in all our theoretical treatment we can confine our attention to systems of the form \eqref{eq:aut}. 
\subsubsection{RNN derivatives}\label{def_pre_rnn} 
Considering the loss function $\mathcal{L}\,= \,\sum_{t=1}^T \mathcal{L}_t\,$ of 
an RNN $F_{\boldsymbol\theta} \in \mathcal{R}$ parameterized by $\boldsymbol\theta$, we have
\begin{align}\label{eq:lo}
\frac{\partial \mathcal{L}}{\partial \theta} \,=\, \sum_{t=1}^T \frac{\partial \mathcal{L}_t}{\partial \theta},
\end{align}
where
\begin{align}\label{eq-s-gr}
\frac{\partial \mathcal{L}_t}{\partial \theta} \, = \, \frac{\partial \mathcal{L}_t}{\partial \vz_t} \, \frac{\partial \vz_t}{\partial \theta}.
\end{align}
~\\
The tangent vector $\frac{\partial \vz_T}{\partial \theta}$ has the form
\begin{align}\label{eq:tangent-vec}
 \frac{\partial \vz_T}{\partial \theta}
     \,= \,
     \frac{\partial^{+} \vz_T}{\partial \theta}\, + \, \sum_{t=1}^{T-2} \bigg( \prod_{r=0}^{t-1} \mJ_{T-r} \bigg)\frac{\partial^{+} \vz_{T-t}}{\partial\theta},
\end{align}
where $\partial^{+}$ denotes the immediate partial derivative. Since for an RNN $F_{\boldsymbol\theta} \in \mathcal{R}$ the activation function is element-wise, with $\theta$ the $m$-th element of a parameter vector $\boldsymbol{\theta}$ (or belonging to the $m$-th row of a parameter matrix $\boldsymbol{\theta}$), we have
\begin{align}
\frac{\partial^{+} \vz_T}{\partial \theta} \, = \, \begin{pmatrix}
0 & \cdots & 0 &  \frac{\partial^{+} z_{m,T}}{\partial \theta} & 0 & \cdots & 0
\end{pmatrix}\tran.   
\end{align}
~\\
For instance, let $\boldsymbol\theta \, = \, \mW$ be a weight matrix, then\\
\begin{align}\label{mat-gra}
\frac{\partial \mathcal{L}}{\partial \mW}
\, = \, 
\begin{pmatrix}
\frac{\partial \mathcal{L}}{\partial w_{11}} & \frac{\partial \mathcal{L}}{\partial w_{12}} & \cdots & \frac{\partial \mathcal{L}}{\partial w_{1M}} \\[1ex]
\frac{\partial \mathcal{L}}{\partial w_{21}} & \frac{\partial \mathcal{L}}{\partial w_{22}} & \cdots & \frac{\partial \mathcal{L}}{\partial w_{2M}}
\\[1ex]
\vdots
\\[1ex]
\frac{\partial \mathcal{L}}{\partial w_{M1}} & \frac{\partial \mathcal{L}}{\partial w_{M2}} & \cdots & \frac{\partial \mathcal{L}}{\partial w_{MM}}
\end{pmatrix}.
\end{align}
~\\
In this case, for the standard RNN we have
\begin{align}\label{}
  \frac{\partial^{+} \vz_{T}}{\partial w_{mk}} 
\, = \, 
\begin{pmatrix}
0 & \cdots & 0 &  z_{k,T-1}\,\xi_{mk}(\vz_{T-1}) & 0 & \cdots & 0
\end{pmatrix}\tran
\, = \, \mathbf{1}_{(m,k)}\, \xi_{mk}(\vz_{T-1})\, \vz_{T-1},
\end{align}
%
~\\
where $\, \xi_{mk}(\vz_{T-1})\, = \,f_{w_{m,k}}^{\prime}\big( \sum_{j=1}^{M} w_{mj}\,z_{j,T-1} \,+\, \sum_{j=1}^{M} b_{mj}\,s_{j,T}+ h_m  
\big)\, $, and $f_{w_{m,k}}^{\prime}$ stands for the derivative of $\,f \,$ with respect to $w_{m,k}$.\\

Therefore, for standard RNNs, \eqref{eq:tangent-vec} becomes
\begin{align}\label{eq-total_d}
\frac{\partial \vz_T}{\partial w_{mk}} 
& \, = \, \mathbf{1}_{(m,k)}\, \xi_{mk}(\vz_{T-1})\, \vz_{T-1} \, + \, \sum_{t=1}^{T-2} \bigg( \prod_{r=0}^{t-1} \mJ_{T-r} \bigg)\mathbf{1}_{(m,k)}\, \xi_{mk}(\vz_{T-t-1})\, \vz_{T-t-1}.
\end{align}
\\
\subsubsection{Piecewise-linear RNN (PLRNN)}\label{sec:plrnn}
The PLRNN has the generic form \citep{Koppe_nonlin_2019,schmidt_identifying_2021}
\begin{align}\label{eq-plrnn}
 \vz_{t} \, = \, F(\vz_{t-1}) \, = \,  \mA \, \vz_{t-1}\, + \, \mW \phi(\vz_{t-1}) \, +\, \mC \vs_t \,+ \, \vh + \bm{\varepsilon}_t, 
\end{align}
~\\
where $\phi(\vz_{t-1})=\max(\vz_{t-1}, 0)$ is the element-wise rectified linear unit (ReLU) function, $\vz_{t} \in \mathbb{R}^{M}$ is the neural state vector, $\mA \in \mathbb{R} ^{M\times M}$ is a diagonal matrix of auto-regression weights, $\mW  \in \mathbb{R} ^{M\times M}$ is a matrix of connection weights, $\vh\in\sR^{M}$ is the bias vector, $\vs_t \in \sR^{K}$ the external input weighted by $\mC\in\sR^{M\times K}$, and $\bm{\varepsilon}_t \,\sim\, \gN(0, \bm{\Sigma})$ a Gaussian noise term with diagonal covariance matrix $\bm{\Sigma}$. 

Equation \eqref{eq-plrnn} can be rewritten as  
\begin{align}\label{eq-d-plrnn}
 \vz_t\, =\, (\mA + \mW \mD_{\Omega(t-1)}) \vz_{t-1} \, +\, \mC \vs_t \,+ \, \vh + \bm{\varepsilon}_t \,=:\, \mW_{\Omega(t-1)} \, \vz_{t-1} \, +\, \mC \vs_t \,+ \, \vh + \bm{\varepsilon}_t,
\end{align}
~\\
where $\mD_{\Omega(t)} \coloneqq \mathrm{diag}(\vd_{\Omega(t)})$ is a diagonal matrix with $\vd_{\Omega(t)} \coloneqq \left(d_1, d_2, \cdots, d_M \right)$ an indicator vector such that $d_m(z_{m,t})=: d_m = 1$ whenever $z_{m,t}>0$, and zeros otherwise. 

For the PLRNN \eqref{eq-d-plrnn} we have 
\begin{align}
\mJ_t \, = \, \frac{\partial \vz_t}{\partial \vz_{t-1}} \, = \, \mW_{\Omega(t-1)} , 
\end{align}
~\\
and $\norm{\mW_{\Omega(t-1)}} \, \leq \, \norm{\mA} \, + \, \norm{\mW}$.\\

Furthermore, the derivatives \eqref{eq:tangent-vec} for the PLRNN \eqref{eq-d-plrnn} are\\
 \begin{align}\label{}
 \frac{\partial \vz_T}{\partial w_{mk}}  \, = \,
 \mathbf{1}_{(m,k)} \mD_{\Omega(T-1)}\, \vz_{T-1} + \sum_{j=2}^{T-1} \, \, \bigg(\prod_{i=1}^{j-1} \mW_{\Omega(T-i)}\bigg) \, \mathbf{1}_{(m,k)} \mD_{\Omega(T-j)} \, \vz_{T-j} .
 %
\end{align} 
~\\
\subsubsection{Long Short-Term Memory (LSTM)}\label{lstm}
The LSTM is defined by the equations
\begin{align}\nonumber
\vi_t&\, =\, \sigma\big(\mW_{ii} \vs_t \,+ \, \mW_{hi} \vh_{t-1}\, + \, \vb_{i}
\big)
\\[1ex]\nonumber
\vf_t&\, =\,\sigma\big(\mW_{if} \vs_t \,+ \, \mW_{hf} \vh_{t-1}\, + \, \vb_{f}
\big)
\\[1ex]\nonumber
\vg_t&\, =\, \tanh{\big(\mW_{ig} \vs_t \,+ \, \mW_{hg} \vh_{t-1}\, + \, \vb_{g}
\big)}
\\[1ex]\nonumber
\vo_t&\, =\,\sigma\big(\mW_{io} \vs_t\,+ \, \mW_{ho} \vh_{t-1} \, + \, \vb_{o}
\big)
\\[1ex]\nonumber
\vc_t&\, =\, \vf_t \,\odot\, \vc_{t-1} + \vi_t \,\odot\, \vg_t
\\[1ex]\label{eq-LSTM}
\vh_t&\, =\, \vo_t \,\odot\, \tanh{(\vc_t)}
\end{align}
where $\{\vs_t\}$ is the input sequence, $\mW$ denotes weight matrices, $\vb$ bias terms, $\vi_t, \vf_t, \vg_t, \vo_t$ demonstrate the input, forget, cell, and output gates, $\vh_t$ and $\vc_t$ are the hidden and cell states at time $t$ respectively, $\sigma$ is the sigmoid 
activation function, and $\odot$ represents the element-wise (Hadamard) product (see \citep{hochreiter_long_1997,graves_2016,vlachas_data-driven_2018} for further information on LSTMs). 

Defining $\vz_t \, :=\, (\vh_t, \vc_t)\tran$, the LSTM \eqref{eq-LSTM} can be represented as the first-order recursive map \\
\begin{align}\label{map-lstm}
\vz_t\, = \, 
F_{\boldsymbol{\theta}} (\vz_{t-1}) \, = \, 
\begin{pmatrix}
 \vo_t \,\odot\, \tanh{(\vf_t \,\odot\, \vc_{t-1} + \vi_t \,\odot\, \vg_t)}
\\[1ex]
\vf_t \,\odot\, \vc_{t-1} + \vi_t \,\odot\, \vg_t
\end{pmatrix}.
\end{align}
~\\
The term $\frac{\partial \mathcal{L}_t}{\partial \theta}$ in \eqref{eq:lo} for some LSTM parameter $\theta$ 
can be written as
\begin{align}\label{lo-lstm-c}
 \frac{\partial \mathcal{L}_t}{\partial \theta} \, = \, \sum_{r=1}^t \frac{\partial \mathcal{L}_t}{\partial \vh_t}\,\frac{\partial \vh_t}{\partial \vz_t}\,
 \frac{\partial \vz_t}{\partial \vz_r} \, \frac{\partial \vz_r}{\partial \theta}.
\end{align}

A necessary condition for LSTMs to have a chaotic orbit is given by: \\
\begin{proposition}\label{thm-lstm}
Let the LSTM given by \eqref{eq-LSTM} have a chaotic attractor $\Gamma^{*}$ with $\mathcal{B}_{\Gamma^{*}}$ as its basin of attraction. Then for every $\vz_1\, = \, (\vh_1, \vc_1)\tran \in \mathcal{B}_{\Gamma^{*}}$
\begin{align}\label{eq-ln2}
\gamma \, := \,\lim_{T\rightarrow\infty}  \sqrt[T]{\norm{  
\begin{pmatrix}
\frac{\partial \vh_T}{\partial \vh_{1}} &  \frac{\partial \vh_T}{\partial \vc_{1}}
\\[2ex]
\frac{\partial \vc_T}{\partial \vh_{1}} &  \frac{\partial \vc_T}{\partial \vc_{1}}
\end{pmatrix}} }\, > \, 1.    
\end{align}
\end{proposition}
\begin{proof}\label{p-thm-lstm}
The Jacobian matrix of \eqref{map-lstm} for $t>1$ can be written in the block form
\begin{align}\label{jac-lstm}
\frac{\partial \vz_t}{\partial \vz_{t-1}}\, = \, J_{t}  \, = \,
\begin{pmatrix}
\frac{\partial \vh_t}{\partial \vh_{t-1}} & \frac{\partial \vh_t}{\partial \vc_{t-1}}
\\[2ex]
 \frac{\partial \vc_t}{\partial \vh_{t-1}} & \frac{\partial \vc_t}{\partial \vc_{t-1}}
\end{pmatrix}.
\end{align}
~\\
Further, due to the chain rule, we have
\begin{align}\nonumber
J_{t}\,J_{t-1} &\, = \, 
\begin{pmatrix}
\frac{\partial \vh_t}{\partial \vh_{t-1}}\frac{\partial \vh_{t-1}}{\partial \vh_{t-2}}\, +\, \frac{\partial \vh_t}{\partial \vc_{t-1}}\frac{\partial \vc_{t-1}}{\partial \vh_{t-2}}& & \frac{\partial \vh_t}{\partial \vh_{t-1}}\frac{\partial \vh_{t-1}}{\partial \vc_{t-2}}\, +\, \frac{\partial \vh_t}{\partial \vc_{t-1}}\frac{\partial \vc_{t-1}}{\partial \vc_{t-2}}
\\[2ex]
\frac{\partial \vc_t}{\partial \vh_{t-1}}\frac{\partial \vh_{t-1}}{\partial \vh_{t-2}}\, +\, \frac{\partial \vc_t}{\partial \vc_{t-1}}\frac{\partial \vc_{t-1}}{\partial \vh_{t-2}} & & \frac{\partial \vc_t}{\partial \vh_{t-1}}\frac{\partial \vh_{t-1}}{\partial \vc_{t-2}}\, +\, \frac{\partial \vc_t}{\partial \vc_{t-1}}\frac{\partial \vc_{t-1}}{\partial \vc_{t-2}}
\end{pmatrix}
\\[1ex]\label{}
&\, = \, 
%
\begin{pmatrix}
\frac{\partial \vh_t}{\partial \vh_{t-2}} & \frac{\partial \vh_t}{\partial \vc_{t-2}}
\\[2ex]
 \frac{\partial \vc_t}{\partial \vh_{t-2}} & \frac{\partial \vc_t}{\partial \vc_{t-2}}
\end{pmatrix}, 
\end{align}
and by induction we obtain
\begin{align}\label{eq-tavan}
\frac{\partial \vz_t}{\partial \vz_{1}} \, = \, J_{t}\,J_{t-1}\,J_{t-2} \cdots J_{2}\, = \, 
%
\begin{pmatrix}
\frac{\partial \vh_t}{\partial \vh_{1}} &  \frac{\partial \vh_t}{\partial \vc_{1}}
\\[2ex]
\frac{\partial \vc_t}{\partial \vh_{1}} &  \frac{\partial \vc_t}{\partial \vc_{1}}
\end{pmatrix}. 
\end{align}
Now assume that \eqref{map-lstm} has a chaotic orbit given by
\begin{align}\label{}
\Gamma^{*} \, = \,  \{\vz_1^{*},\vz_2^{*}, \cdots, \vz^{*}_{T}, \cdots \}.
\end{align}
According to \eqref{eq-tavan}, the largest Lyapunov exponent of $\Gamma^{*}$ is given by 
\begin{align}\nonumber
 \lambda_{\Gamma^{*}} &\, = \, \lim_{T\rightarrow\infty} \frac{1}{T} \, \ln \norm{ J^{*}_{T}\, J^{*}_{T-1}\, \cdots \,  J^{*}_{2}}\, = \, \lim_{T\rightarrow\infty} \frac{1}{T} \, \ln \,
 \norm{  
\begin{pmatrix}
\frac{\partial \vh^{*}_T}{\partial \vh^{*}_{1}} &  \frac{\partial \vh^{*}_T}{\partial \vc^{*}_{1}}
\\[2ex]
\frac{\partial \vc^{*}_T}{\partial \vh^{*}_{1}} &  \frac{\partial \vc^{*}_T}{\partial \vc^{*}_{1}} ,
\end{pmatrix}}.
\end{align}
~\\
Since $\Gamma^{*}$ is chaotic, so $ \lambda_{\Gamma^{*}} >0$, which gives
\begin{align}\label{eq-yeki}
\lim_{T\rightarrow\infty} \sqrt[T]{ \norm{  
\begin{pmatrix}
\frac{\partial \vh^{*}_T}{\partial \vh^{*}_{1}} &  \frac{\partial \vh^{*}_T}{\partial \vc^{*}_{1}}
\\[2ex]
\frac{\partial \vc^{*}_T}{\partial \vh^{*}_{1}} &  \frac{\partial \vc^{*}_T}{\partial \vc^{*}_{1}}
\end{pmatrix}}} \, > \, 1.    
\end{align}
Based on  Oseledec's multiplicative ergodic Theorem, \eqref{eq-yeki} holds for every $\vz_1 \in \mathcal{B}_{\Gamma^{*}}$. 
This completes the proof.
\end{proof}
\subsubsection{Gated Recurrent Unit (GRU)}\label{}
A GRU network is defined by the equations\\
\begin{align}\nonumber
\vz_t&\, =\, \sigma\big(\mW_{z}\, \vs_t \,+ \, \mU_{z} \vh_{t-1}\, + \, \vb_{z} \big)
\\[1ex]\nonumber
\vr_t&\, =\,\sigma\big(\mW_{r} \,\vs_t \,+ \, \mU_{r} \vh_{t-1}\, + \, \vb_{r} \big)
\\[1ex]\label{eq-gru}
\vh_t&\, =\, (1-\vz_t) \,\odot\, \tanh{\big(\mW_{h}\,\vs_t\, + \, \mU_{h}(\vr_t\,\odot\, \vh_{t-1})\, + \, \vb_h \big)\, + \, \vz_t \, \odot \, \vh_{t-1}} ,
\end{align}
~\\
where $\vr_t$ represents the reset gate, $\vz_t$ the update gate, $\vs_t$ and $\vh_t$ denote the inputs and the hidden state respectively, $\mW_{z},\mW_{r},\mW_{h} \in \mathbb{R}^{M \times N}$ and $\mU_{z},\mU_{r},\mU_{h} \in \mathbb{R}^{M \times M}$ are weight matrices, $\vb_{z}, \vb_{r},\vb_{h} \in \mathbb{R}^{M}$ are bias vectors, and $\sigma$ is the element-wise logistic sigmoid function (for more details about GRUs see \citep{cho-etal-2014-properties}).
\subsubsection{Unitary evolution RNN (uRNN)}\label{thm_ortho_rnn}
The uRNN, proposed in \citep{Arjovsky_2016}, is defined as the nonlinear DS
\begin{align}\label{eq-uRNN}
\vz_t \,= \,  \sigma_{\vb} \big( \mW \vz_{t-1} \,+\, \mV \vs_t\big),
\end{align}
for which $\mW \in U(M)$ is an unitary matrix, $\mV \in \mathbb{C}^{M\times N}$, $\vb \in \mathbb{R}^{M}$ is the bias parameter, $\vs_t$ is the real- or complex-valued input of dimension $N$, and \\
\begin{align}\label{}
[\sigma_{\vb}(\vz)]_{i}\, = \, [\sigma_{\text{modReLU}}(\vz)]_i \, = \, 
\begin{cases}
\big(|z_i|+ b_i \big)\frac{z_i}{|z_i|} \hspace{.7cm}  \text{if} \, \, |z_i|+ b_i  \, \geq \, 0
\\[1ex]
\, \, 0     \hspace{2.4cm}  \text{if} \, \, |z_i|+ b_i  \, < \, 0    
\end{cases}.
\end{align}
~\\
\begin{proposition}\label{thm:ortho-rnn}
The uRNN given by \eqref{eq-uRNN} 
cannot have any chaotic orbit.  
\end{proposition}
\begin{proof}
For any arbitrary orbit $\mathcal{O}_{\vz_1}$ of \eqref{eq-uRNN} we have  
\begin{align}
 \norm{ J_{T}\, J_{T-1}\, \cdots \,  J_{2}}\, = \ \norm{\prod_{k=0}^{T-2}  \mD_{T-k} \, \mW\tran},   
\end{align}
where $\mD_t\, = \, diag\Big( \sigma^{\prime}_{\vb}\big( \mW \vz_{t-1} \,+\, \mV \vs_t\big)\Big)$. Since $\mW$ is unitary and so a norm preserving matrix, it is concluded that
\begin{align}
 \norm{\prod_{k=0}^{T-2}  \mD_{T-k} \, \mW\tran} \, \leq \, \prod_{k=0}^{T-2}  \norm{\mD_{T-k}\, \mW\tran}\, = \, \prod_{k=0}^{T-2}  \norm{\mD_{T-k}} \, = \, 1,   
\end{align}
which implies 
\begin{align}\label{eq-n-c}
 \lambda_{max}\, = \, \lim_{T\rightarrow\infty} \frac{1}{T} \, \ln \norm{ J_{T}\, J_{T-1}\, \cdots \,  J_{2}} \, \leq \, 0.    
\end{align}
This rules out the existence of chaos (since $ \lambda_{max}\, > \, 0$ is a necessary condition for $\mathcal{O}_{\vz_1}$ to be chaotic).
\end{proof}

Note that, more generally, any RNN which is constrained such as to exhibit global convergence to a fixed point or cycle, by definition must have a maximum Lyapunov exponent $\lambda_{max} \leq 0$ (in accordance with Theorem \ref{thm-k-2}), hence cannot exhibit chaotic behavior by definition.

\subsection{Theorems: Proofs}
\subsubsection{Proof of theorem \ref{thm-k-2}, parts (ii) \& (iii)}\label{p-thm-k-2}
\begin{proof}
$(ii)\,$ If $\, \mJ$ is the Jordan normal form of $\prod_{s=0}^{k-1} J_{t^{*k}-s}\,$, then $\,\prod_{s=0}^{k-1} J_{t^{*k}-s} \, = \, \mP \, \mJ \, \mP^{-1}$, where 
\\
\begin{align}
\mJ \, = \, 
\begin{pmatrix}
\mJ_{m_1}(\lambda_1) & 0 & 0 & \cdots & 0\\
0 & \mJ_{m_2}(\lambda_2)& 0 & \cdots & 0\\
\vdots & \hdots & \ddots & \hdots & \vdots \\
0 &  \hdots & 0 & \mJ_{m_{p-1}}(\lambda_{p-1})& 0\\
0 &  \hdots & \hdots & 0 & \mJ_{m_{p}}(\lambda_p)
\end{pmatrix},
\end{align}
~\\
and $m_i$ is the algebraic multiplicity of each eigenvalue $\lambda_i$. Since $\rho \big(\prod_{s=0}^{k-1} J_{t^{*k}-s} \big)<1$, so the eigenvalue $\lambda_i$ associated with each Jordan block satisfies $\,|\lambda_i|<1\, (i=1, \cdots, p)$. Moreover, every $m_i \times m_i$ Jordan block has the form
\begin{align}\label{jor-1}
\mJ_{m_{i}}(\lambda_i) &\, = \, 
\begin{pmatrix}
\lambda_i && 1 && 0 && \cdots && 0\\
0 && \lambda_i && 1 && \cdots && 0\\
\vdots && \vdots && \ddots && \ddots && \vdots \\
0 &&  0 && \hdots && \lambda_i &&  1 \\
0 &&  0 && \hdots && 0 && \lambda_i
\end{pmatrix}.
\end{align}
Accordingly
\begin{align}\label{eq-p-n}
\norm{\bigg(\prod_{s=0}^{k-1} J_{t^{*k}-s}\bigg)^j} \, = \, \norm{ \mP \, \mJ^j \, \mP^{-1}} \, \leq \, p\,\norm{ \mJ^j}   ,
\end{align}
~\\
in which $\, p= \norm{\mP} \norm{\mP^{-1}}$. Furthermore, for $j\in \mathbb{N}$, $\mJ^{j}$ is a block diagonal matrix of the form
\begin{align}\label{}
\mJ^{j} \, = \, 
\begin{pmatrix}
\mJ^{j}_{m_1}(\lambda_1) & 0 & 0 & \cdots & 0\\
0 & \mJ^{j}_{m_2}(\lambda_2)& 0 & \cdots & 0\\
\vdots & \hdots & \ddots & \hdots & \vdots \\
0 &  \hdots & 0 & \mJ^{j}_{m_{p-1}}(\lambda_{p-1})& 0\\
0 &  \hdots & \hdots & 0 & \mJ^{j}_{m_{p}}(\lambda_p)
\end{pmatrix},
\end{align}
~\\
in which every $m_i \times m_i$ Jordan block has the form
\begin{align}\label{}
\mJ^{j}_{m_{i}}(\lambda_i) &\, = \, 
\begin{pmatrix}
\lambda_i^{j} && {j \choose 1}\,\lambda_i^{{j}-1} && {j \choose 2}\,\lambda_i^{{j}-2} && \cdots && {j \choose m_i-1}\,\lambda_i^{{j}-m_i+1}\\[1ex]
0 && \lambda_i^{j} && {j \choose 1}\,\lambda_i^{{j}-1} && \cdots && {j \choose m_i-2}\,\lambda_i^{{j}-m_i+2}\\[1ex]
\vdots && \vdots && \ddots && \ddots && \vdots \\[1ex]
0 &&  0 && \hdots && \lambda_i^{j} &&  {j \choose 1}\,\lambda_i^{{j}-1}\\[1ex]
0 &&  0 && \hdots && 0 && \lambda_i^{j}
\end{pmatrix}.
\end{align}
~\\
In addition, for every block $\mJ^{j}_{m_{i}}(\lambda_i)$, we have
\begin{align}\nonumber
\norm{\mJ^{j}_{m_{i}}(\lambda_i)} &\, \leq \, \sqrt{m_i} \,\norm{\mJ^{j}_{m_{i}}(\lambda_i)}_{\infty}\, = \, \sqrt{m_i} \, \sum_{q=1}^{m_i} \Big| \big(\mJ^{j}_{m_{i}}(\lambda_i) \big)_{1q} \Big|
\\[1ex]\nonumber
&\, = \, \sqrt{m_i} \, \sum_{q=1}^{m_i}{{j} \choose q-1} |\lambda_i|^{{j}-q+1} \, = \, |\lambda_i|^{{j}} \sqrt{m_i} \,\bigg( |\lambda_i|^{1-m_i} \sum_{q=1}^{m_i}{{j} \choose q-1} |\lambda_i|^{m_i-q}\bigg)
\\[1ex]\label{}
&\, \leq \, |\lambda_i|^{j}\, j^{\,m_i} \, \sqrt{m_i} \, \bigg(|\lambda_i|^{1-m_i} \sum_{q=1}^{m_i} |\lambda_i|^{m_i-q}\bigg) \, =: \, |\lambda_i|^{j}\, j^{\,m_i} \, N_{\lambda_i}.
\end{align}
~\\
Moreover, for any $1< \tilde{r}_i < \frac{1}{|\lambda_i|}$, there exists some $l_i$ such that $j^{\,m_i} < \tilde{r}_i^{\,j} \,$ for $j \geq l_i$. This means for $j \geq l_i$
\begin{align}\label{ub-J}
&\norm{\mJ^{j}_{m_{i}}(\lambda_i)} \, \leq \,   N_{\lambda_i}\,|\tilde{r}_i\, \lambda_i|^{j},
\end{align}
such that $\, |\tilde{r}_i\, \lambda_i| = \tilde{r}_i|\lambda_i|  <1$. \\

Besides, for $\mJ^{j}=\mJ^{j}_{m_{1}}(\lambda_1)\oplus \mJ^{j}_{m_{2}}(\lambda_2)\oplus \cdots \oplus \mJ^{j}_{m_{p}}(\lambda_p)$
\begin{align}\label{max_J}
\norm{\mJ^{j}} \, = \, \max_{1\leq i \leq p} \norm{\mJ^{j}_{m_{i}}(\lambda_i)} \, =: \, \norm{\mJ^{j}_{m}(\lambda)}.
\end{align}
~\\
Hence, from \eqref{eq-p-n}, \eqref{ub-J} and \eqref{max_J}, it is deduced that for $j \geq l$
\begin{align}\label{eq-jor}
\norm{\Big(\prod_{s=0}^{k-1} J_{t^{*k}-s} \Big)^{j}}\, \leq \, p\,N_{\lambda}\,|\tilde{r}\, \lambda|^{j} \,=:\, \bar{p}\,r^j,
\end{align}
~\\
in which $\, r=|\tilde{r}\, \lambda |  <1$. \\

Furthermore, let for $\Gamma_k$
\begin{align}\nonumber
& \max_{T\geq 1}\Big\{ \norm{ \mJ^{*}_{T}} \Big\}\, = \, \max_{0 \leq s \leq k-1}\Big\{ \norm{J_{t^{*k}-s}} \Big\}\, = \, \bar{m},
\\[1ex]\nonumber
& \max_{T\geq 1}\Big\{ \norm{ \frac{\partial^{+} \vz_T}{\partial \theta}} \Big\}\, = \,\max_{0 \leq s \leq k-1}\Big\{ \norm{ \frac{\partial^{+} \vz_{t^{*k}-s}}{\partial \theta}} \Big\}\, = \, \xi,
\\[1ex]\label{eq-tasavi}
& \max_{T\geq 1}\Big\{ \norm{ \vz_{T}} \Big\}\, = \,\max_{0 \leq s \leq k-1}\Big\{ \norm{ \vz_{t^{*k}-s}} \Big\}\, = \, \bar{q}.
\end{align}
~\\
Hence, defining $\vz_{0}\, =\, 0$, for this $k$-cycle
\begin{align}\nonumber
\norm{\frac{\partial \vz_{T}}{\partial \theta}} & \, = \, 
 \norm{\frac{\partial^{+} \vz_T}{\partial \theta} \, + \, \sum_{t=1}^{T-2} \bigg( \prod_{r=0}^{t-1} \mJ^{*}_{T-r} \bigg)\frac{\partial^{+} \vz_{T-t}}{\partial\theta}}
\\[1ex]\nonumber
& \, = \, 
 \norm{\frac{\partial^{+} \vz_T}{\partial \theta} \, + \, \sum_{t=1}^{T-1} \bigg( \prod_{r=0}^{t-1} \mJ^{*}_{T-r} \bigg)\frac{\partial^{+} \vz_{T-t}}{\partial\theta}}
\\[1ex]\label{}
&\, \leq\, \bar{q}\, \xi \bigg( 
1  \, + \, \sum_{t=1}^{T-1} \norm{ \prod_{r=0}^{t-1} \mJ^{*}_{T-r}} \bigg).
\end{align}
~\\
On the other hand, for $\, T\, = \, kj\,$, from \eqref{eq-jor} and \eqref{eq-tasavi} we have
\\
\begin{align}\nonumber
 \sum_{t=1}^{T-1} \norm{ \prod_{r=0}^{t-1} \mJ^{*}_{T-r}}& \, = \,  \sum_{t=1}^{kj-1} \norm{ \prod_{r=0}^{t-1} \mJ^{*}_{kj-r}} \, = \,  \sum_{t=1}^{k-1} \norm{ \prod_{r=0}^{t-1} \mJ^{*}_{kj-r}}\, +   \sum_{t=k}^{2k-1} \norm{ \prod_{r=0}^{t-1} \mJ^{*}_{kj-r}}
\\[1ex]\nonumber
&\hspace{.3cm} \, +\,  \sum_{t=2k}^{3k-1} \norm{ \prod_{r=0}^{t-1} \mJ^{*}_{kj-r}} +  \cdots + \sum_{t=(j-2)k}^{(j-1)k-1} \norm{ \prod_{r=0}^{t-1} \mJ^{*}_{kj-r}} 
 + \sum_{t=(j-1)k}^{kj-1} \norm{ \prod_{r=0}^{t-1} \mJ^{*}_{kj-r}} 
\\[1ex]\nonumber
& \, = \,  \sum_{t=1}^{k-1} \norm{ \prod_{r=0}^{t-1} \mJ^{*}_{kj-r}}\, + \,
\sum_{i=2}^{j}\,\sum_{t=(i-1)k}^{ik-1}\norm{ \prod_{r=0}^{t-1} \mJ^{*}_{kj-r}} 
\\[1ex]\label{}
& \, \leq \,  \big(\bar{m}+\bar{m}^2 + \cdots + \bar{m}^{k-1}\big) \, + \, \sum_{i=2}^{j} \, \bar{p}\, \big( 1+\bar{m}+\bar{m}^2 + \cdots + \bar{m}^{k-1}\big)\, r^{i-1}.
\end{align}
~\\
Thus, considering $\big(\bar{m}+\bar{m}^2 + \cdots + \bar{m}^{k-1}\big)\, = \, \mathcal{M}$, it is deduced that
\begin{align}
 \lim_{T\rightarrow\infty}  \, \norm{\frac{\partial \vz_{T}}{\partial  \theta} }\, = \,  \lim_{j\rightarrow\infty}  \, \norm{\frac{\partial  \vz_{kj}}{\partial  \theta}} &\, \leq\, \bar{q}\, \xi \big( 
1 \,+ \,\mathcal{M}+ \frac{\bar{p}\, r(1+\mathcal{M})}{1-r} \big) \, = \, \bar{\mathcal{M}} \, <\, \infty,    
\end{align}
~\\
which, by \eqref{eq-s-gr}, implies $\,\frac{\partial \mathcal{L}_T}{\partial \theta}\,$ will be bounded for $T\rightarrow\infty$. 
\\

$(iii)\,$ Consider the PLRNN given by \eqref{eq-plrnn}, where for simplicity we ignore the external inputs and noise terms. Let $\{\vz_{t_1}, \vz_{t_2}, \vz_{t_3} , \ldots \}$ be an orbit which converges to $\Gamma_{k}$. Hence
\begin{align}\label{eq-d}
\displaystyle{\lim_{n \to \infty} d(\vz_{t_n}, \Gamma_k}) \, = \, 0,    
\end{align}
which implies there exists a neighborhood $U$ of $\Gamma_k$ and $k$ sub-sequences
$\{\vz_{t_{km}}\}_{m=1}^{\infty}, \{\vz_{t_{km+1}}\}_{m=1}^{\infty}$, $\cdots$, $\{\vz_{t_{km+(k-1)}}\}_{m=1}^{\infty}$ of the sequence $\{\vz_{t_n} \}_{n=1}^{\infty}$ such that all these sub-sequences belong to $U$ and 
\begin{itemize}
    \item[a)] $\vz_{t_{km+s}}\, = \, F^{k} (\vz_{t_{k(m-1)+s}}), s=0, 1, 2, \cdots, k-1$,\\
    \item[b)]$\displaystyle{\lim_{m \to \infty} \vz_{t_{km+s}} = \vz_{t^{*k}-s} },s=0, 1, 2, \cdots, k-1$, \\
    \item[c)]for every $\vz_{t_n}\in U$ there is some $s \in \{0, 1, 2, \cdots, k-1\}$ such that $\vz_{t_n} \in \{\vz_{t_{km+s}}\}_{m=1}^{\infty}$.
\end{itemize}
~\\
In this case, for every  $\vz_{t_n}\in U$ with $\vz_{t_n} \in \{\vz_{t_{km+s}}\}_{m=1}^{\infty}\,$, there exists some $\tilde{n} \in \mathbb{N}\,$ such that $\, \vz_{t_n}\, = \, \vz_{t_{k\tilde{n}+s}}\,$ and $\,\displaystyle{\lim_{\tilde{n} \to \infty} \vz_{t_{k\tilde{n}+s}} = \vz_{t^{*k}-s}}\,$. Therefore, continuity of $F$ results in
\begin{align}
\displaystyle{\lim_{\tilde{n} \to \infty} F(\vz_{t_{k\tilde{n}+s}}) = F(\vz_{t^{*k}-s})},    
\end{align}
and so by \eqref{eq-d-plrnn}\\
\begin{align}\label{}
 \displaystyle{\lim_{\tilde{n} \to \infty} \big(\mW_{\Omega({t_{k\tilde{n}+s}})}\,\vz_{t_{k\tilde{n}+s}} \, + \, \vh} \big) \, = \,
 \mW_{\Omega(t^{*k}-s)}\,\vz_{t^{*k}-s} \, + \, \vh,
\end{align}
which implies\\
\begin{align}\label{lim-2}
 \displaystyle{\lim_{\tilde{n} \to \infty} \mW_{\Omega({t_{k\tilde{n}+s}})}\,\vz_{t_{k\tilde{n}+s}}} \, = \,
 \mW_{\Omega(t^{*k}-s)}\,\vz_{t^{*k}-s}.
\end{align}
~\\
Assuming $\displaystyle{\lim_{\tilde{n} \to \infty}\mW_{\Omega({t_{k\tilde{n}+s}})} }=\mL$, since \eqref{lim-2} holds for every $\vz_{t^{*k}-s}$, substituting $\vz_{t^{*k}-s}=\ve_1\tran=(1, 0, \cdots , 0)^T$ in \eqref{lim-2}, we can prove that the first column of $\mL$ equals the first column of $\mW_{\Omega(t^{*k}-s)}$. Performing the same procedure for $\vz_{t^{*k}-s}=\ve_i\tran$, $i=2, 3, \cdots, M$, yields
\begin{align}\label{eq-limit}
 \displaystyle{\lim_{\tilde{n} \to \infty} \mW_{\Omega({t_{k\tilde{n}+s}})}} \, = \,
 \mW_{\Omega(t^{*k}-s)}.
 \end{align}
According to \eqref{eq-d}, $U$ contains an infinite number of terms of the sequence $\{\vz_{t_n} \}_{n=1}^{\infty}$, i.e. 
\begin{align}\label{}
\exists N \in \mathbb{N}\, \, \, \, s.t. \, \, \, \,  \, \, \,  n \geq N \, \,  \Longrightarrow \, \,  \vz_{t_n}\in U.
\end{align}
Suppose that $\vz_{t_n}\in U$ for some $n \geq N$. Thus, there exists some $s\in \{0, 1, 2, \cdots, k-1 \}$ such that $\vz_{t_n} \in \{\vz_{t_{km+s}}\}_{m=1}^{\infty}$. Without loss of generality let $s=0$. Hence, there is some $\tilde{n} \in \mathbb{N}\,$ such that $\, \vz_{t_n}\, = \, \vz_{t_{k\tilde{n}}}\,$ and $\,\displaystyle{\lim_{\tilde{n} \to \infty} \vz_{t_{k\tilde{n}}} = \vz_{t^{*k}}}$. In this case, moving forward in time gives
\begin{align}\nonumber
& \vz_{t_{n}}\, = \, \vz_{t_{k\tilde{n}}} \hspace{.2cm}\big(\vz_{t_{n}} \in \{\vz_{t_{km}}\}_{m=1}^{\infty}\big),\hspace{4.3cm} \displaystyle{\lim_{\tilde{n} \to \infty} \vz_{t_{k\tilde{n}}} = \vz_{t^{*k}}},
\\[1ex]\nonumber
& \vz_{t_{n+1}}\, = \, \vz_{t_{k\tilde{n}+1}} \hspace{.2cm}\big(\vz_{t_{n+1}} \in \{\vz_{t_{km+1}}\}_{m=1}^{\infty}\big),\hspace{3.1cm} \displaystyle{\lim_{\tilde{n} \to \infty} \vz_{t_{k\tilde{n}+1}} = \vz_{t^{*k}-1}},
\\[1ex]\nonumber
& \vz_{t_{n+2}}\, = \, \vz_{t_{k\tilde{n}+2}} \hspace{.2cm}\big(\vz_{t_{n+2}} \in \{\vz_{t_{km+2}}\}_{m=1}^{\infty}\big),\hspace{3.1cm} \displaystyle{\lim_{\tilde{n} \to \infty} \vz_{t_{k\tilde{n}+2}} = \vz_{t^{*k}-2}},
\\\nonumber
& \vdots
\\\nonumber
& \vz_{t_{n+k-1}}\, = \, \vz_{t_{k\tilde{n}+k-1}} \hspace{.2cm}\big(\vz_{t_{n+(k-1)}} \in \{\vz_{t_{km+k-1}}\}_{m=1}^{\infty}\big),\hspace{1.6cm} \displaystyle{\lim_{\tilde{n} \to \infty} \vz_{t_{k\tilde{n}+k-1}} = \vz_{t^{*k}-(k-1)}},
\\[2ex]\nonumber
& \vz_{t_{n+k}}\, = \, \vz_{t_{k(\tilde{n}+1)}} \hspace{.2cm}\big(\vz_{t_{n+k}} \in \{\vz_{t_{km}}\}_{m=1}^{\infty}\big),\hspace{3.3cm} \displaystyle{\lim_{\tilde{n} \to \infty} \vz_{t_{k(\tilde{n}+1)}} = \vz_{t^{*k}}},
\\[1ex]\nonumber
& \vz_{t_{n+k+1}}\, = \, \vz_{t_{k(\tilde{n}+1)+1}} \hspace{.2cm}\big(\vz_{t_{n+k+1}} \in \{\vz_{t_{km+1}}\}_{m=1}^{\infty}\big),\hspace{2cm} \displaystyle{\lim_{\tilde{n} \to \infty} \vz_{t_{k(\tilde{n}+1)+1}} = \vz_{t^{*k}-1}},
\\\nonumber
& \vdots
\\\nonumber
& \vz_{t_{n+2k-1}}\, = \, \vz_{t_{k(\tilde{n}+1)+k-1}} \hspace{.2cm}\big(\vz_{t_{n+2k-1}} \in \{\vz_{t_{km+k-1}}\}_{m=1}^{\infty}\big),\hspace{1.1cm} \displaystyle{\lim_{\tilde{n} \to \infty} \vz_{t_{k(\tilde{n}+1)+k-1}} = \vz_{t^{*k}-(k-1)}},
\\[2ex]\nonumber
& \vz_{t_{n+2k}}\, = \, \vz_{t_{k(\tilde{n}+2)}} \hspace{.2cm}\big(\vz_{t_{n+2k}} \in \{\vz_{t_{km}}\}_{m=1}^{\infty}\big),\hspace{3.1cm} \displaystyle{\lim_{\tilde{n} \to \infty} \vz_{t_{k(\tilde{n}+2)}} = \vz_{t^{*k}}},
\\\label{eq-seq}
& \vdots
\end{align}
~\\
Consequently, for $n \geq N$ and $j\in\mathbb{N}$, we can write
\begin{align}\nonumber
& \prod_{i=0}^{kj-1} \mW_{\Omega(t_{n + kj-1-i})} 
 \\[1ex]\nonumber
 &\, = \,  \Big(\prod_{i=1}^{k} \mW_{\Omega(t_{k(\tilde{n}+j)+k-i})} \Big)\Big( \prod_{i=1}^{k} \mW_{\Omega(t_{k(\tilde{n}+j-1)+k-i})}\Big) \cdots \Big( \prod_{i=1}^{k} \mW_{\Omega(t_{k(\tilde{n})+k-i})}\Big)
 \\[1ex]\label{eq-product}
 &\, = \, 
 \prod_{l=0}^{j} \, \, \prod_{i=1}^{k} \mW_{\Omega(t_{k(\tilde{n}+j-l)+k-i})}.
\end{align}
On the other hand, in equation \eqref{eq-d-plrnn}, there are different configurations for matrix $\mD_{\Omega(t-1)}$ and hence different forms for matrix $\,\mW_{\Omega({t_{k\tilde{n}+s}})}\,
$. In this case, the phase space of the system is divided into different sub-regions by some borders; see \citep{monfared_existence_2020,monfared_transformation_2020} for more details. Also, since the system \eqref{eq-d-plrnn} is a linear map in each sub-region, the $k$ periodic points of $\Gamma_k$ must belong to different sub-regions (at least two different sub-regions). Accordingly, based on \eqref{eq-limit} and \eqref{eq-seq}, there exists some $\,\tilde{N}\in \mathbb{N}\,$ such that for every $\,\tilde{n} \geq \tilde{N}\,$ both $\vz_{t_{k\tilde{n}+s}}$ and $\vz_{t^{*k}-s}$ belong to the same sub-region and so the matrices $\,\mW_{\Omega({t_{k\tilde{n}+s}})}\,
 $ and $\,\mW_{\Omega(t^{*k}-s)}\,$ ($s \in \{0, 1, 2, \cdots, k-1\}$) are identical. Hence, for $n\geq N$, $\tilde{n} \geq\tilde{N}$ and $j \in \mathbb{N}$, equation \eqref{eq-product} becomes
\begin{align}\label{eq-zarb-j}
 \prod_{i=0}^{kj-1} \mW_{\Omega(t_{n + kj-1-i})}
 \, = \,
\prod_{l=0}^{j} \, \, \prod_{i=1}^{k} \mW_{\Omega(t_{k(\tilde{n}+j-l)+k-i})} \, = \, \bigg(\prod_{s=0}^{k-1}\mW_{\Omega(t^{*k}-s)}\bigg)^j.
 \end{align}
~\\
Therefore, similar to the part $(ii)$, we can prove for every $\vz_1 \in \mathcal{B}_{\Gamma_{k}}$, $\, \frac{\partial \vz_{T}}{\partial \theta} \,$ and $\,\frac{\partial \mathcal{L}_T}{\partial \theta}\,$ will also remain bounded. 
\end{proof}
\subsubsection{Proof of theorem \ref{thm-3}, part (ii)}\label{p-thm-3}
\begin{proof}
$(ii)\, $ Let for every $ T >2$
\begin{align}
 \mL_{T}\, := \,  J^{*}_{T}\, J^{*}_{T-1}\, \cdots \,  J^{*}_{2}.
\end{align}
$\{\mL_{T} \}_{T \in \mathbb{N},\, T>2}$ is a sequence of matrices $\mL_{T} \, = \, \big[l^{(T)}_{i j}\big]_{1 \leq i,j \leq M}$ and, due to \eqref{yaein}, $\, \lim_{T\rightarrow\infty} \norm{\mL_{T}} \, =\, \infty $. Hence, there is at least one sub-sequence $\{l^{(T_n)}_{m k}\}_{T_n \in \mathbb{N}, \, T_n>2}$ (for some $m, k \in \{1,2, \cdots, M\}$) such that $\lim_{T_n\rightarrow\infty}l^{(T_n)}_{m k} \, =\, \infty $. \\

On the other hand
\begin{align}\label{}
 \frac{\partial \vz^{*}_T}{\partial \theta} \, = \, 
 \frac{\partial^{+} \vz^{*}_T}{\partial \theta} \, + \, \sum_{t=1}^{T-2} \bigg( \prod_{r=0}^{t-1} \mJ^{*}_{T-r} \bigg)\frac{\partial^{+} \vz^{*}_{T-t}}{\partial\theta}.
\end{align}
~\\
Moreover, there exists some $N>2$ such that (for $t=T-N+1$)
\begin{align}\label{shart}
\frac{\partial^{+} \vz^{*}_{N-1}}{\partial\theta} 
\, \neq \, 0.  
\end{align}
For $\theta$ as the $k$-th element of a parameter vector $\boldsymbol{\theta}$ (or belonging to the $k$-th row of a parameter matrix $\boldsymbol{\theta}$), 
the term
\begin{align}\label{der-ch}
\bigg(\prod_{r=0}^{T-N} \mJ^{*}_{T-r}\bigg) \frac{\partial^{+} \vz^{*}_{N-1}}{\partial \theta}
%
 %
\end{align}
~\\
is a vector in which the $i$-th element is $\,l^{(T)}_{i k}\, \frac{\partial^{+} \vz^{*}_{k,N-1}}{\partial\theta}
$.\\

Since $\lim_{T_n \rightarrow\infty}l^{(T_n)}_{m k} \, =\, \infty $, due to \eqref{shart} $\lim_{T_n\rightarrow\infty}l^{(T_n)}_{m k}\,
\frac{\partial^{+} \vz^{*}_{k,N-1}}{\partial\theta}\, =\, \infty $, which implies $\frac{\partial \vz^{*}_T}{\partial \theta}$ will diverge as $T\rightarrow\infty$.
Similarly, by \eqref{eq-s-gr}, we can prove $\frac{\partial \mathcal{L}_T^{*}}{\partial \theta}$ is divergent for $T\rightarrow\infty$.

By Oseledec's multiplicative ergodic Theorem, 
the results also hold for every $\vz_1 \in \mathcal{B}_{\Gamma^{*}}$.
\end{proof}
\subsubsection{Proof of theorem \ref{thm-quasi}}\label{p-thm-quasi}
\begin{proof}
Let $\Gamma \,= \,\{\vz_{1}, \vz_{2}, \ldots \vz_T, \cdots \}$ be a quasi-periodic attractor. Then, the largest Lyapunov exponent of $\Gamma$ is 
\begin{align}\label{LE-quasi}
 \lambda\, = \, \lim_{T\rightarrow\infty} \frac{1}{T} \, \ln \norm{ J_{T}\, J_{T-1}\, \cdots \,  J_{2}}\, = \, \lim_{T\rightarrow\infty} \frac{1}{T} \, \ln \norm{\frac{\partial \vz_T}{\partial \vz_1}} \, = \, 0.
 \end{align}
~\\
We prove for every $0< \epsilon <1$ 
\begin{align}
\lim_{T\rightarrow\infty}(1-\epsilon)^{T-1}  \, < \, \lim_{T\rightarrow\infty} \norm{\frac{\partial \vz_T}{\partial \vz_1}} \, < \, \lim_{T\rightarrow\infty} (1+\epsilon)^{T-1}.    
\end{align}
~\\
For this purpose, we show $\forall \,0< \epsilon <1$
\begin{itemize}
    \item[(I)] $\,\lim_{T\rightarrow\infty}(1-\epsilon)^{T-1}  \, < \, \lim_{T\rightarrow\infty} \norm{\frac{\partial \vz_T}{\partial \vz_1}}$, and \\
    \item[(II)] $\,\lim_{T\rightarrow\infty} \norm{\frac{\partial \vz_T}{\partial \vz_1}} \, < \, \lim_{T\rightarrow\infty} (1+\epsilon)^{T-1}$.
\end{itemize}
Assume for the sake of contradiction that (I) does not hold. Then there exists some $0< \epsilon <1$ such that
\begin{align}
\lim_{T\rightarrow\infty}(1-\epsilon)^{T-1}  \, \geq \, \lim_{T\rightarrow\infty} \norm{\frac{\partial \vz_T}{\partial \vz_1}}.    
\end{align}
Therefore
\begin{align}
\exists \,T_0>1 \, \, \, s.t.\, \, \, \forall \,T \geq T_0 \, \, \Longrightarrow \, \, (1-\epsilon)^{T-1}  \, \geq \,  \norm{\frac{\partial \vz_T}{\partial \vz_1}},
\end{align}
and so 
\begin{align}
\exists \,T_0>1 \, \, \, s.t.\, \, \, \forall \,T \geq T_0 \, \, \Longrightarrow \, \, \frac{\ln(1-\epsilon)^{T-1}}{T-1}  \, \geq \, \frac{\ln \norm{\frac{\partial \vz_T}{\partial \vz_1}}}{T-1}.
\end{align}
Consequently, due to \eqref{LE-quasi}, for $T\rightarrow\infty$ we have $\, \ln(1-\epsilon) \,\geq \, 0$. This implies $\epsilon \leq 0$, which is a contradiction. 

Similarly if we assume (II) is not true, then there exists some $0< \epsilon <1$ such that
\begin{align}
\lim_{T\rightarrow\infty} \norm{\frac{\partial \vz_T}{\partial \vz_1}} \, \geq \, \lim_{T\rightarrow\infty} (1+\epsilon)^{T-1}.    
\end{align}
Thereby
\begin{align}
\exists \,T_0>1 \, \, \, s.t.\, \, \, \forall \,T \geq T_0 \, \, \Longrightarrow \, \,  \norm{\frac{\partial \vz_T}{\partial \vz_1}} \, \geq \,  (1+\epsilon)^{T-1},
\end{align}
and thus
\begin{align}
\exists \,T_0>1 \, \, \, s.t.\, \, \, \forall \,T \geq T_0 \, \, \Longrightarrow \, \, \frac{\ln \norm{\frac{\partial \vz_T}{\partial \vz_1}}}{T-1} \, \geq \, \frac{\ln(1+\epsilon)^{T-1}}{T-1}.
\end{align}
~\\
This means $\, \ln(1+\epsilon) \,\leq \, 0$ as $T\rightarrow\infty$, i.e. $\epsilon \leq 0$, which is a contradiction. 

Therefore \eqref{re-quasi} holds for $\Gamma$ and also, according to Oseledec's multiplicative ergodic Theorem, for every $\vz_1$ in the basin of attraction of $\Gamma$. 
\end{proof}
\subsection{Additional results on relation between dynamics and gradients}\label{}
%
\subsubsection{Further results and remarks related to Theorem \ref{thm-3}}\label{rem-cor-thm-chaos}
\begin{remark}\label{rm-unst-orb}
The result of Theorem \ref{thm-3} also holds for unstable orbits $\, \{\vz_1, \vz_2, \vz_3, \cdots \} \,$ with positive largest Lyapunov exponent. Trivially, for such orbits that diverge to infinity (unbounded latent states) gradients of the loss function will explode as $T\rightarrow\infty$. 
\end{remark}
\begin{remark}
For RNNs with ReLU activation functions 
there are finite compartments 
in the phase space each with a different functional form. In such a case, to define the largest Lyapunov exponent of $\,\Gamma^{*}$, in the proof of Theorem \ref{thm-3} we assume that $\,\Gamma^{*}$ never maps to the points of the borders.
\end{remark}
Based on Theorem \ref{thm-3}, we can also formulate the necessary conditions for chaos and diverging gradients in standard RNNs with particular activation functions by considering the norms of their recurrence matrix, for which the following Corollary provides the basis:
\begin{corollary}\label{cor-1}
Let for a standard RNN 
\begin{align}\label{bound-der}
\, \norm{diag \big( f^{\prime}\big( \mW \vz_{t-1} \,+\, \mB \vs_t \, + \,\vh  \big) \big)} \, \leq \, \gamma \, < \, \infty.  
\end{align}
If the RNN is chaotic, then $\, \norm{\mW}\,\gamma\, >\, 1\,$.
\end{corollary}
\begin{proof}
Assume for the sake of contradiction that $\, \norm{\mW}\,\gamma\, \leq \, 1\,$. From 
\begin{align}\nonumber
\norm{\prod_{2 < t \leq T} \mW diag \big( f^{\prime}\big( \mW \vz_{t-1} \,+\, \mB \vs_t \, + \,\vh  \big) \big)} &\, \leq \, \prod_{2 < t \leq T} \norm{\mW diag \big( f^{\prime}\big( \mW \vz_{t-1} \,+\, \mB \vs_t \, + \,\vh \big) \big)}
\\[1ex]\label{}
&\, \leq \, ( \norm{\mW}\,\gamma)^{T-2},
\end{align}
 ~\\
it is concluded that 
$ \,\lim_{T\rightarrow\infty}  \norm{\prod_{2 < t \leq T} \mW diag \big( f^{\prime}\big( \mW \vz_{t-1}\,+\, \mB \vs_t \, + \,\vh  \big) \big)}\, < \, \infty \,$, which contradicts \eqref{yaein}. This means $\, \norm{\mW}\,\gamma\, >\, 1\,$ is a necessary condition for the standard RNN to be chaotic.
\end{proof}
\begin{remark}
For RNN with the \textit{tanh} and \text{sigmoid} activation functions $\gamma=1$ and $\gamma=\frac{1}{4}$, respectively. Thus, by Corollary \ref{cor-1}, the necessary conditions for chaos in these two cases are $\, \norm{\mW}\, >\, 1\,$ and $\, \norm{\mW}\, >\, 4\,$, respectively. 
\end{remark}
\subsubsection{Other connections between dynamics and gradients}\label{sec-other}
There is a direct link between the norms of the Jacobians of the RNN along trajectories and the EVGP. By observing this link, we can formulate some general conditions 
that will have implications for the behavior of the gradients regardless of the limiting behavior of the RNN, as collected in the following theorem:
\begin{theorem}\label{thm-other}
Let $\mathcal{O}_{\vz_{1}}\,= \,\{\vz_{1}, \vz_{2}, \ldots \vz_T, \cdots \}$ be a sequence (orbit) generated by 
an 
RNN $F_{\boldsymbol\theta} \in \mathcal{R}$ parameterized by $\boldsymbol\theta$, and $\mP_T\, :=\, \mJ_T-\mI, \, \, T=2, 3, \cdots $. 
\vspace{-.3cm}
\begin{itemize}
\item[(i)] Assume that $\mathcal{O}_{\vz_{1}}$ is an orbit for which $\,\norm{\frac{\partial^{+} \vz_T}{\partial \theta}} \, \leq \, \xi\, \, \,\forall t$. 
If $\, \sum_{T=2}^{\infty} \norm{\mJ_T} \, < \, \infty$, then the Jacobian $\frac{\partial \vz_T}{\partial \vz_1}$, the tangent vector $\frac{\partial \vz_T}{\partial \theta}$ and thus the gradient of the loss function, $\frac{\partial \mathcal{L}_T}{\partial \theta}$, will be bounded for $T\rightarrow\infty$.
\vspace{-.2cm}
\item[(ii)] If $\, \sum_{T=2}^{\infty} \norm{\mP_T} \, < \, \infty$, then the Jacobian $\frac{\partial \vz_T}{\partial \vz_1}$ will neither vanish nor explode as $T \rightarrow \infty$.
\item[(iii)] Let $\norm{ \mJ_T} \, \neq \, 0, \, \, T \geq 2$, and $\, \sum_{T=2}^{\infty} \ln{\norm{\mJ_T}}$ diverge to $\,-\infty$, then the Jacobian $\frac{\partial \vz_T}{\partial \vz_1}$ vanishes as $T$ tends to infinity.  
\end{itemize}
\end{theorem}
Part $(i)$ of Theorem \ref{thm-other} relaxes some of the conditions required in 
Theorem \ref{thm-k-2} for bounded gradients by imposing a Lipschitz condition on the immediate derivatives. Part $(ii)$ generalizes conditions satisfied, for instance, in orthogonal (unitary) RNNs \citep{Arjovsky_2016,henaff_recurrent_2016} or fully regularized PLRNNs \citep{schmidt_identifying_2021}.
\begin{proof}
Let $\norm{.}$ be any matrix norm satisfying $\norm{\mA_1\mA_2}  \, \leq \, \norm{\mA_1} \norm{\mA_2} $.

$(i)\,$ By boundedness of 
$\,\frac{\partial^{+} \vz_T}{\partial \theta}$ we have
\begin{align}\nonumber
\norm{\frac{\partial \vz_{T}}{\partial \theta}} & 
\, = \, \norm{\frac{\partial^{+} \vz_T}{\partial \theta}\, + \, \sum_{t=1}^{T-2} \bigg( \prod_{r=0}^{t-1} \mJ_{T-r} \bigg)\frac{\partial^{+} \vz_{T-t}}{\partial\theta}}
\\[1ex]\label{eq-i}
&\, \leq\, 
\xi \bigg( 
1  \, + \, \sum_{t=1}^{T-2} \norm{ \prod_{r=0}^{t-1} \mJ_{T-r}} \bigg)\, \leq\, 
\xi \bigg( 
1  \, + \, \sum_{t=1}^{T-2}\,  \prod_{r=0}^{t-1} \norm{\mJ_{T-r}} \bigg).
\end{align}
Moreover, 
\begin{align}\nonumber
1  \, + \,\sum_{t=1}^{T-2} \, \prod_{r=0}^{t-1} \norm{\mJ_{T-r}} &\, \leq \,  1  \, + \, \sum_{p} \norm{\mJ_{p}} \, + \, \sum_{p<q} \norm{\mJ_{p}}\norm{\mJ_{q}} \, + \, \sum_{p<q<r} \norm{\mJ_{p}}\norm{\mJ_{q}} \norm{\mJ_{r}} \,+\, \cdots 
\\[1ex]\label{eq-i-2}
&\, =\,\big(1+\norm{\mJ_{T}}\big)\big(1+\norm{\mJ_{T-1}}\big)\cdots \big(1+\norm{\mJ_{2}}\big)\, =:\,\prod_{t=2}^{T} \big(1+\norm{\mJ_{t}}\big).
\end{align}
Since $\, \sum_{T=2}^{\infty} \norm{\mJ_T}$ converges, according to \citep{Wedderburn_1964}, the infinite products $\prod_{T=2}^{\infty} \big(1+\norm{\mJ_{T}}\big)$ in \eqref{eq-i-2} converge to a finite number $\tilde{\mathcal{K}} \neq 0$. Consequently, by \eqref{eq-i} and \eqref{eq-i-2} 
\begin{align}
\lim_{T\rightarrow\infty}\norm{\frac{\partial \vz_{T}}{\partial \theta}} \, \leq \, \tilde{\mathcal{K}} \,< \, \infty,
\end{align}
which implies $\frac{\partial \mathcal{L}_T}{\partial \theta}$ will be bounded for $T\rightarrow\infty$.\\

Furthermore 
\begin{align}
\lim_{T\rightarrow\infty}\norm{\frac{\partial \vz_T}{\partial \vz_1}} \, \leq \, \prod_{T=2}^{\infty} \norm{\mJ_{T}}  \,: = \, \lim_{T\rightarrow\infty} \bigg(\norm{\mJ_{T}}\, \norm{\mJ_{T-1}}\, \cdots \,  \norm{\mJ_{2}}\bigg) \, \leq \, \prod_{T=2}^{\infty} \big(1+\norm{\mJ_{T}}\big)\, \leq \, \tilde{\mathcal{K}},
\end{align}
which completes the proof.
\\

$(ii)$ Since $\, \sum_{T=1}^{\infty} \norm{\mP_T}  \, < \, \infty$, due to \citep{Wedderburn_1964} the infinite product 
\begin{align}
 \prod_{T=2}^{\infty} \big(\mI+\mP_{T}\big)\,=\, \prod_{T=2}^{\infty} \mJ_T \,:= \, \lim_{T\rightarrow\infty} \mJ_{T}\, \mJ_{T-1}\, \cdots \,  \mJ_{2},
 \end{align}
converges to a matrix $\mK \neq \mO$, which implies 
\begin{align}
0\, < \, \lim_{T\rightarrow\infty} \norm{\frac{\partial \vz_T}{\partial \vz_1}} = \norm{\mK} \, < \, \infty.    
\end{align}
$(iii)$ For $\norm{ \mJ_T} \, \neq \, 0, \, \, T \geq 2$, we have
\begin{align}\nonumber
& 0 \,  \leq \, \norm{\frac{\partial \vz_T}{\partial \vz_1}}  \, \leq \, \norm{\mJ_{T}}\, \norm{\mJ_{T-1}}\, \cdots \,  \norm{\mJ_{2}} 
\\[1ex]
& =  e^{\ln{\norm{\mJ_{T}}}} e^{\ln{\norm{\mJ_{T-1}}}} \cdots e^{\ln{\norm{\mJ_{2}}}} =  e^{\sum_{t=2}^{T} \ln{\norm{\mJ_t}}}. 
\end{align}
Hence if $\,\sum_{T=2}^{\infty} \ln{\norm{\mJ_T}}\rightarrow -\infty$, then
\begin{align}
\lim_{T\rightarrow\infty}\frac{\partial \vz_T}{\partial \vz_1} \, = \,  \mO.  
\end{align}
\end{proof}
\subsection{Empirical evaluation: Datasets}

\paragraph{Lorenz attractor}\label{supp:lorenz}
The Lorenz system \citep{lorenz_deterministic_1963} is a simplified model for atmospheric convection, given by 
\begin{align}\label{eq-Lorenz} \nonumber
\frac{\mathrm{d} x}{\mathrm{~d} t}&=\sigma(y-x) ,\\ 
\frac{\mathrm{d} y}{\mathrm{~d} t}&=x(\rho-z)-y, \\ \nonumber
\frac{\mathrm{d} z}{\mathrm{~d} t}&=x y-\beta z.
\end{align}
The system is of particular interest for its chaotic regime and was studied here for $\sigma = 16$, $\rho = 45.92$ and $\beta = 4$. For these parameters the Lorenz system is known to have a maximal Lyapunov exponent $\lambda_{\max}=1.5$ \citep{rosenstein_practical_1993}. To generate a time series, the ODEs were integrated with a step size $\Delta t=0.01$ using \texttt{scipy.integrate}. 
Accordingly, the prediction time is $\tau_{pred}=\frac{\ln(2)}{\Delta t\  \lambda_{max}}= 46.2$.

\paragraph{Duffing oscillator} The Duffing oscillator \citep{Duffing1918} is an example of a periodically forced oscillator with nonlinear elasticity
\begin{align}\label{supp:Duffing}
    \ddot x+\delta \dot x+\beta x+\alpha x^3=\gamma \cos(\omega t).
\end{align}
Note that this system is non-autonomous, that is externally forced due to the r.h.s. of eqn. \ref{supp:Duffing}. The following parameters were chosen to arrive at a chaotically forced oscillator: $\alpha = 1.0$, $\beta=-1.0$, $\delta = 0.1$, $\gamma=0.35$, and $\omega= 1.4$. For these parameters the Duffing oscillator has a maximum Lyapunov exponent of $\lambda_{max} = 0.0995$. The dataset used here was created with the code from \citep{gilpin2021chaos} as a three dimensional embedding with step size $\Delta t = 0.17$. The prediction time is $\tau_{pred}=39.28$.

\paragraph{Rössler system}\label{supp:roessler}
Another prime textbook example for a chaotic system is the Rössler system \citep{rossler_equation_1976} given by:
\begin{align}\nonumber
\frac{d x}{d t}&=-y-z ,\\ 
\frac{d y}{d t}&=x+a y, \\\nonumber
\frac{d z}{d t}&=b+z(x-c).
\end{align}

For the parameters $a=0.15$, $b=0.2$ and $c=10$, the maximal Lyapunov exponent is $\lambda_{\max}=0.09$ \citep{rosenstein_practical_1993}. To arrive at a time series, a step size of $\Delta t= 0.1$ was chosen for integration. This gives us a prediction time of $\tau_{pred}= 77.0$ for this system. 

\paragraph{Mackey-Glass equation}
The Mackey-Glass equation \citep{glass_pathological_1979} is a nonlinear time delay differential equation 
\begin{align}
\dot x = \beta \frac{x_{\rho}}{1+x_{\rho}^n}-\gamma x\, \, \, \, \text{ with }\beta, \gamma, \rho >0.
\end{align}
Here $x_{\rho}$ represents the value of the variable $x$ at time $t - \rho$ (note that strictly, mathematically, this makes the system infinite-dimensional). Choosing the parameters to be  $\beta=2$, $\gamma=1.0$, $n= 9.65$, and $\rho= 2.0$, leads to chaotic behavior with a maximum Lyapunov exponent of $\lambda_{max} = 0.21$. The dataset was created as a 10-dimensional embedding with the code from \citep{gilpin2021chaos} using $\Delta t = 0.04$. This yields a prediction time of $\tau_{pred} = 82.2$.

\paragraph{Empirical temperature time series}\label{supp:emp-time-series}
This time series was recorded at the \href{https://www.bgc-jena.mpg.de/wetter/}{Weather Station at the Max Planck Institute for Biogeochemistry in Jena, Germany}, spanning the time period between 2009 and 2016, and reassembled by François Chollet for the book \textit{Deep Learning with Python}. The data set can be accessed at \href{https://www.kaggle.com/pankrzysiu/weather-archive-jena}{https://www.kaggle.com/pankrzysiu/weather-archive-jena}. \\
To expose the underlying chaotic dynamics of the time series, trends and yearly cycles were removed, and nonlinear noise-reduction was performed (using \texttt{ghkss} from \textit{TISEAN}, see also  \citep{kantz_nonlinear_1993}). Fig. \ref{fig:emp-data} (a) shows a snippet of the temperature data in comparison with the de-noised time-series. High-frequency noise was further reduced through Gaussian kernel smoothing ($\sigma = 200$), and the resulting time series was sub-sampled (every $5^{th}$ data point was retained). Fig. \ref{fig:emp-data} (b) clearly reveals a fractional dimension of $D_{eff}=2.8$ for the de-noised and smoothed time-series. This strongly suggests that the  dynamics governing the time series are chaotic.
We created a time delay embedding \citep{kantz_schreiber_2003} with $m=5$ (estimated by the false nearest neighbor technique, see \citep{Kennel_fnn_1992}) and delay $\Delta t = 500$ (obtained as the first minimum of the mutual information). 
The first three embedding dimensions are shown in Fig. \ref{fig:emp-data}(c). The maximal Lyapunov exponent of this time series was determined with \texttt{lyap\_r} from \textit{TISEAN} \citep{hegger_practical_1999} to be $\lambda_{\max}=0.016$, see Fig. \ref{fig:tau-sweep-emp-data}(a). This value is in close agreement with the literature \citep{millan_nonlinear_2010}. The predictability time of this system is estimated to be $\tau_{pred}=43.3$.
\begin{figure}[hbt!]
\centering    
\subfigure{\label{fig:a:1}\includegraphics[width=44mm]{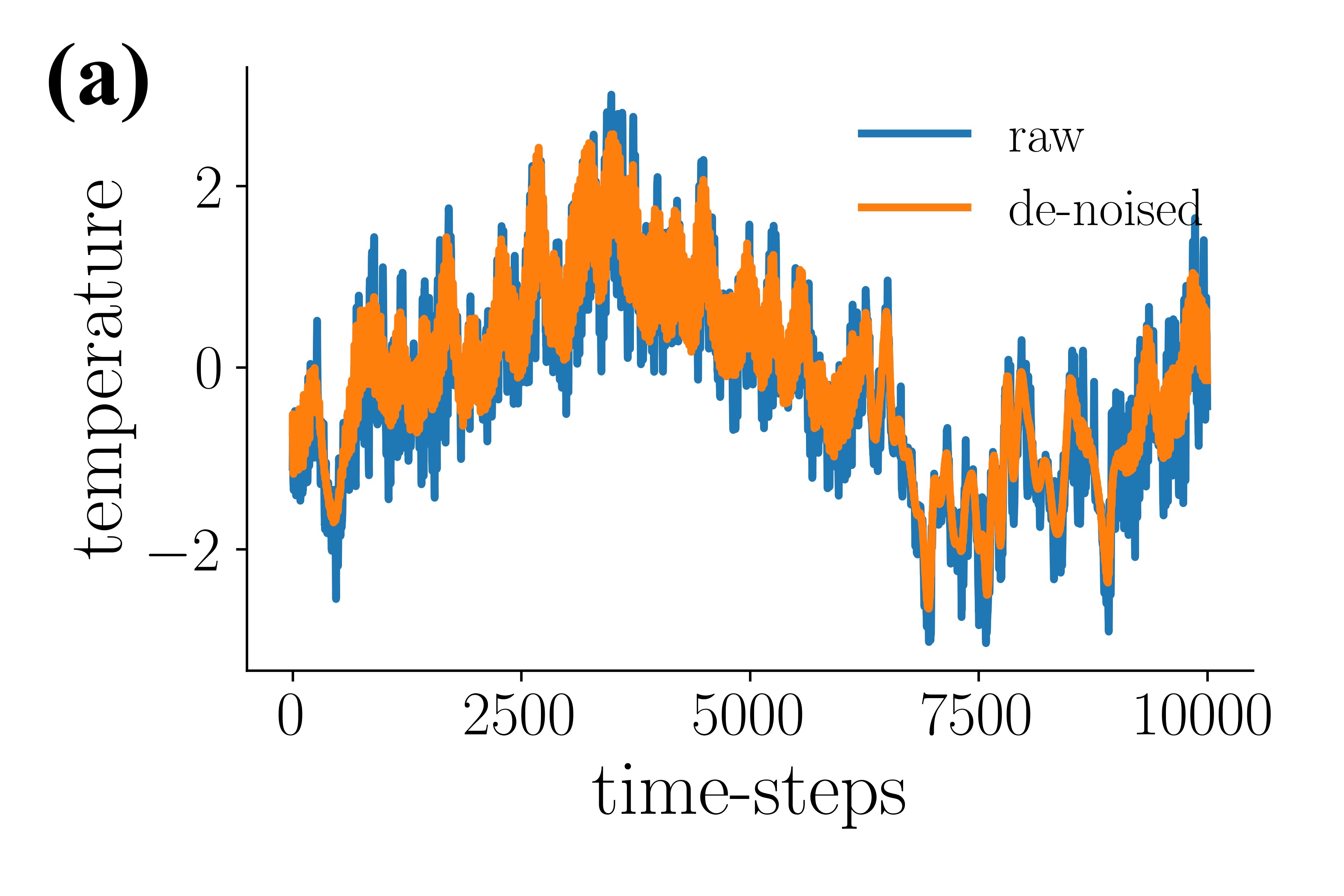}}
\subfigure{\label{fig:b:1}\includegraphics[width=44mm]{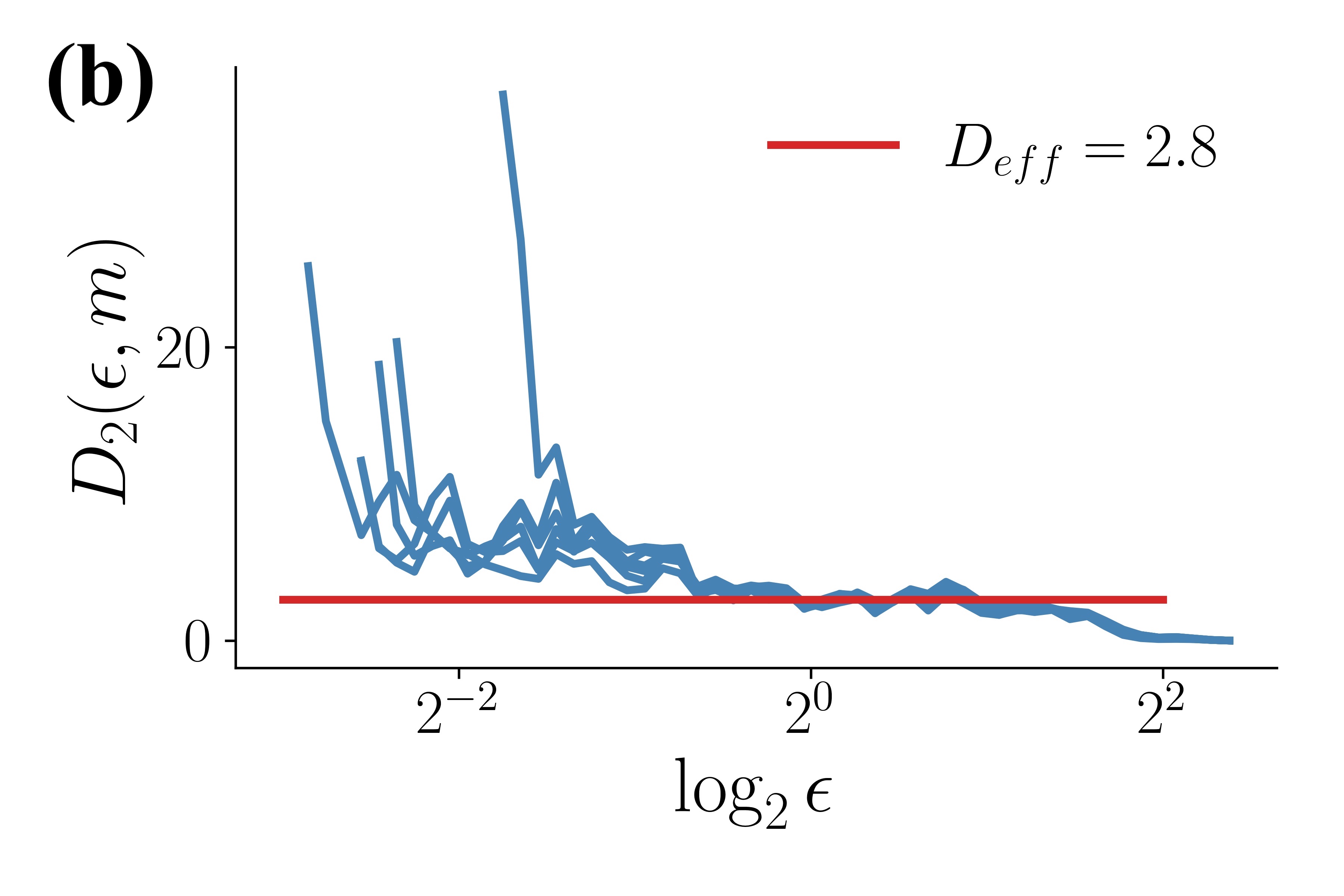}}
\subfigure{\label{fig:b:1}\includegraphics[width=44mm]{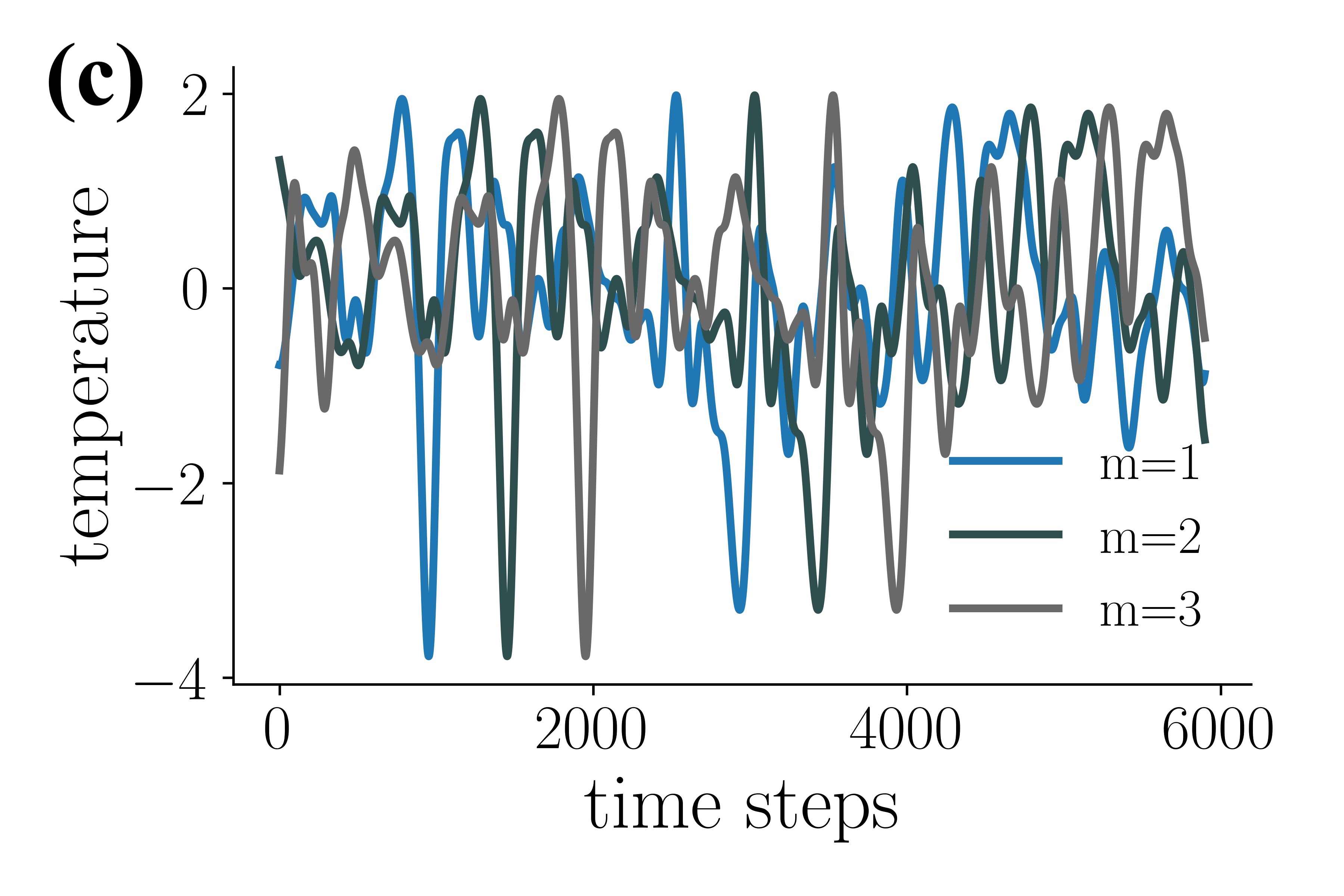}}
\caption{\small(a) Snippet of the original temperature data and de-noised time series. (b) Blue lines show the local slopes of the correlation sums for embedding dimensions $m \in \{5, \dots, 10\}$. The convergence of these estimates in $m$ reveals a fractional dimension indicated by the plateau. (c) First three dimensions of the time-delay embedding series as used for training.}\label{fig:emp-data}
\normalsize
\end{figure}
~\\[4ex]
All datasets used were standardized (i.e., centered with unit variance) prior to training.

\subsection{Empirical evaluation: measures of reconstruction quality}\label{supp:metrics}

\paragraph{Attractor overlap}\label{supp-klx}
To asses the geometrical similarity of the chaotic attractor produced by the RNN to the one underlying the observations, we calculate the Kullback-Leibler divergence of the ground truth distribution $p_{\text{true}}(\vx)$ and the distribution $p_{\text{gen}}(\vx|\vz)$ generated by RNN simulation. To do so in practice, we employ a binning approximation (see \citep{Koppe_nonlin_2019})
\begin{align} \nonumber
D_{\mathrm{stsp}}\left(p_{\mathrm {true }}(\mathbf{x}), p_{\mathrm {gen }}(\mathbf{x} \mid \mathbf{z})\right) \approx \sum_{k=1}^{K} \hat{p}_{\mathrm {true }}^{(k)}(\mathbf{x}) \log \left(\frac{\hat{p}_{\mathrm {true }}^{(k)}(\mathbf{x})}{\hat{p}_{\mathrm {gen }}^{(k)}(\mathbf{x} \mid \mathbf{z})}\right),
\end{align}
where $K$ is the total number of bins, and $\hat{p}_{\text {true }}^{(k)}(\mathbf{x})$ and $\hat{p}_{\text {gen }}^{(k)}(\mathbf{x} \mid \mathbf{z})$ are estimates obtained as relative frequencies through sampling trajectories from the observed time-series and the trained RNN, respectively.

\paragraph{Hellinger distance between power spectra}\label{supp-psc}
Since in DS reconstruction we mainly aim to capture invariant and time-independent properties of the underlying system, besides the geometrical agreement, we compare the similarity in true and RNN-reconstructed power spectra. To do so, we generate a time series of length $100,000$ from the RNN and calculate 
its dimension-wise power spectra $S(\omega)$
using the fast Fourier transform (\texttt{scipy.fft}).
By standardizing all trajectories prior to Fourier transforming them, we have $\int_{-\infty}^\infty S(\omega) = 1$ due to the Plancherel theorem. This allows us to compare two power spectra, $S(\omega)$ and $P(\omega)$, with the Hellinger distance
\begin{align}
    H(S(\omega),P(\omega))= \sqrt{1-\int_{-\infty}^{\infty}\sqrt{S(\omega)P(\omega)}\ d\omega} \ \ \in \ [0,1].
\end{align}
To reduce the influence of noise we apply Gaussian kernel smoothing. The Hellinger distances between observed and generated spectra for all dimensions are then averaged to give the reported overall distance $D_H$.
\newpage
\subsection{Further empirical evaluations}
\label{sec-supp-emp}
\subsubsection{Reconstruction: Rössler System}
\label{sec:roessler}
\begin{figure}[!h]  
\centering    
\subfigure{\label{fig:a:1}\includegraphics[width=0.4\linewidth]{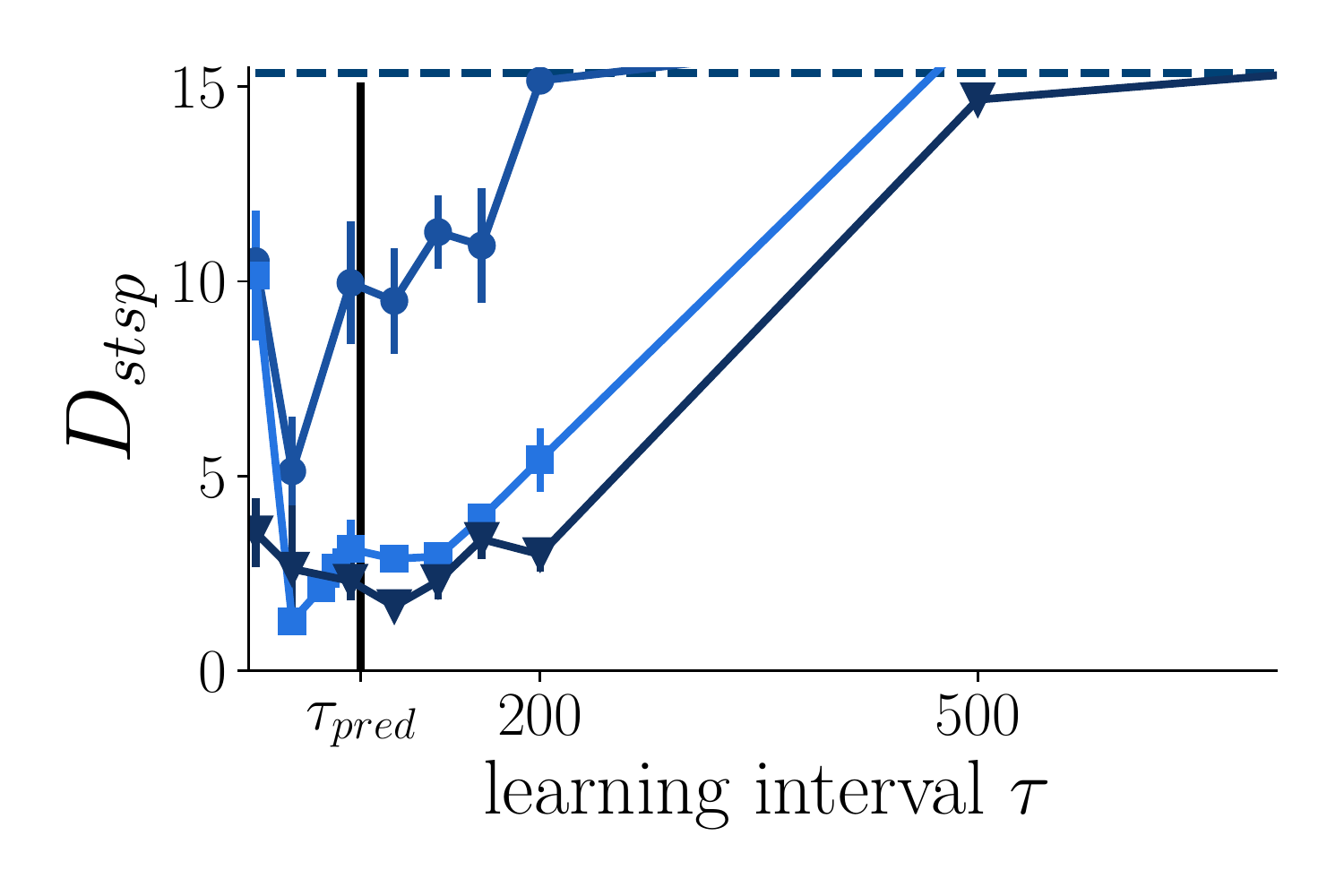}}
\hspace{0.5cm}
\subfigure{\label{fig:b:1}\includegraphics[width=0.4\linewidth]{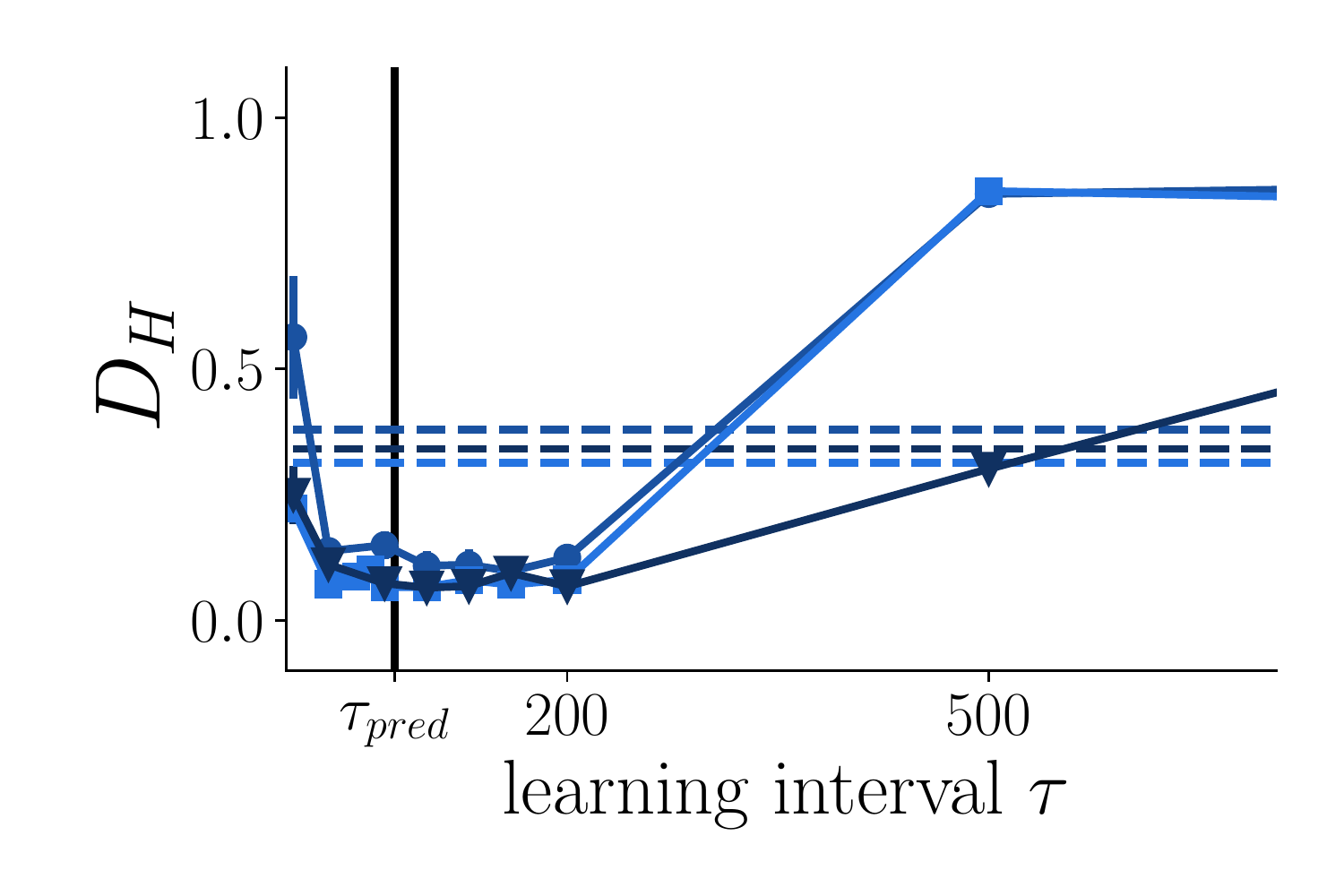}}
\vspace{-0.3cm}
\caption{\small Overlap in attractor geometry ($D_{stsp}$, lower = better) and dimension-wise comparison of power-spectra  ($D_H$, lower = better) against learning interval $\tau$ for the Rössler attractor. Continuous lines = 
sparsely forced BPTT. 
Dashed lines = 
classical BPTT with gradient clipping. Prediction time indicated vertically in black.}\label{fig:tau-sweep-roessler}
\end{figure}
\normalsize
\begin{figure}[!h] 
\centering    
\subfigure{\label{fig:a:1}\includegraphics[scale=0.06]{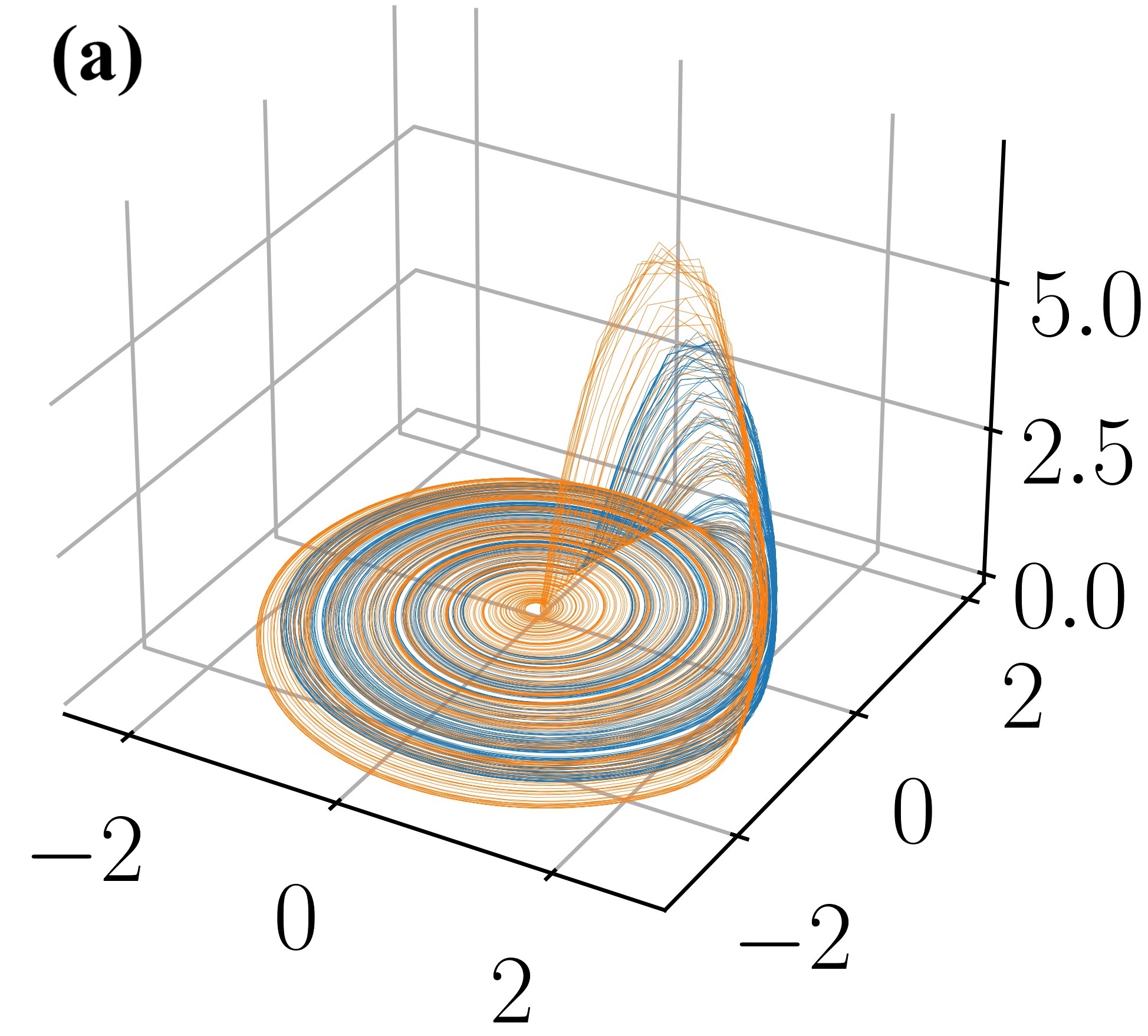}}
\hspace{0.1cm}
\subfigure{\includegraphics[scale=0.06]{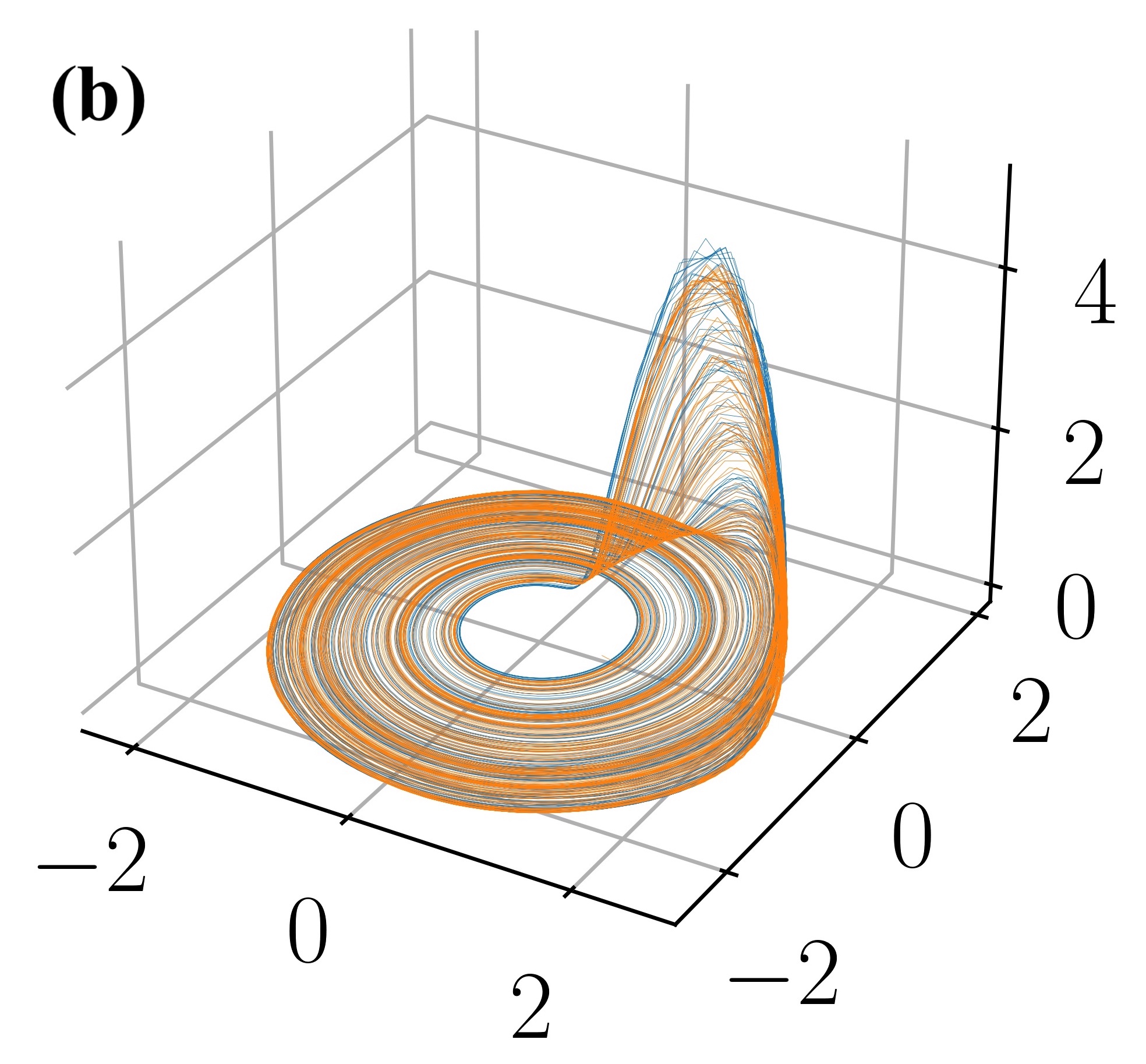}}
\hspace{0.1cm}
\subfigure{\includegraphics[scale=0.06]{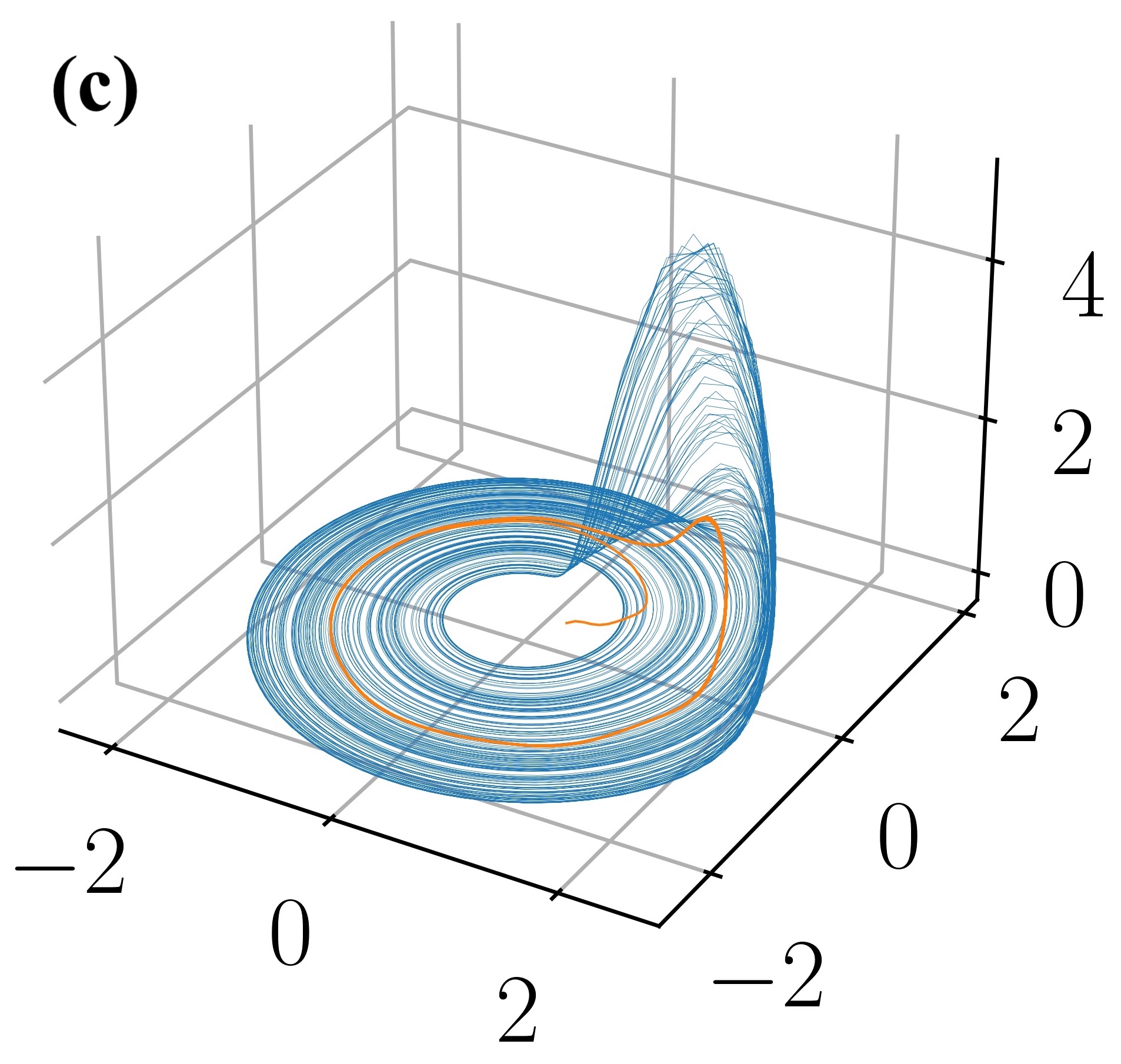}}
\caption{\small The Rössler attractor (blue) and  reconstruction by a LSTM (orange) trained with a learning interval (a) chosen too small ($\tau = 5$), (b) chosen optimally ($\tau = 30$), and (c) chosen too large ($\tau=200$). 
}\label{fig:roessler-reconstruction}
\normalsize
\end{figure}
\subsubsection{Reconstruction: High-dimensional Mackey-Glass system}
\label{sec:MackeyGlass}
\begin{figure}[!h]   
\centering    
\subfigure{\label{fig:a:1}\includegraphics[width=0.4\linewidth]{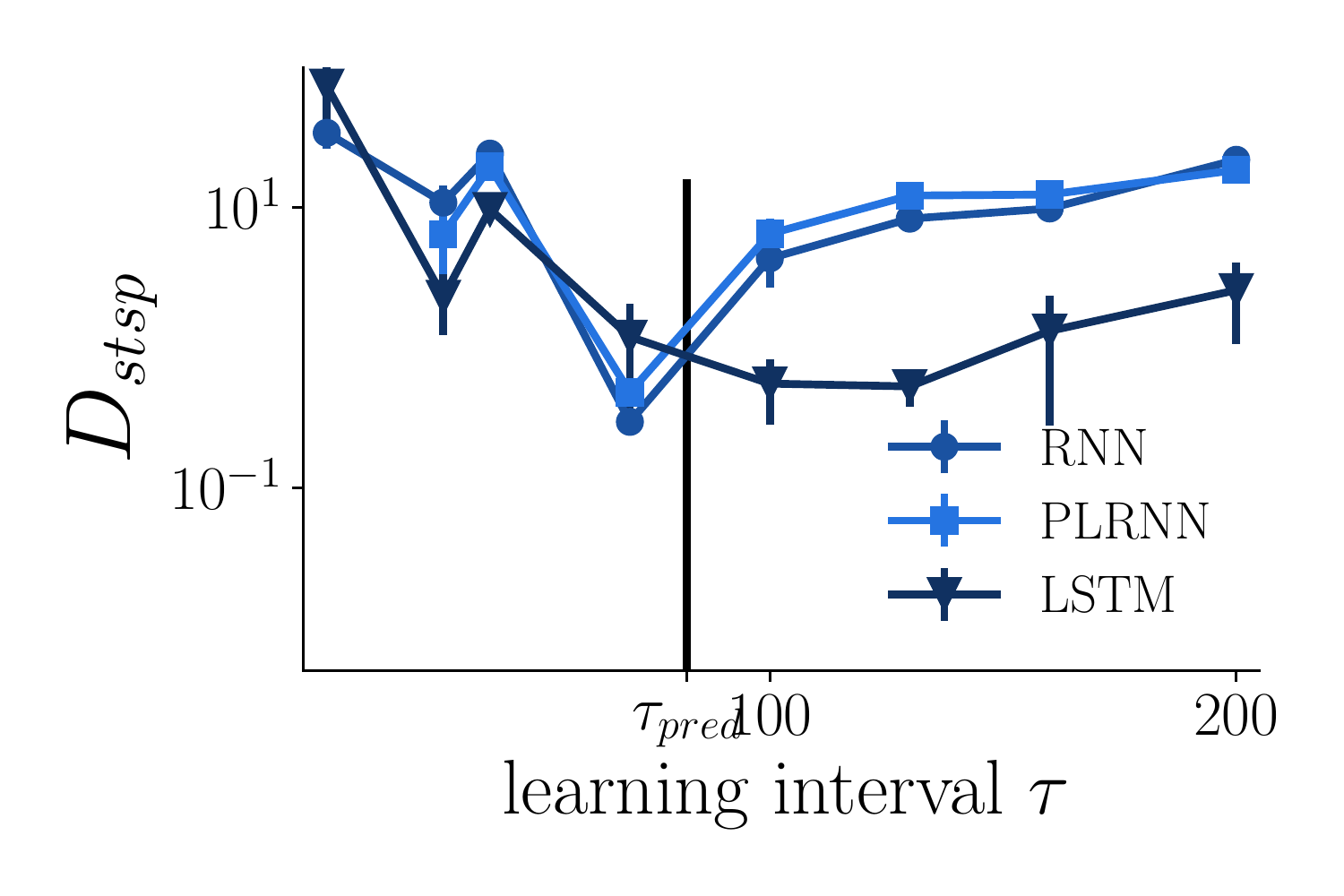}}
\hspace{0.5cm}
\subfigure{\label{fig:b:1}\includegraphics[width=0.4\linewidth]{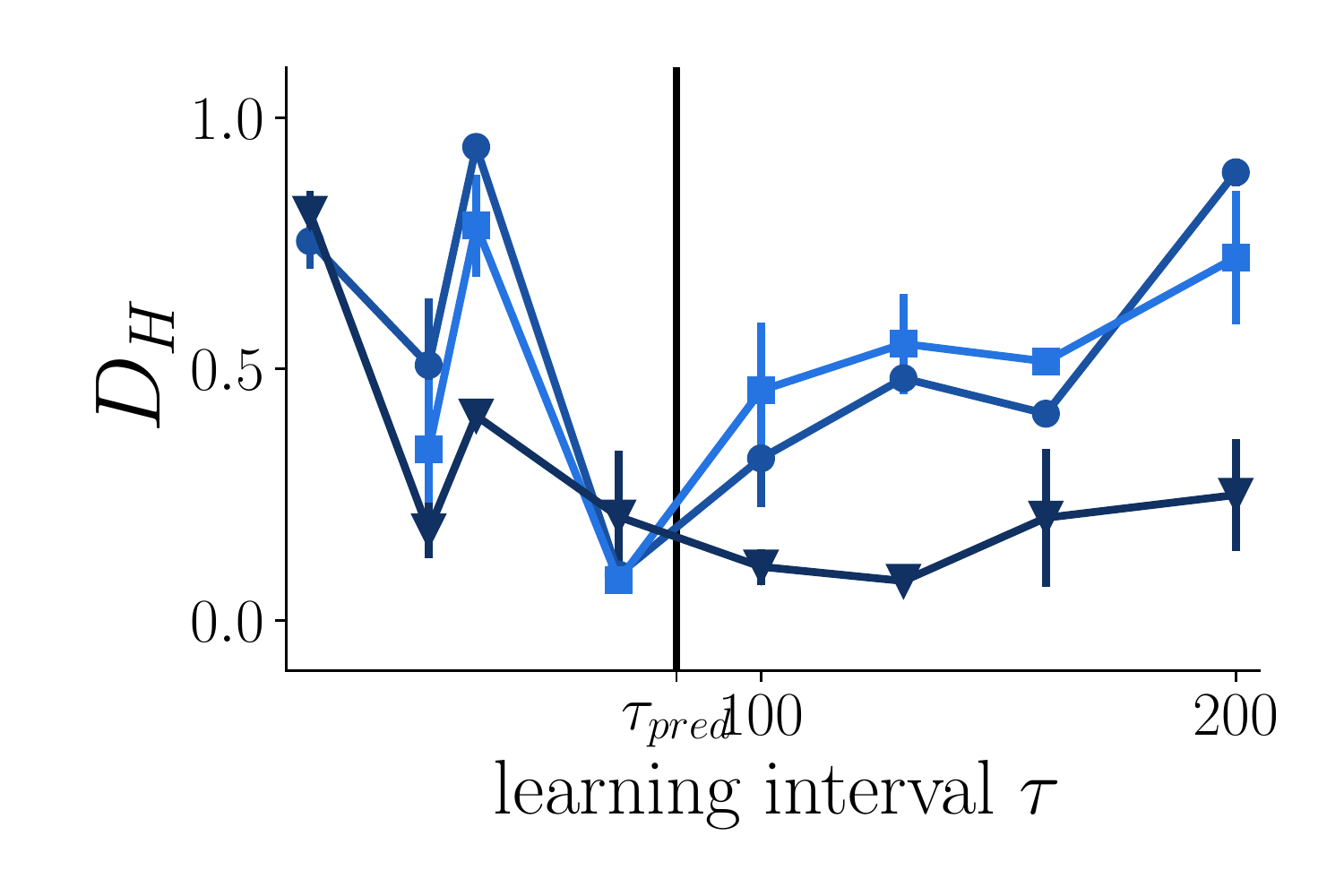}}
\vspace{-0.3cm}
\caption{\small Overlap in attractor geometry ($D_{stsp}$, lower = better) and dimension-wise comparison of power-spectra  ($D_H$, lower = better) against learning interval $\tau$ for the 10d Mackey-Glass system. Continuous lines = 
sparsely forced BPTT. 
Prediction time indicated vertically in black.}\label{fig:tau-sweep-MackeyGlass}
\end{figure}
\normalsize
\subsubsection{Reconstruction: Partially observed Lorenz System}
\label{sec:partObs}
For this evaluation we trained models only on the variables $\{y,z\}$ of the Lorenz system, eqn. \ref{eq-Lorenz}. In order to compute the attractor overlap ($D_{stsp}$) in the true state space, however, after training the observation matrix $\mB$ was recomputed by linearly regressing the first 10 latent states onto the first 10 observations from all three Lorenz variables in eqn. \ref{eq-Lorenz}.
\begin{figure} [H]  
\centering    
\subfigure{\label{fig:a:1}\includegraphics[width=0.4\linewidth]{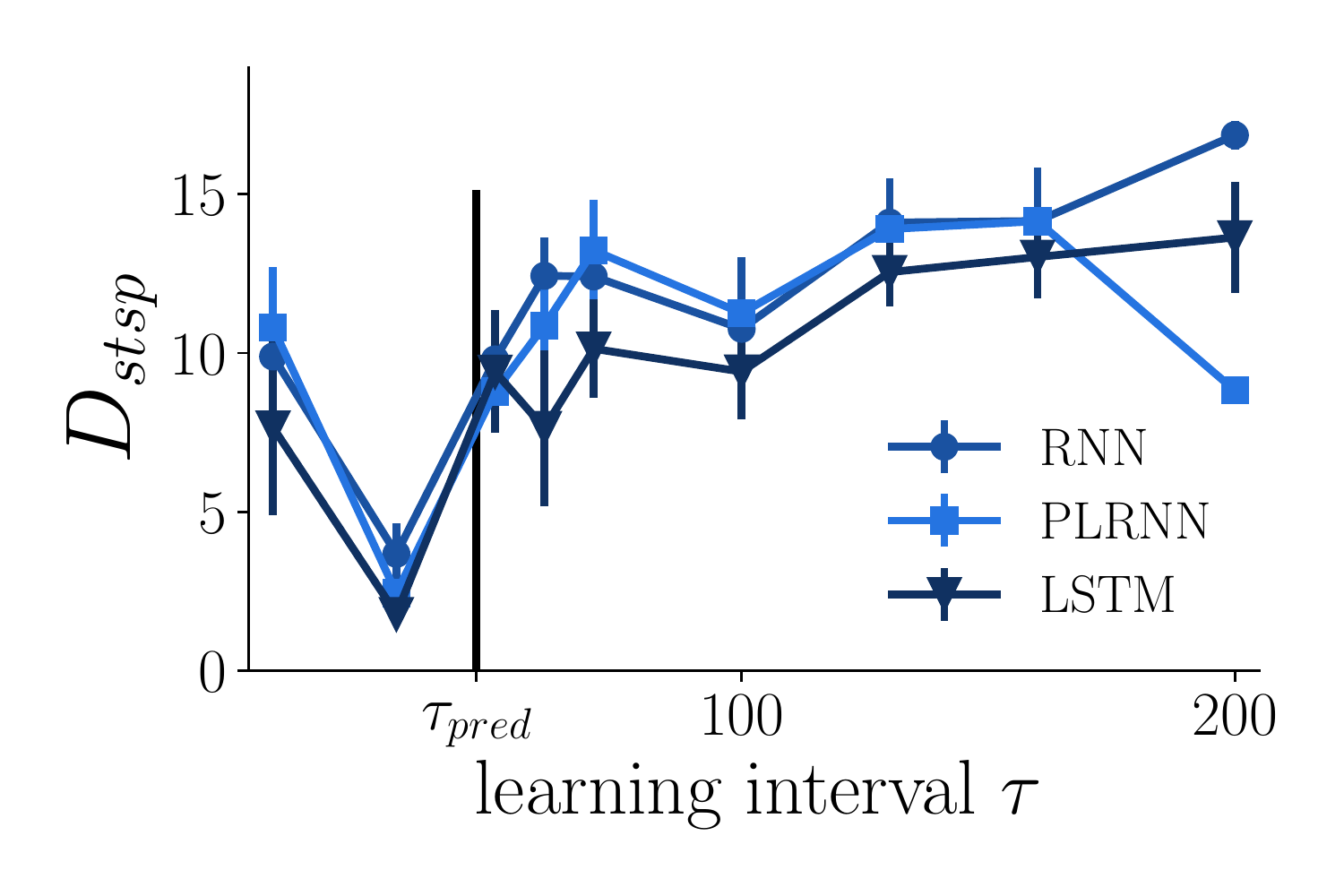}}
\hspace{0.5cm}
\subfigure{\label{fig:b:1}\includegraphics[width=0.4\linewidth]{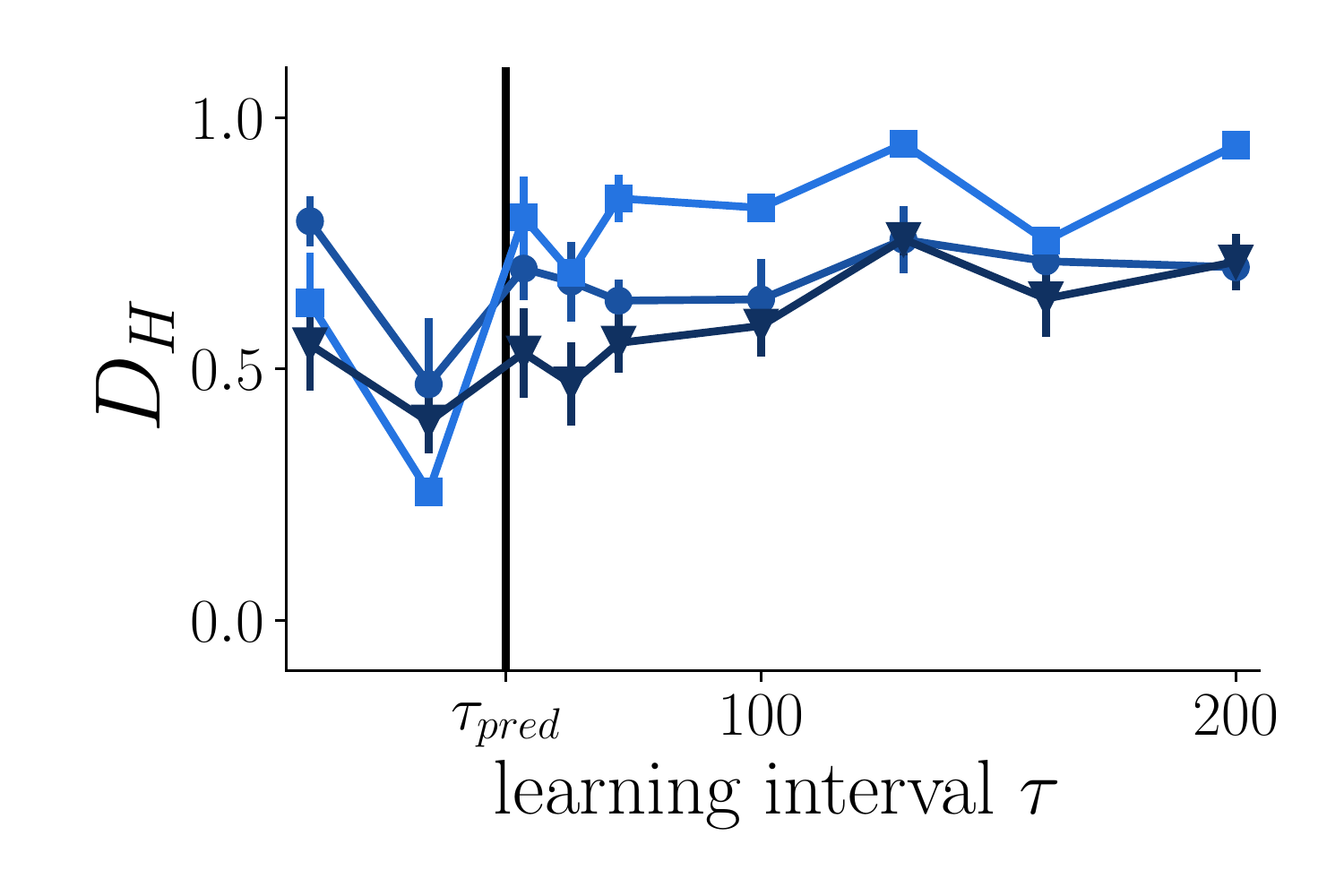}}
\caption{\small Overlap in attractor geometry ($D_{stsp}$, lower = better) and dimension-wise comparison of power-spectra  ($D_H$, lower = better) against learning interval $\tau$ for the partially observed Lorenz system. Continuous lines = 
sparsely forced BPTT. 
Prediction time indicated vertically in black.}\label{fig:tau-sweep-partObs}
\end{figure}
\normalsize
\subsubsection{Other initialization procedures: 
Truncated BPTT with zero resetting or forward-iterated states}
\label{sec:windowing}
A common procedure in training RNNs is partitioning the time series into chunks of length $\tau$ (as we did based on the Lyapunov spectrum), but then simply resetting the hidden states 
$\vz_{1(k)}$ at the beginning of each chunk (window) $k$ to $\mathbf{0}$, or forward-iterating them from the previous chunk $k-1$, i.e. $\vz_{1(k)}=F_{\boldsymbol\theta}(\vz_{\tau(k-1)})$. Formally this would mean that we do not force the trajectory back on track as in our approach, but instead may 
either kick it off track (zero resetting) or just let it freely evolve whilst still truncating the gradients (forward-iterating). To illustrate this, here we trained an LSTM on chunks (windows) with a length given by the optimal $\tau$ ($\tau_{opt}=30$ for the Lorenz system), but then initialized the hidden states to 
either $0$ or to the forward-iterated last state at the beginning of each window. The performance obtained 
with zero-resetting is indicated by the 
dashed line in Fig. \ref{fig:windowing}a below, 
while the performance with forward-iterated states is shown in Fig. \ref{fig:windowing}c. As another control, we also checked dependence on window length (without forcing) in Fig. \ref{fig:windowing}b.
\begin{figure}[H]  
\centering    
\subfigure{\label{fig:a:1}\includegraphics[width=0.4\linewidth]{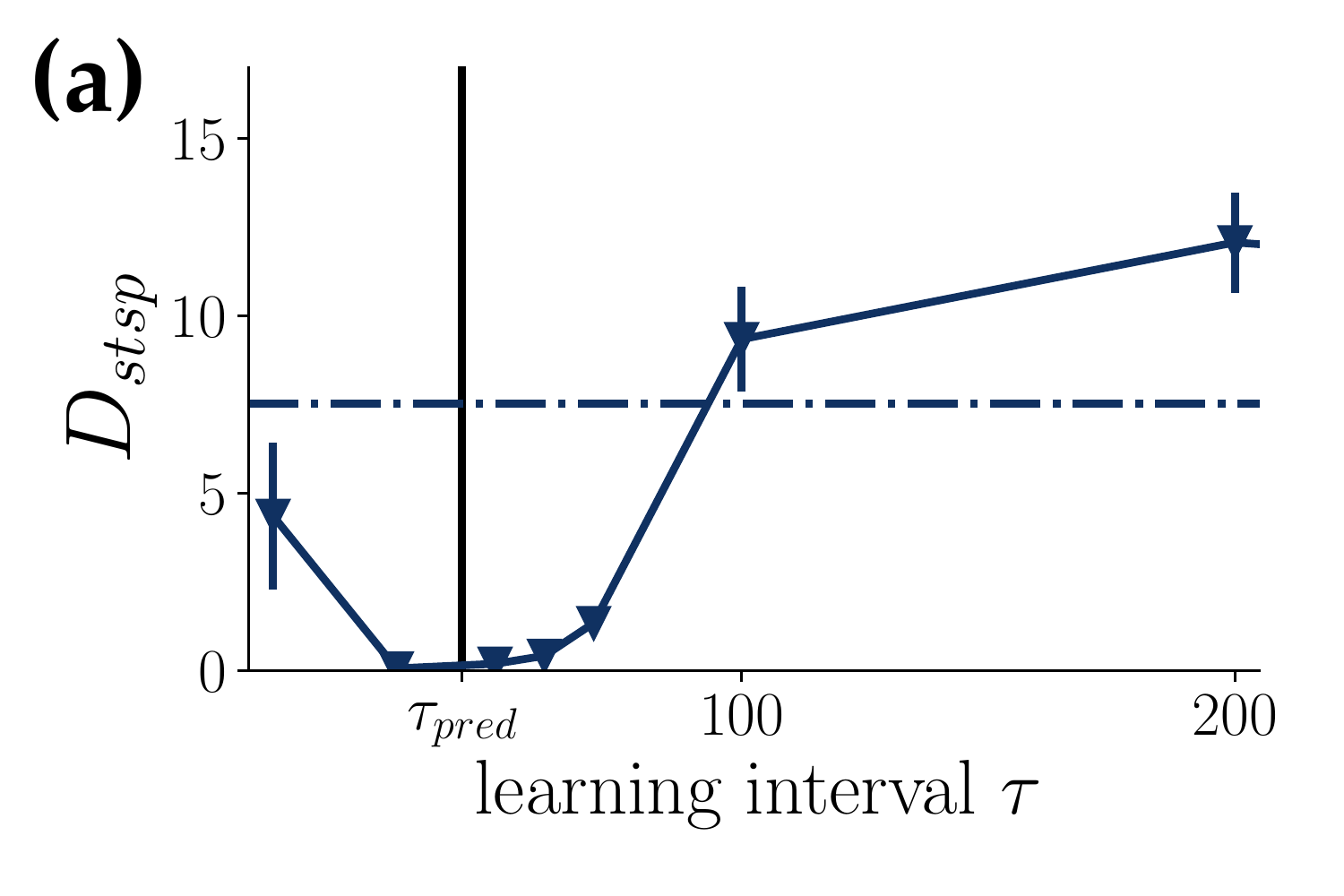}}
\hspace{0.5cm}
\subfigure{\label{fig:b:1}\includegraphics[width=0.4\linewidth]{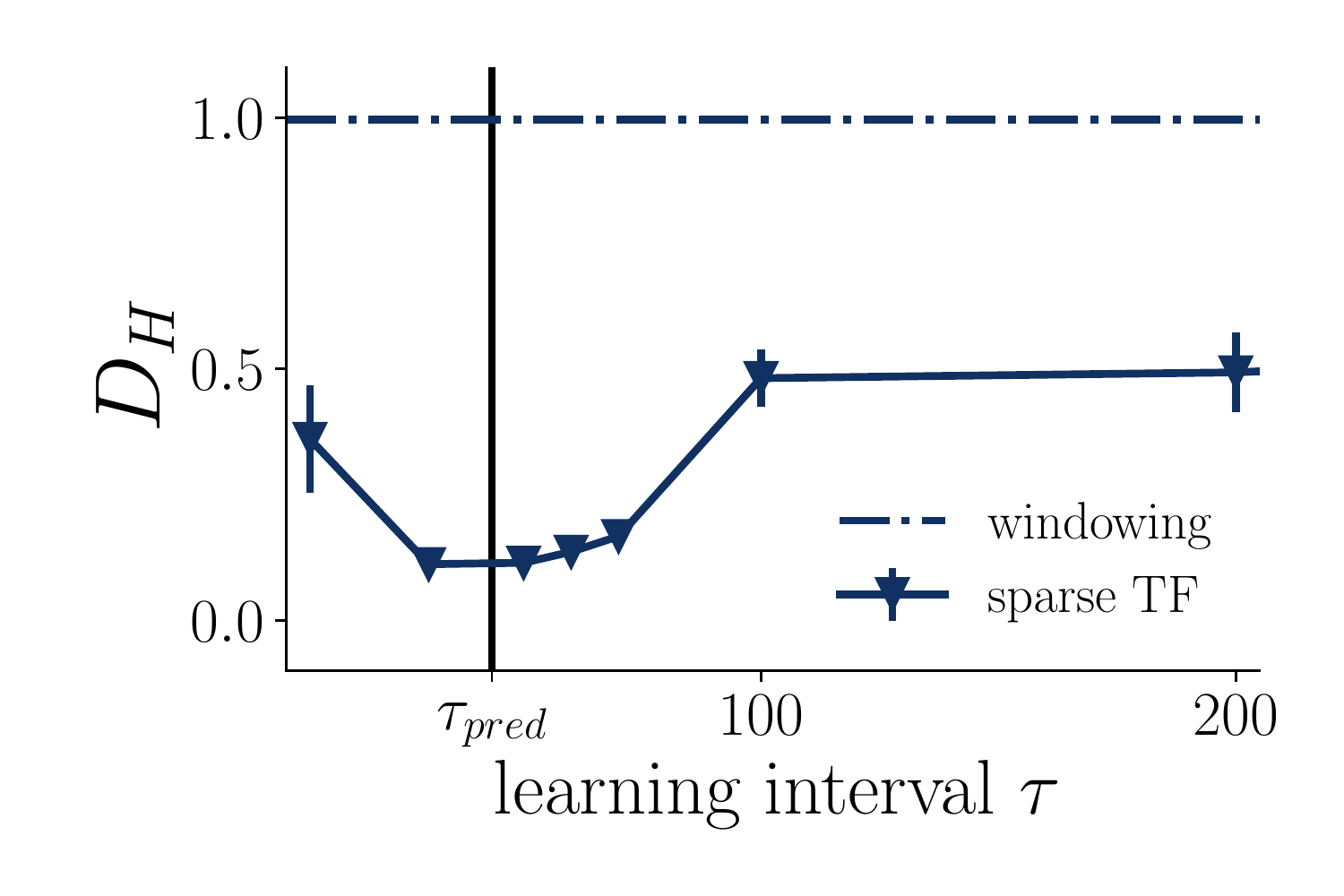}}
\subfigure{\label{fig:a:1}\includegraphics[width=0.4\linewidth]{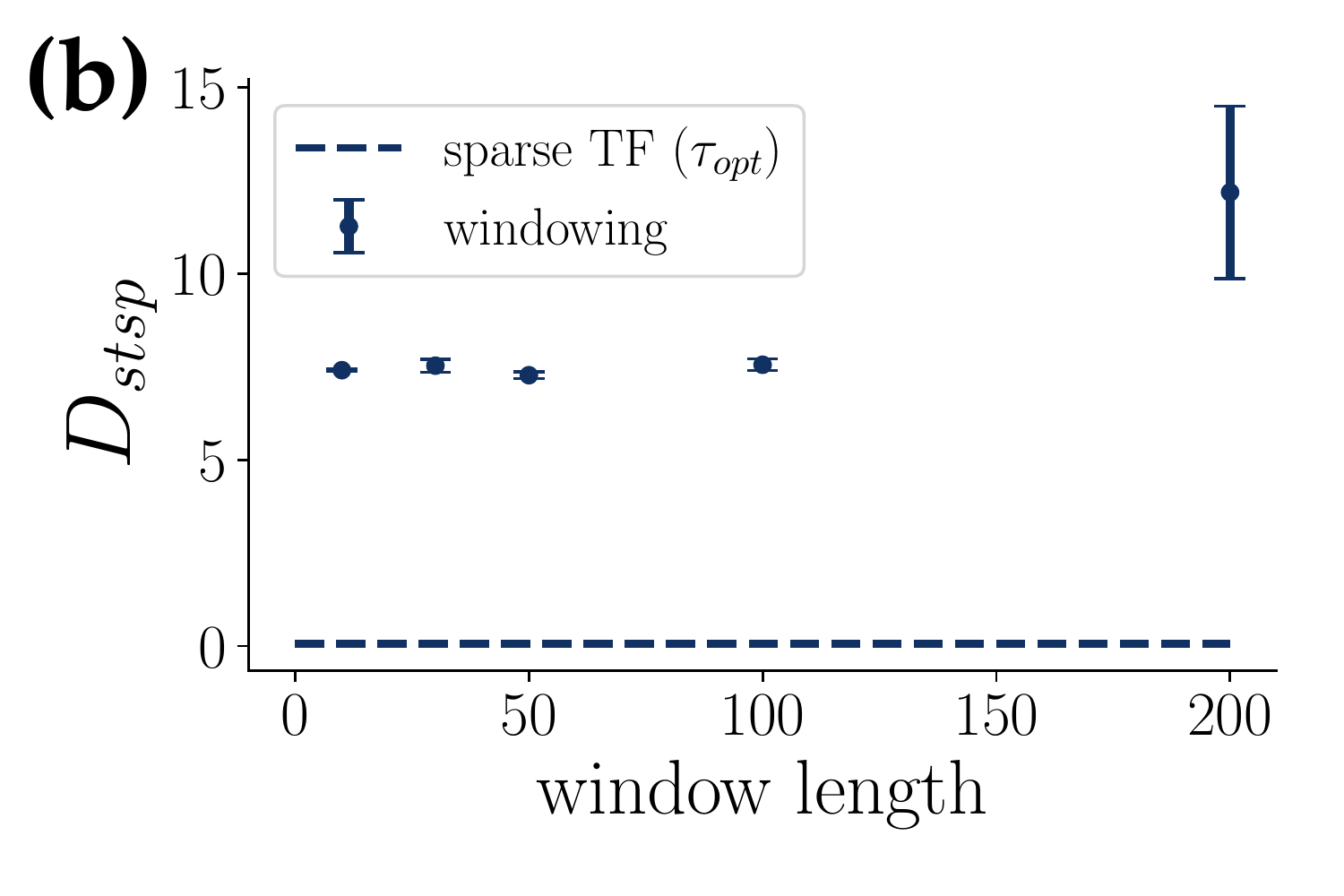}}
\hspace{0.5cm}
\subfigure{\label{fig:a:1}\includegraphics[width=0.4\linewidth]{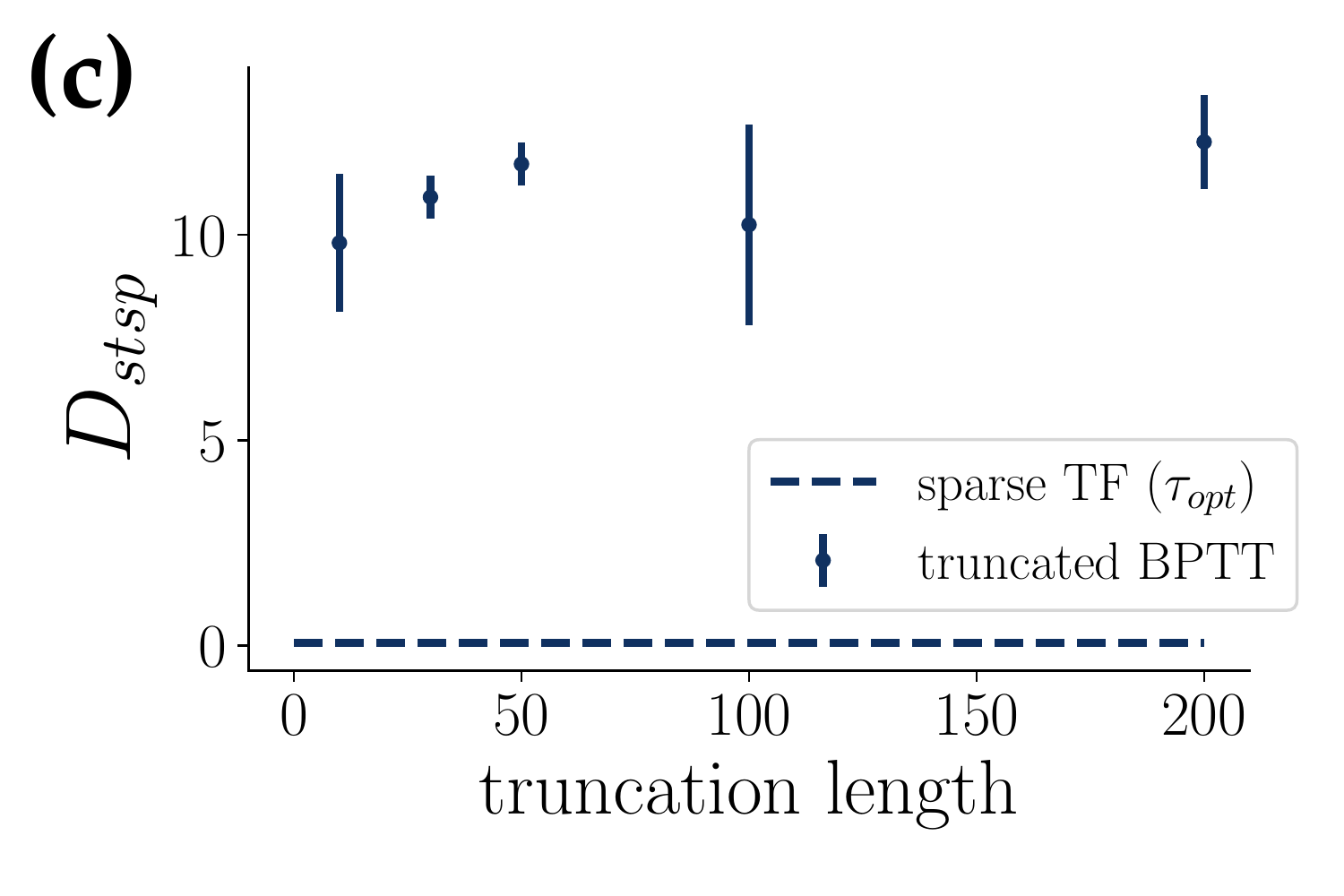}}
\vspace{-0.3cm}
\caption{\small (a) Overlap in attractor geometry ($D_{stsp}$, lower = better) and dimension-wise comparison of power-spectra  ($D_H$, lower = better) against learning interval $\tau$ for the Lorenz system. Continuous lines = 
sparsely forced BPTT. 
Dashed-dotted lines = 
windowing without forcing (choosing windows according to the optimal prediction time, but resetting hidden states to zero rather than its TF control value). Prediction time indicated vertically in black. 
(b) Dependence of geometrical reconstruction quality on window length. Without forcing, the window length hardly has any bearing on reconstruction quality. (c) Same as (b) but with initial states of each window $k$ forward-iterated from the previous window's state, $\vz_{1(k)}=F_{\boldsymbol\theta}(\vz_{\tau(k-1)})$, instead of zero resetting.}\label{fig:windowing}
\normalsize
\end{figure}
\subsubsection{Electroencephalogram (EEG) data}
\label{Sec:EEG-analysis}
We used EEG data recorded by \citet{schalk_bci2000} and provided on PhysioNet \citep{PhysioNet}, from which we took the baseline recording of the first patient for our analysis. Preprocessing was performed as outlined above for the temperature time-series, i.e. we applied nonlinear noise-reduction (see Fig.\ref{fig:emp-data-EEG} (a)) and Gaussian kernel smoothing ($\sigma = 5$). Fig. \ref{fig:emp-data-EEG} (b) indicates a fractional dimension $D_{eff}=2.5$ for the de-noised and smoothed times series. We created a time delay embedding with an embedding dimension of $m=10$ and a delay time of $\Delta t = 40$. The maximal Lyapunov exponent for this time series was determined to be $\lambda_{max} = 0.017$, see Fig. \ref{fig:tau-sweep-emp-data-EEG} (a). With this, we obtain a predictability time $\tau_{pred} = 40.77$.
\begin{figure}[H]
\centering    
\subfigure[]{\label{fig:a:1}\includegraphics[width=44mm]{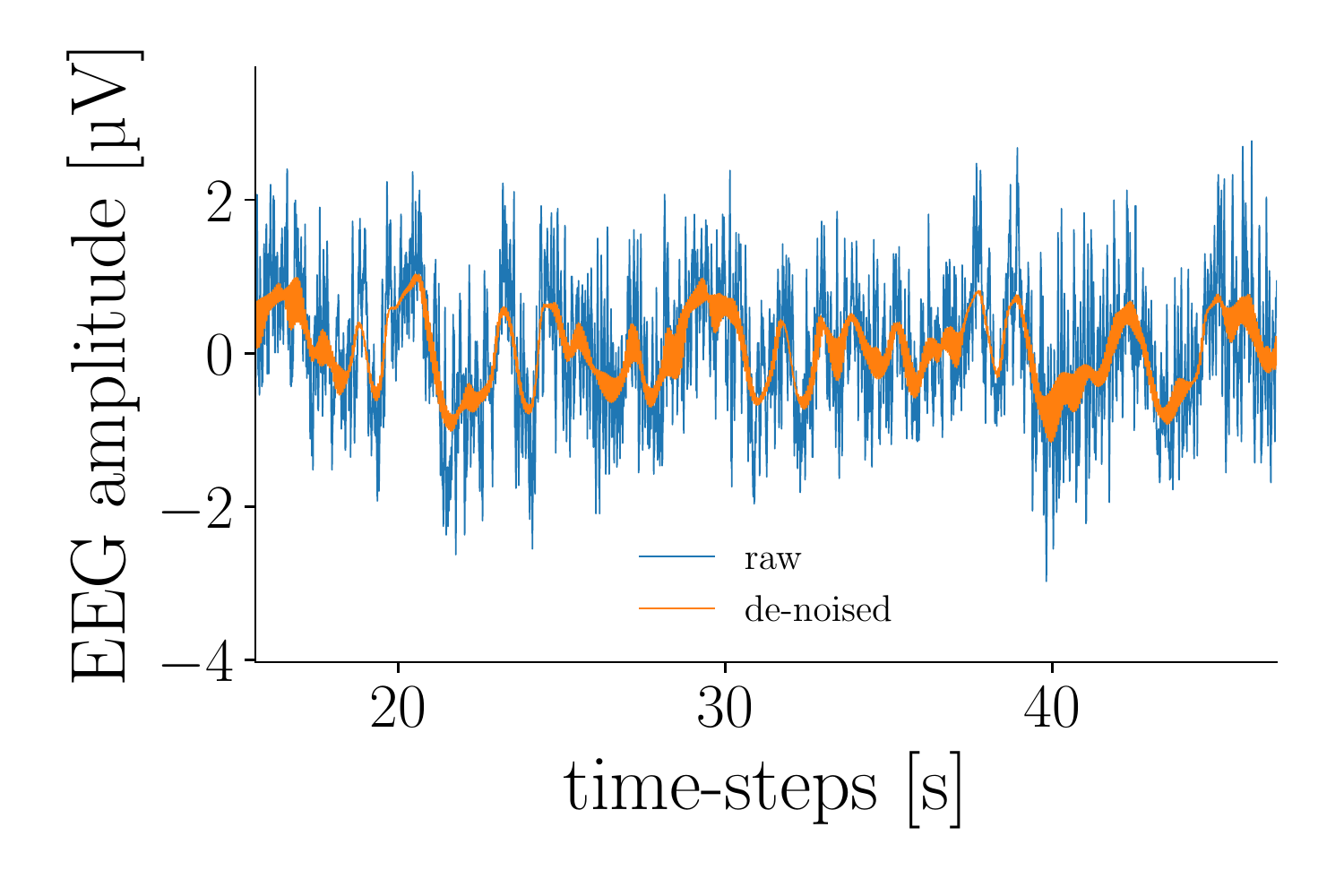}}
\subfigure[]{\label{fig:b:1}\includegraphics[width=44mm]{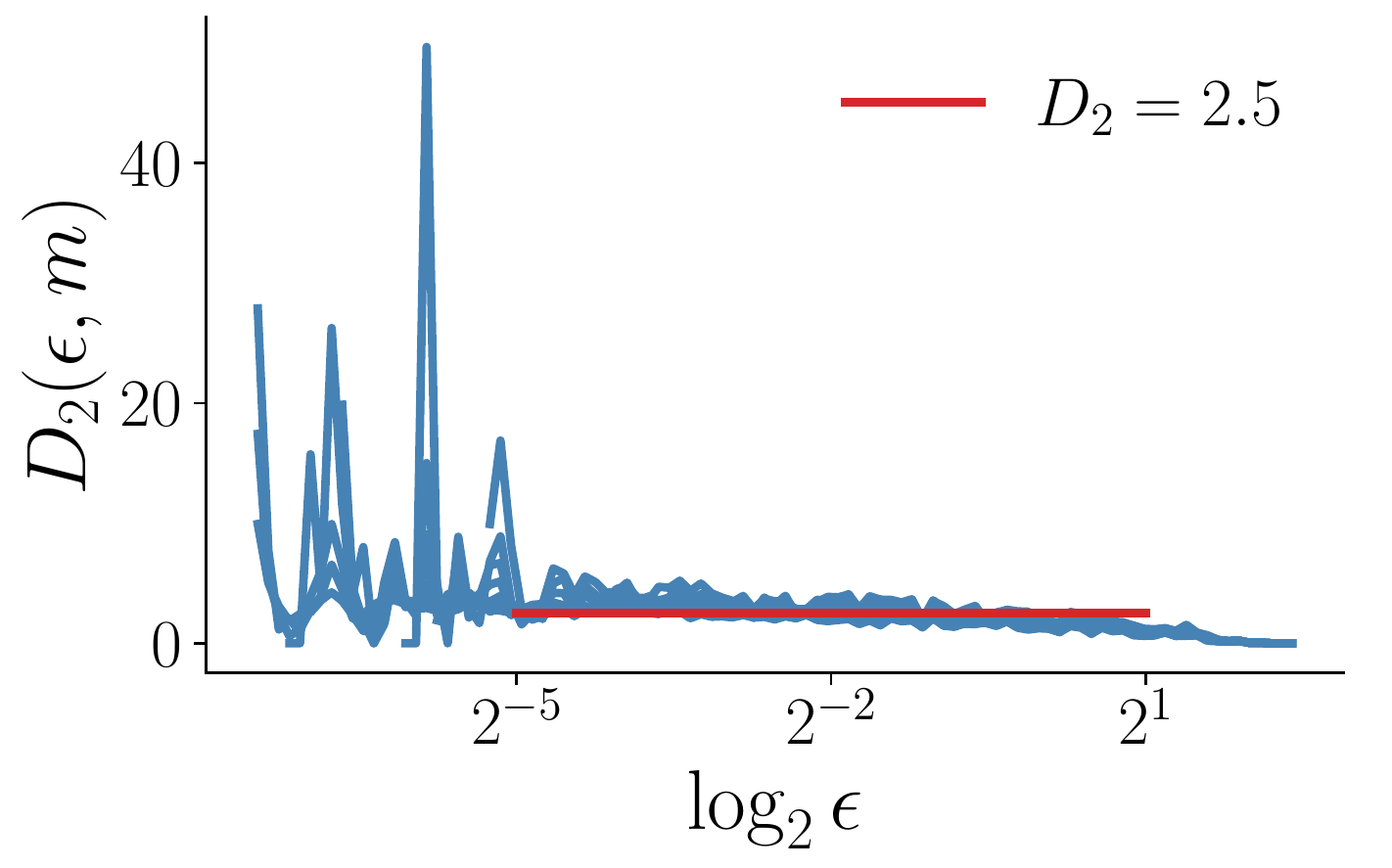}}
\subfigure[]{\label{fig:b:1}\includegraphics[width=38mm]{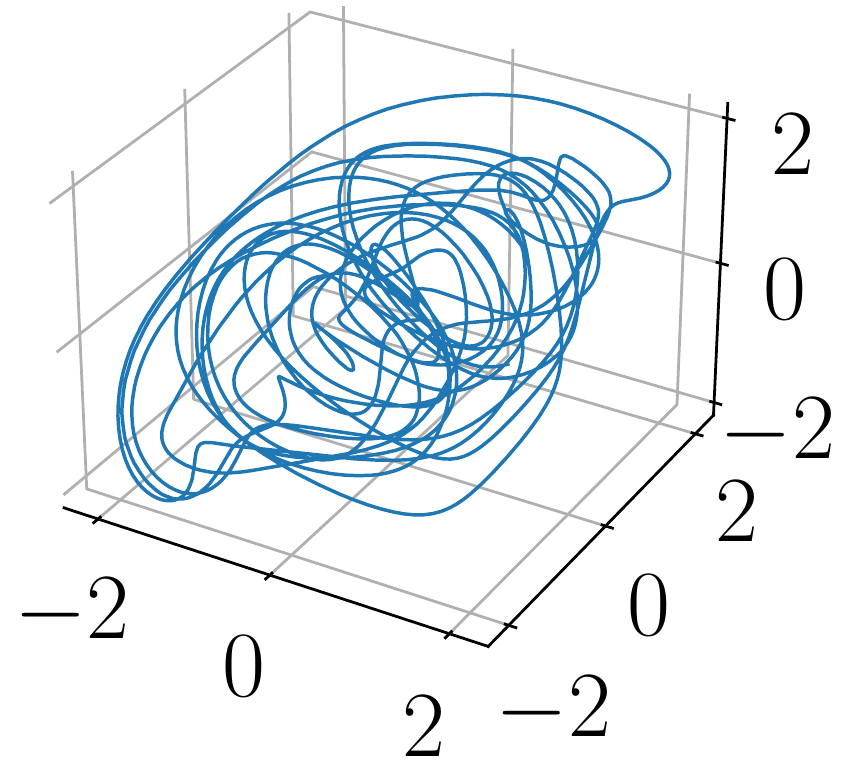}}
\caption{\small(a) Snippet of the original EEG data and de-noised time series. (b) Blue lines show the local slopes of the correlation sums for embedding dimensions $m \in \{5, \dots, 15\}$. The convergence of these estimates in $m$ reveals a fractional dimension indicated by the plateau. (c) First three dimensions of the time-delay embedding series as used for training.}\label{fig:emp-data-EEG}
\normalsize
\end{figure}
\begin{figure}[H]
\centering
\subfigure{\label{fig:a:1}\includegraphics[width=0.32\linewidth]{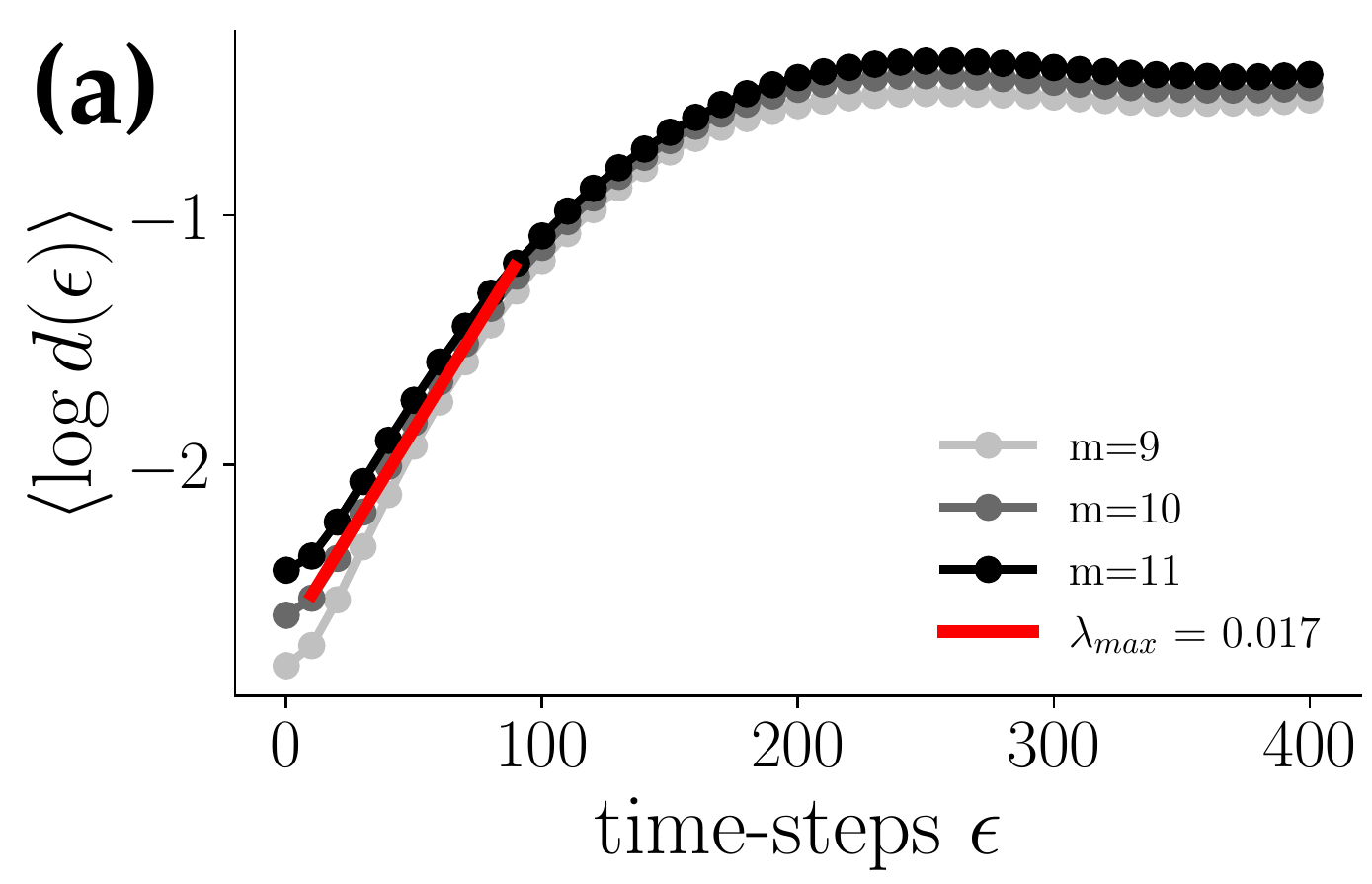}}
\subfigure{\label{fig:b:1}\includegraphics[width=0.32\linewidth]{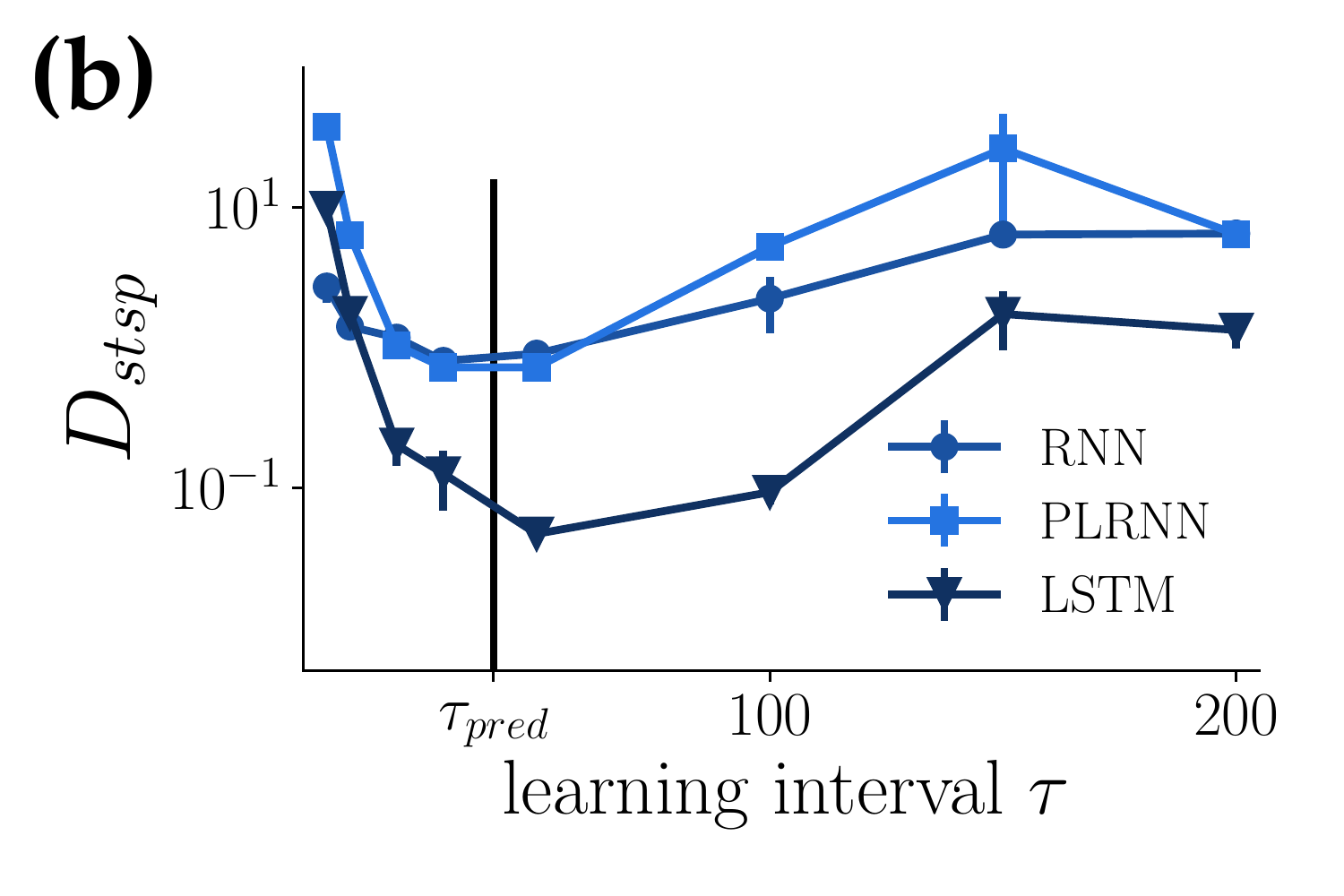}}
\subfigure{\label{fig:a}\includegraphics[width=0.32\linewidth]{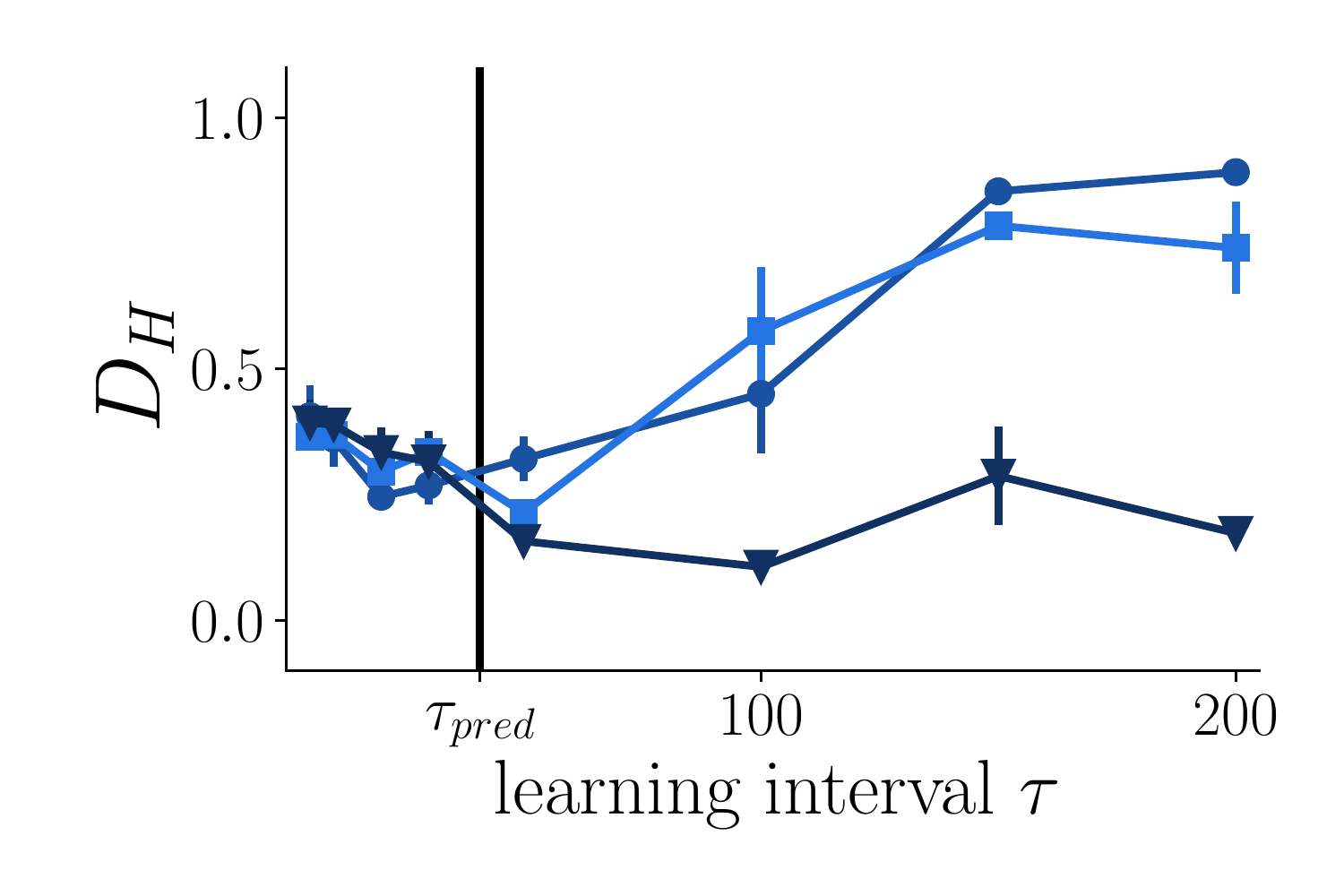}}
\vspace{-0.1cm}
\caption{\small (a) The maximal Lyapunov exponent was determined as the slope of the average log-divergence of nearest neighbors in embedding space ($m=$ embedding dimension). (b) Reconstruction quality assessed by attractor overlap (lower = better) and dimension-wise comparison of power-spectra  ($D_H$, lower = better). Black vertical lines = $\tau_{\text{pred}}$. }\label{fig:tau-sweep-emp-data-EEG}
\end{figure}
\normalsize
{\subsubsection{ Miscellaneous additional results}}
\begin{figure}[h]  
\centering    
\subfigure{\label{fig:a:1}\includegraphics[width=0.4\linewidth]{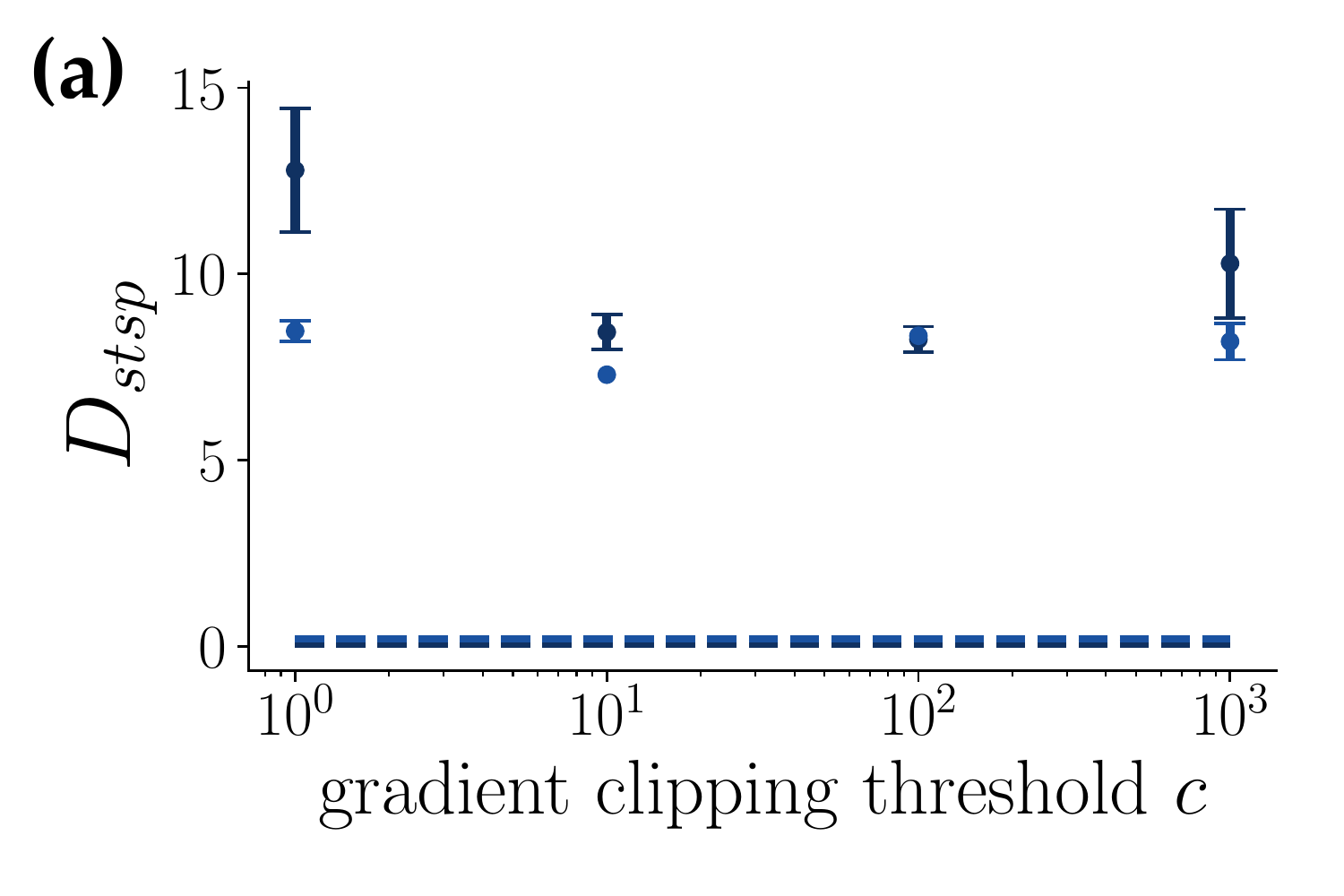}}
\hspace{0.5cm}
\subfigure{\label{fig:b:1}\includegraphics[width=0.4\linewidth]{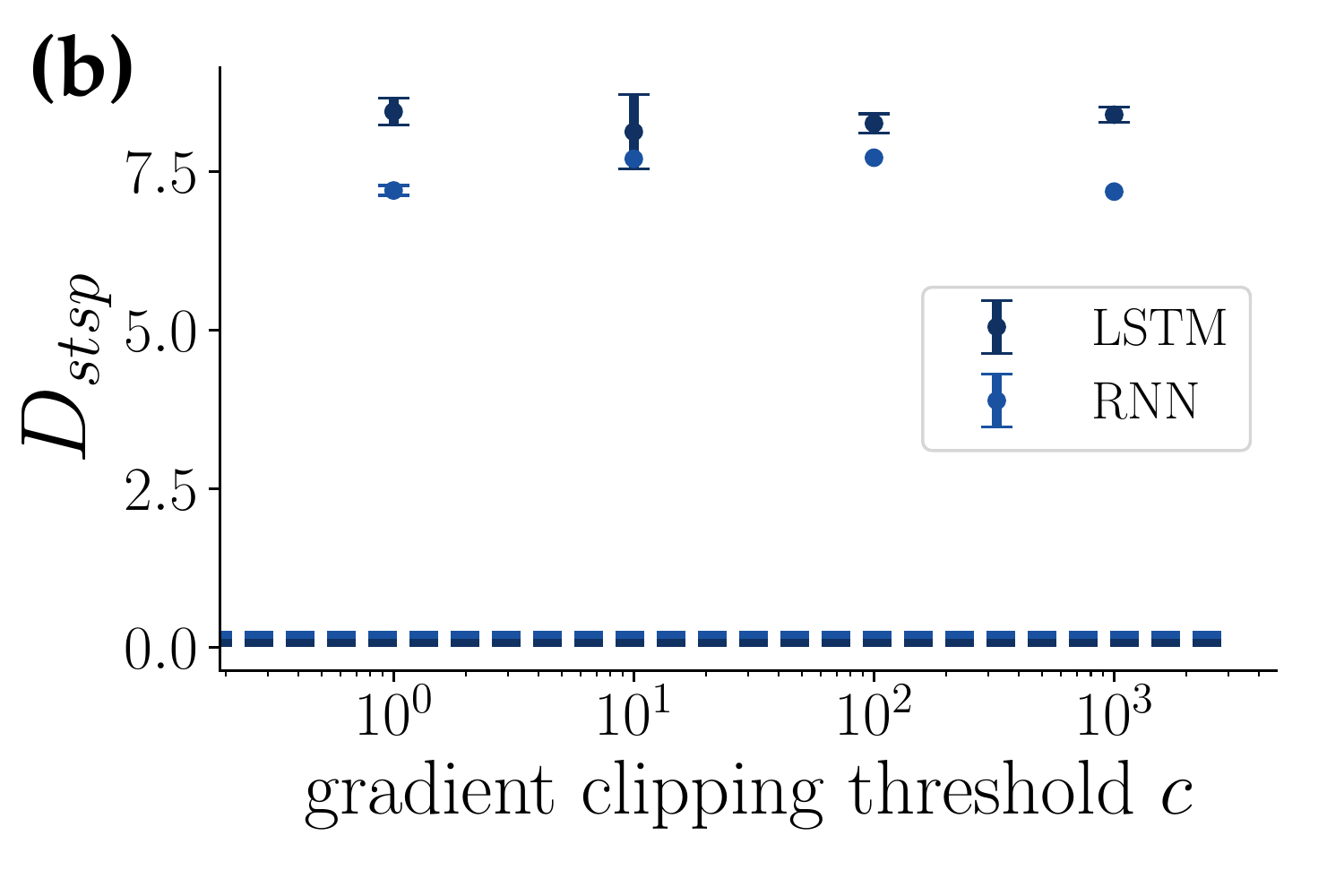}}
\vspace{-0.3cm}
\caption{\small 
Dependence of geometrical reconstruction quality ($D_{stsp}$) on the Lorenz system for various clipping thresholds 
in classical BPTT. (a) Gradient clipping by constraining the Euclidean norm to $c$. (b) Gradient clipping by constraining the max (infinity) norm to $c$. For comparison, in both graphs the values obtained for sparse teacher forcing with optimal forcing interval $\tau_{pred}$ are shown as dashed lines. 
}\label{fig:gradClipping}
\normalsize
\end{figure}
\begin{figure}[h] 
\centering    
\subfigure[learning interval $\tau$ too small]{\label{fig:a:1}\includegraphics[width=0.32\linewidth]{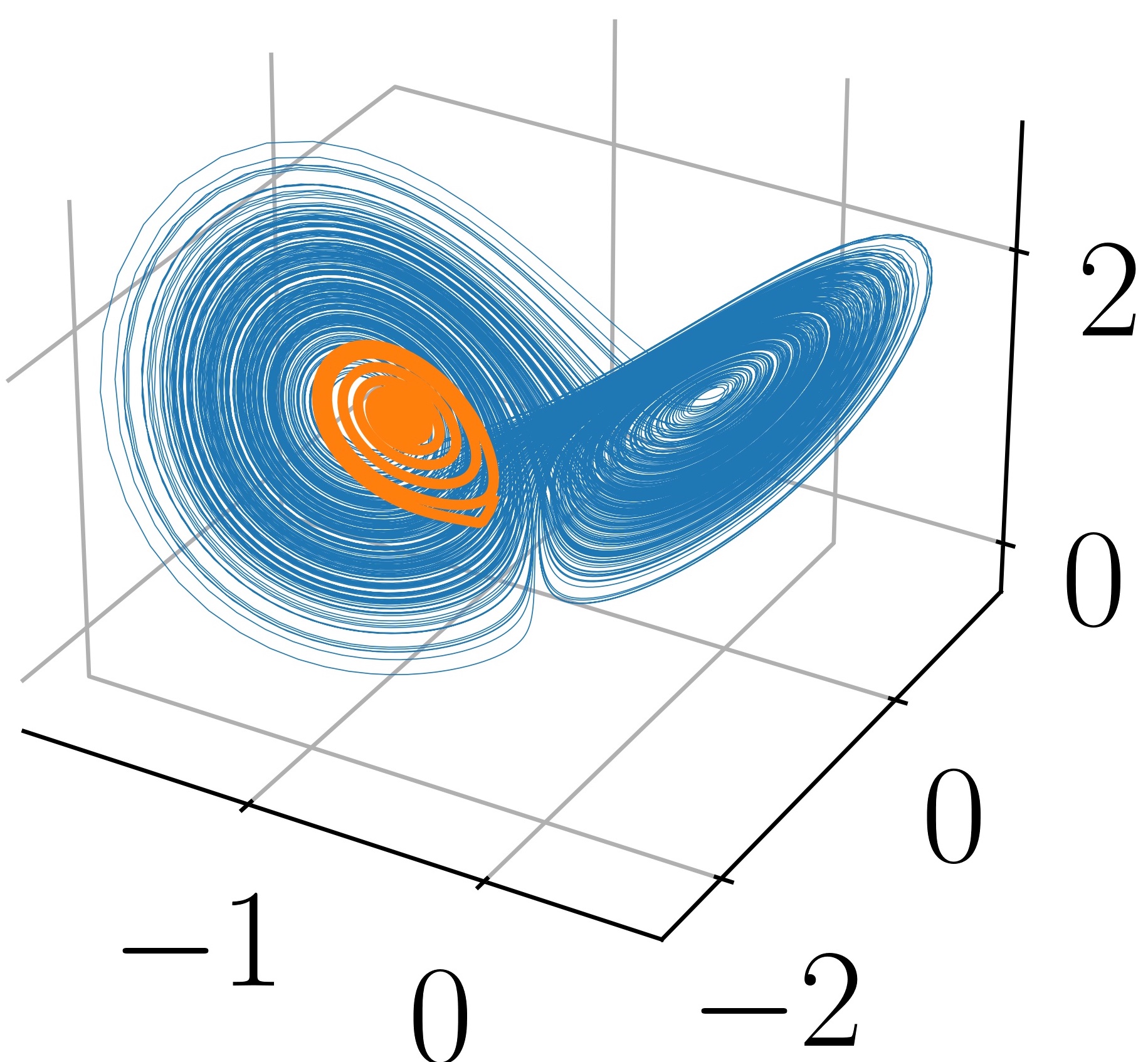}}
\hspace{0.1cm}
\subfigure[learning interval $\tau$ optimal]{\includegraphics[width=0.32\linewidth]{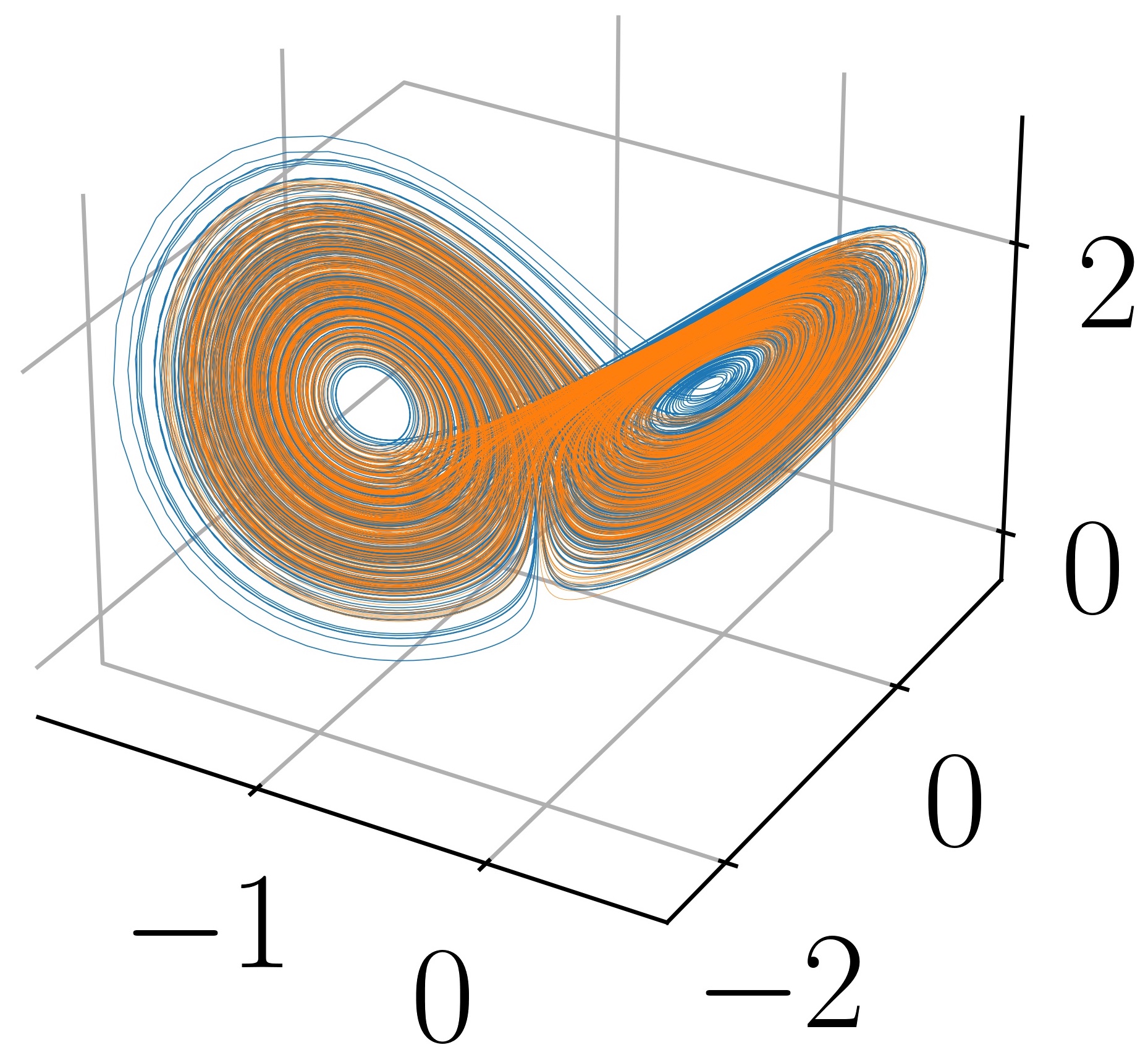}}
\hspace{0.1cm}
\subfigure[learning interval $\tau$ too large]{\includegraphics[width=0.32\linewidth]{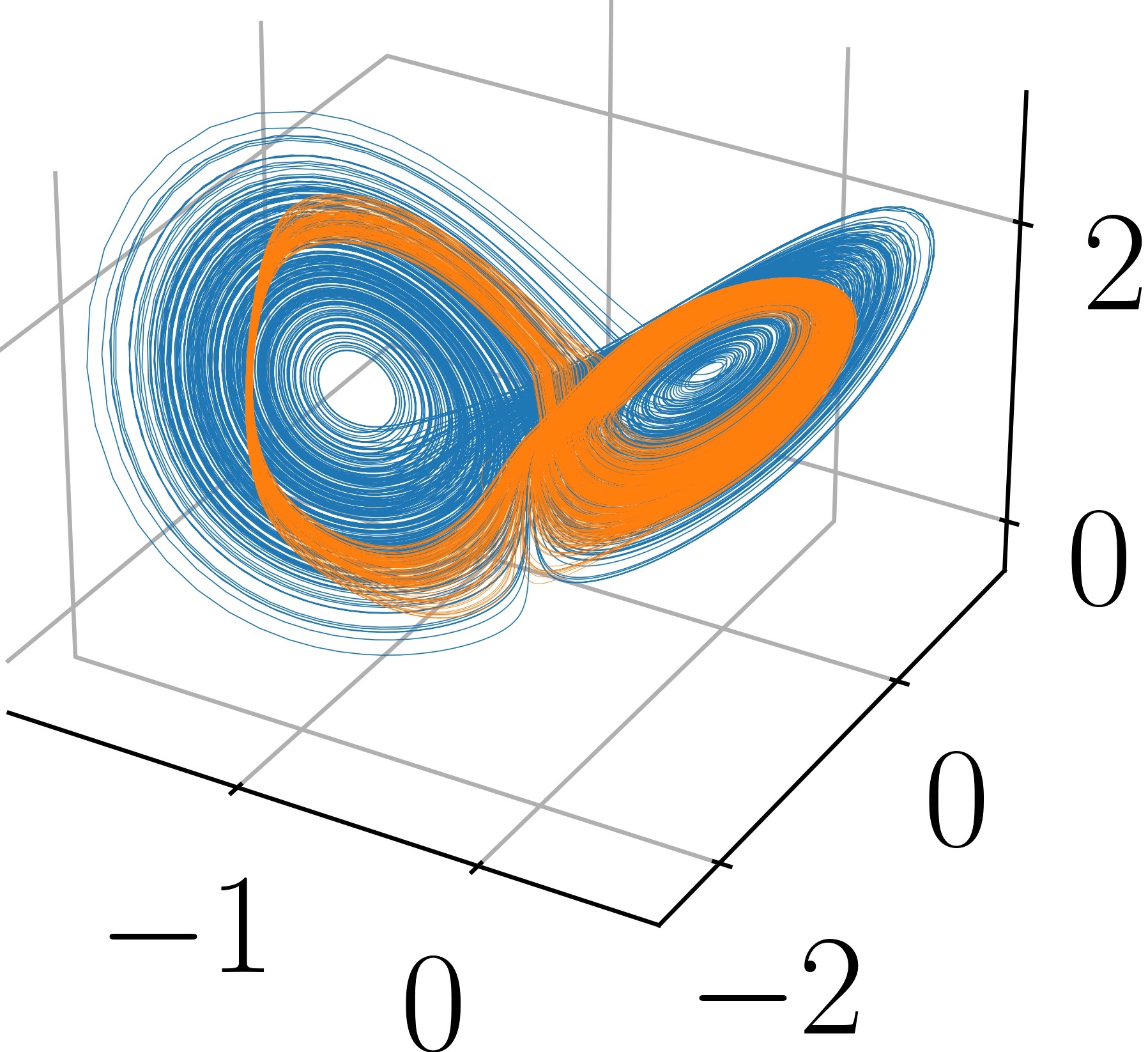}}
\caption{\small
Same as Fig. \ref{fig:lorenz-reconstruction} for vanilla RNNs. Although, as this graph confirms, with sparsely forced BPTT training of vanilla RNNs on chaotic systems becomes feasible, generally they were somewhat harder to train than the other RNN architectures (likely due to their known problems with long-range dependencies).
\label{fig:RNN-reconstruction}}
\normalsize
\end{figure}
\begin{figure}[h]  
\centering    
\subfigure{\label{fig:b:1}\includegraphics[width=0.6\linewidth]{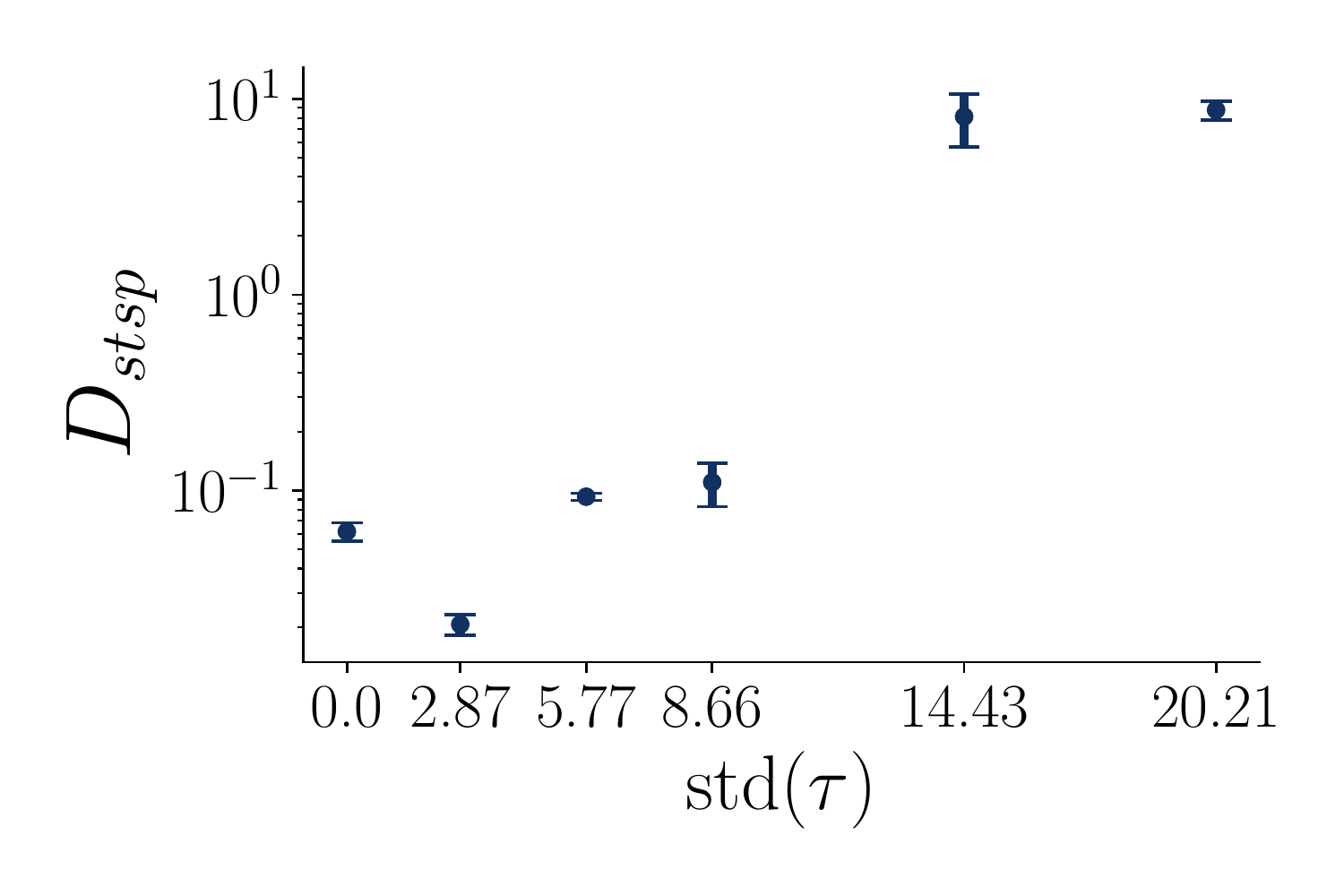}}
\vspace{-0.3cm}
\caption{\small 
Teacher forcing for LSTM with the learning interval $\tau$ drawn uniformly random around the optimal value ($\tau_{opt}=30$) with standard deviation std$(\tau)$ (for std$(\tau) >8.66$ the interval becomes asymmetric, however, due to the lower bound at $\tau=1$). As 
std$(\tau)$ is increased, performance generally degrades. A little jittering around the optimal interval $\tau_{pred}$ may potentially help, however (as more commonly observed in various machine learning procedures). 
}\label{fig:random-sparseTF}
\normalsize
\end{figure}
\pagebreak

\subsection{Sparsely forced BPTT}
\paragraph{Loss truncation} \label{supp:loss_trunc}
One implicit consequence of the teacher forcing, eqn. \eqref{eq:forcing}, is the interruption of the hidden-to-hidden connections at these time points. More specifically, if the system is forced at time $t  \in \mathcal{T}$, then there is no connection between $\vz_{t}$ and $\vz_{t+1}$, that is
\begin{align}\label{eq:supp-van-jac}
    \mJ_{t+1} = \frac{\partial \vz_{t+1}}{\partial \vz_{t}} = \frac{\partial RNN(\tilde{\vz}_{t})}{\partial \vz_{t}}=0.
\end{align}
To see how these vanishing Jacobians truncate the loss gradients w.r.t to some parameter $\theta$, let us focus on the loss gradients immediately after the forcing, 
\begin{align}\nonumber
    \frac{\partial \mathcal{L}_{t+1}}{\partial \theta} 
    &= \frac{\partial \mathcal{L}_{t+1}}{\partial \vz_{t+1}} \sum_{k=1}^{t+1} \frac{\partial \vz_{t+1}}{\partial \vz_k}\frac{\partial^{+} \vz_k}{\partial \theta}\\\nonumber
    &= \frac{\partial \mathcal{L}_{t+1}}{\partial \vz_{t+1}}(\frac{\partial^{+} \vz_{t+1}}{\partial\theta} +\sum_{k=1}^{t}  \underbrace{\frac{\partial\vz_{t+1}}{\partial \vz_k}}_{=0\, \text{, because of \eqref{eq:supp-van-jac}}}\frac{\partial^{+} \vz_k}{\partial\theta}) \\ \label{eq:supp-tr-loss-1}
    &=\frac{\partial \mathcal{L}_{t+1}}{\partial \vz_{t+1}}\frac{\partial^{+} \vz_{t+1}}{\partial \theta}. 
\end{align}
Eqn.\eqref{eq:supp-tr-loss-1} shows that sparsely forced BPTT implicitly truncates the loss gradients because it interrupts the hidden-to-hidden connection from $\vz_{t}$ to $\vz_{t+1}$ for $t \in \mathcal{T}$.  More generally, defining $\widetilde{t} := \max \{t^{\prime} \in \mathcal{T}: t^{\prime} \leq t \}$, the overall loss gradients are truncated to  
 \newcommand\myeq{\stackrel{\mathclap{\normalfont\mbox{tr.}}}{=}}
\begin{align}\nonumber
\frac{\partial \mathcal{L}}{\partial \theta} &=
\sum_{t=1}^T \frac{\partial \mathcal{L}_t}{\partial \vz_t} \sum_{k=1}^t \frac{\partial \vz_t}{\partial \vz_k}\frac{\partial^{+} \vz_k}{\partial \theta}
\\
&\myeq \sum_{t=1}^T \frac{\partial \mathcal{L}_t}{\partial \vz_t} \sum_{k=\widetilde{t}}^t \frac{\partial \vz_t}{\partial \vz_k}\frac{\partial^{+} \vz_k}{\partial \theta}.
\end{align}


\end{document}